%% file: ALLN.tex
\documentclass[11pt]{article}

\usepackage[margin=1in]{geometry}
\usepackage{amsmath,amsthm,amsfonts}
\usepackage{xspace}
\usepackage[hidelinks]{hyperref}
\usepackage{cleveref}
\usepackage{graphicx, xcolor}
\usepackage{relsize}
\usepackage{enumitem}

\usepackage{fullpage}
\usepackage{boxedminipage}
\usepackage[boxed]{algorithm}
\usepackage{epigraph}
\usepackage[sc]{mathpazo}
\usepackage{amsmath}
\usepackage{caption}
\usepackage{subcaption}
\usepackage[normalem]{ulem}
\usepackage{amsthm}
\usepackage{amssymb}
\usepackage{amsfonts}
\usepackage{dsfont}

\newcommand{\lp}{\left}
\newcommand{\rp}{\right}

\newtheorem{theorem}{Theorem}[section]

\newtheorem{definition}[theorem]{Definition}

\newtheorem{lemma}[theorem]{Lemma}

\newtheorem{claim}[theorem]{Claim}
\newtheorem{remark}[theorem]{Remark}

\newtheorem*{question*}{Question}

\newcommand{\Ex}{{\mathbb{E}}}
\newcommand{\eps}{\epsilon}
\newcommand{\Ldim}{\mathsf{Ldim}}
\newcommand{\dS}{\mathbb{B}}
\newcommand{\maj}{\mathsf{Lmaj}}

\newcommand{\E}{\mathcal{E}}
\newcommand{\A}{\mathcal{A}}

\newcommand{\Bin}{\mathrm{Bin}}

\newcommand{\Ber}{\mathrm{Ber}}
\newcommand{\Res}{\mathrm{Res}}
\newcommand{\Uni}{\mathrm{Uni}}
\newcommand{\support}{\mathrm{support}}
\newcommand{\VC}{\mathrm{VC}}

\newcommand{\App}{\mathrm{App}}
\newcommand{\Net}{\mathrm{Net}}
\newcommand{\Disc}{\mathrm{Disc}}

\renewcommand{\vec}[1]{\overline{#1}}
\newcommand{\rv}[1]{\boldsymbol{#1}}
\newcommand{\rvv}[1]{\vec{\rv{#1}}}
\newcommand{\rvI}{{\rv{I}}}
\newcommand{\rvvI}{\vec{\rvI}}
\newcommand{\rvx}{\rv{x}}
\newcommand{\vx}{\vec{x}}
\newcommand{\rvvx}{\vec{\rvx}}
\newcommand{\om}{\overline{m}}
\newcommand{\um}{\underline{m}}
\newcommand{\Adv}{\mathrm{Adv}}

\newcommand{\omri}[1]{\textcolor{orange}{Omri: {#1}}}
\newcommand{\shay}[1]{\textcolor{blue}{Shay: {#1}}}

\newcommand{\yuval}[1]{\textcolor{purple}{ Yuval: {#1}}}

\title{Adversarial Laws of Large Numbers \\and\\ Optimal Regret in Online Classification}

\author{\begin{tabular}[t]{c@{\extracolsep{4em}}cc} 
	
	Noga Alon\thanks{Department of Mathematics, Princeton University, Princeton, New Jersey, USA and Schools of Mathematics and Computer Science, Tel Aviv University, Tel Aviv, Israel. Research supported in part by
NSF grant DMS-1855464,
BSF grant 2018267
and the Simons Foundation. Email: \url{nalon@math.princeton.edu}.}
		&
		Omri Ben-Eliezer\thanks{Center for Mathematical Sciences and Applications, Harvard University, Massachusetts, USA. Research partially
conducted while the author was at Weizmann Institute of Science, supported in part by a grant from the Israel Science Foundation (no. 950/15). Email: \url{omribene@cmsa.fas.harvard.edu}.}
		&
		Yuval Dagan\thanks{Computer Science and Artificial Intelligence Laboratory, Massachusetts Institute of Technology, Cambridge, Massachusetts, USA. Email: \url{dagan@mit.edu}.} 
		\vspace{0.15cm}
		\\
		Shay Moran\thanks{Department of Mathematics, Technion, Israel. Email: \url{smoran@technion.ac.il}.
		Research supported in part by the Israel Science Foundation (grant No. 1225/20), by an Azrieli Faculty Fellowship, and by a grant from the United States - Israel Binational Science Foundation (BSF).}
		&
		Moni Naor\thanks{Department of Computer Science and Applied Mathematics, Weizmann Institute of Science, Rehovot, Israel.
Supported in part by grants from the Israel Science Foundation (no. 950/15 and 2686/20) and by the Simons Foundation Collaboration
on the Theory of Algorithmic Fairness. Incumbent of the Judith Kleeman Professorial Chair. Email: \url{moni.naor@weizmann.ac.il}.}
		&
		Eylon Yogev\thanks{Department of Computer Science, Boston University and Department of Computer Science, Tel Aviv University. Email: \url{eylony@gmail.com}. Research supported in part by ISF grants 484/18, 1789/19, Len Blavatnik and the Blavatnik Foundation, and The Blavatnik Interdisciplinary Cyber Research Center at Tel Aviv University.}
		\end{tabular}
		}

\date{}

\begin{document}
\maketitle
\begin{abstract}
Laws of large numbers guarantee that given a large enough sample from some population, the measure of any fixed sub-population is well-estimated by its frequency in the sample.
We study laws of large numbers in sampling processes that can affect the environment they are acting upon and interact with it. Specifically, we consider the sequential sampling model proposed by Ben-Eliezer and Yogev (2020), and characterize the classes which admit a uniform law of large numbers in this model: these are exactly the classes that are \emph{online learnable}.
Our characterization may be interpreted as an online analogue to the equivalence between learnability and uniform convergence in statistical (PAC) learning. 
	
The sample-complexity bounds we obtain are tight for many parameter regimes, and as an application, we determine the optimal regret bounds in online learning, stated in terms of \emph{Littlestone's dimension}, thus resolving the main open question from Ben-David, P\'al, and Shalev-Shwartz (2009), which was also posed by Rakhlin, Sridharan, and Tewari (2015).
\end{abstract}

\thispagestyle{empty}
\pagebreak
\thispagestyle{empty}

\tableofcontents
\thispagestyle{empty}
\pagebreak
\pagenumbering{arabic}

\input{intro}

\input{main-results}

\input{online-learning-intro}

\input{Extensions}
\input{technical-overview}

\input{related-work}

\input{preliminaries}


\input{epsilon-approx}
\input{epsilon-nets}

\input{double-sampling}
\input{dynamic-sets}
\input{reductions-sampling-schemes}
\input{online}

\input{lower-bounds}

\bibliographystyle{alpha}
\bibliography{adaptive-sampling,online,stat}

\appendix
\input{prob-material}

\end{document}

%% file: intro.tex
\section{Introduction}
	When analyzing an entire population is infeasible, statisticians apply {\em sampling methods}
	by selecting a {\it sample} of elements from a target population as a guide to the entire population. 
	Thus, one of the most fundamental tasks in statistics is to provide bounds on the sample size
	which is sufficient to soundly represent the population, and probabilistic tools are used to derive such guarantees,
	under a variety of assumptions. Virtually all of these guarantees are based on classical probabilistic models which assume that {\it the target population is fixed in advance and does not depend on the sample collected throughout the process}. Such an assumption, that the setting is \emph{offline} (or \emph{oblivious} or \emph{static}), is however not always realistic. 
	In this work we explore an abstract framework which removes this 
	assumption, and prove that natural and efficient sampling processes 
	produce samples which soundly represent the target population.

Situations where the sampling process explicitly or implicitly affects the target population are abundant in modern data analysis.
Consider, for instance, navigation apps that optimize traffic by routing drivers to less congested routes:
{such apps collect statistics from drivers to estimate the traffic-load on the routes, 
and use these estimates to guide their users through faster routes.
Thus, such apps interact with and affect the statistics they estimate.
Consequently, the assumption that the measured populations do not depend on the measurements is not realistic.}

Similar issues generally arise in settings involving decision-making in the face of an ever-changing (and sometimes even adversarial) environment;
a few representative examples include autonomous driving \cite{Sitawarin2018}, adaptive data analysis \cite{Dwork2015Science,Woodworth18}, security \cite{NaorY15}, and theoretical analysis of algorithms \cite{ChuGPSSW18}. 
Consequently, there has recently been a surge of works exploring such scenarios, a partial list includes
\cite{MironovNS11,GilbertHRSW12,GilbertHSWW12,HardtW13,NaorY15,BenEliezerJWY20,cherapanamjeri2020adaptive,haghtalab2020smoothed,HassidimKMMS20}. 
In this work, we focus on the sequential sampling model recently proposed by Ben-Eliezer and Yogev \cite{BenEliezerYogev2020}.

\subsection{The Adversarial Sampling Model}
We next formally describe the sampling setting and the main question we investigate.
Ben-Eliezer and Yogev \cite{BenEliezerYogev2020} model sampling processes over a domain $X$ as a sequential game between two players: a sampler and an adversary.
The game proceeds in $n$ rounds, where in each round $i=1,\ldots,n$:
\begin{itemize}[noitemsep]
\item The adversary picks an item $x_i\in X$ and provides it to the sampler. 
The choice of $x_i$ might depend on $x_1, \ldots, x_{i-1}$ and on all information sent to the adversary up to this point.
\item Then, the sampler decides whether to add $x_i$ to its sample.
\item Finally, the adversary is informed of whether $x_i$ was sampled by the sampler.
\end{itemize}
The number of rounds $n$ is known in advance to both players.\footnote{Though we will also consider samplers which are oblivious to the number of rounds $n$.} We stress that both players can be randomized, in which case their randomness is private (i.e., not known to the other player).

\vspace{1mm}

\noindent{\bf Oblivious Adversaries.} In the oblivious (or static) case, the sampling process consists only of the first two bullets.
Equivalently, oblivious adversaries decide on the entire stream in advance, without receiving any feedback from the sampler.
Unless stated otherwise, the adversary in this paper is assumed to be adaptive (not oblivious).

\paragraph{Uniform Laws of Large Numbers.}

Uniform laws of large numbers (ULLN) quantify the minimum sample size which is sufficient to {\it uniformly estimate multiple statistics of the data.} (Rather than just a {\it single} statistic, as in standard laws of large numbers.)
This is relevant, for instance, in the example given above regarding the navigation app: 
	it is desirable to accurately compute the congestion along \emph{all} routes (paths). 
	Otherwise, one congested route may be regarded as entirely 
	non-congested, and it will be selected for navigation. 

Given a family $\E$ of subsets of $X$,
	we consider ULLNs that estimate the frequencies of each subset $E\in \E$ within the adversarial stream.
	Formally, let $\vx=\{x_1,\ldots,x_n\}$ denote the input-stream produced by the adversary, 
	and let $\vec{s}=\{x_{i_1},\ldots,x_{i_k}\}$ denote the sample chosen by the sampler.
	The sample $\vec{s}$ is called an {\it $\eps$-approximation of the stream~$\vx$ with respect to $\E$} if:
\begin{equation}\label{eq:ULLN-def}
(\forall E\in {\cal E} ) : \quad \left\lvert \frac{\lvert\bar s \cap E \rvert}{\lvert\bar s \rvert} - \frac{\lvert\bar x \cap E \rvert}{\lvert\bar x \rvert} \right\rvert \le \epsilon.
\end{equation}
That is, $\vec{s}$ is an $\eps$-approximation of $\vx$ if the {\it true-frequencies} ${\lvert\bar x \cap E \rvert}/{\lvert\bar x \rvert}$ are uniformly approximated by the {\it empirical frequencies} ${\lvert\bar s \cap E \rvert}/{\lvert\bar s \rvert}$. 
The following question is the main focus of this work:

\begin{question*}[Main Question]
Given a family $\E$, an error-parameter $\epsilon>0$, and $k \in \mathbb{N}$, 
is there a sampler that, given any adversarially-produced input stream $\vx$,
picks a sample $\vec{s}$ of at most $k$ items 
which forms an $\eps$-approximation of $\bar x$, with high probability?
\end{question*}

\paragraph{The Story in the Statistical Setting.}
It is instructive to compare with the statistical setting in which the sample $\vec{s}$ is drawn independently from an unknown distribution over $X$.
	Here, ULLNs are characterized by the Vapnik-Chervonenkis (VC) Theory
	which asserts that a family $\E$ satisfies a ULLN if and only if its VC dimension, $\VC(\E)$, is finite~\cite{VC1971}.

This fundamental result became a corner-stone in statistical machine learning.
	In particular, {\it The Fundamental Theorem of PAC Learning} states that the following properties are equivalent for any family $\E$: 
	(1) $\E$ satisfies a uniform law of large numbers,
	(2) $\E$ is PAC learnable, and (3)~$\E$~has a finite VC dimension.
	Quantitatively, the sample size required for both $\epsilon$-approximation and for PAC learning with excess-error $\epsilon$ is $\Theta((\VC(\E)+\log(1/\delta))/\epsilon^2)$.
	
\paragraph{Spoiler:}	Our main result (stated below) can be seen as an online/adversarial analogue of this theorem where the {\bf Littlestone dimension} replaces the VC dimension.

%% file: main-results.tex
\section{Main Results}

\subsection{Adversarial Laws of Large Numbers}

The main result in this paper is a characterization of adversarial uniform laws of large numbers in the spirit of VC theory and The Fundamental Theorem of PAC Learning.
	We begin with the following central definition.
\begin{definition}[Adversarial ULLN]
We say that a family $\E$ satisfies an \emph{adversarial ULLN} if for any $\epsilon,\delta > 0$, 
	there exist $k=k(\epsilon,\delta)\in\mathbb{N}$ and a sampler $\cal S$ satisfying the following.
	For any adversarially-produced input-stream~$\vx$ (of any size), $\cal S$ chooses a sample of at most $k$ items,
	which form an $\eps$-approximation of $\vx$ with probability at least $1-\delta$.
	We denote by $k(\E,\epsilon,\delta)$ the minimal such value of $k$.
\end{definition}

Note that this definition requires the sample complexity $k=k(\epsilon,\delta)$ to be a constant independent of the stream size $n$. Another reasonable requirement is $k = o(n)$. It turns out that these two requirements are equivalent.

Which families $\E$ satisfy an adversarial law of large numbers? 
	Clearly, $\E$ must have a finite VC-dimension, as otherwise, basic VC-theory implies that
	any sampler will fail to produce an $\epsilon$-approximation even against oblivious adversaries 
	which draw the input-stream $\vx$ independently from a distribution on $X$.
	However, finite VC dimension is not enough in the fully adversarial setting: 
	\cite{BenEliezerYogev2020} exhibit a family $\E$ with $\VC(\E)=1$ that does not satisfy an adversarial
	ULLN.



Our first result provides a characterization of adversarial ULLN in terms of \emph{Online Learnability},
	which is analogous to the Fundamental Theorem of PAC Learning.
	In this context, the role of VC dimension is played by the {\it Littlestone dimension}, 
	a combinatorial parameter which captures online learnability similar to how the VC dimension captures PAC learnability.
	{(See \Cref{sec:basic_defs} for the formal definition.)}
\begin{theorem}[Adversarial ULLNs -- Qualitative Characterization]\label{thm:qualitative}
Let $\E$ be a family of subsets of $X$.
	Then, the following statements are equivalent:
	\begin{enumerate}
	\item $\E$ satisfies an adversarial ULLN;
	\item $\E$ is online learnable; and 
	\item $\E$ has a finite Littlestone dimension.
	\end{enumerate}
\end{theorem}
The proof follows from Theorems \ref{thm:quantitative} and \ref{thm:lbforall} (and from the well-known equivalence between online learnability and finite Littlestone dimension \cite{Littlestone87,BenDavidPS09}).
Our quantitative upper bound for the sample-complexity $k(\E,\epsilon,\delta)$, which is the main technical contribution of this paper, is stated next.
\begin{theorem}[Adversarial ULLNs -- Quantitative Characterization]\label{thm:quantitative}
	Let $\E$ be a family with Littlestone dimension $d$. 
	Then, the sample size $k(\E,\epsilon,\delta)$, which suffices to produce an $\epsilon$-approximation satisfies:
	\[
	k(\E,\epsilon,\delta) \le O\lp(\frac{d+\log(1/\delta)}{\epsilon^2}\rp).
	\]
\end{theorem}
The above upper bound is realized by natural and efficient samplers;
for example it is achieved by: 
(i) the \emph{Bernoulli sampler} $\Ber(n,p)$ which retains each element with probability $p=k/n$;
(ii) the \emph{uniform sampler} $\Uni(n,k)$ that draws a subset $I \subseteq [n]$ uniformly at random from all the subsets of size $k$ and selects the sample $\{x_t\colon t\in I\}$; and (iii)
the \emph{reservoir sampler} $\Res(n,k)$ (see \Cref{subsec:extensions}) that maintains a uniform sample continuously throughout the stream.

\subsubsection{Lower Bounds}

The upper bound in \Cref{thm:quantitative} cannot be improved in general.
	In particular, it is tight in all parameters for {\it oblivious samplers}: 
	a sampler is called oblivious if the indices of the chosen subsample are independent of the input-stream.
	(The Bernoulli, Reservoir, and Uniform samplers are of this type.)
	A lower bound of $\Omega((d+ \log(1/\delta))/\eps^2)$ for oblivious samplers directly follows from VC-theory, 
	and applies to any family $\E$ for which the VC dimension and Littlestone dimension are of the same order.\footnote{E.g., projective spaces, Hamming balls, lines in the plane, and others. } 
For unrestricted samplers we obtain bounds of $\Omega(d/\eps^2)$ for $\eps$-approximation and $\Omega(d\log(1/\eps)/\eps)$ for $\eps$-nets.
	We state these results and prove them in \Cref{sec:LBs}.
	
The above lower bound proofs hold for specific ``hard'' families $\E$.
	This is in contrast with the statistical or oblivious settings in which a lower bound of $\Omega((\VC(\E) + \log(1/\delta))/\eps^2)$ applies to any class. 
	We do not know whether an analogous result 
	holds in the adversarial sampling setting and leave it as an open problem. 
	We do show, however, that the linear dependence in $d$ is necessary for any $\cal E$, as part of proving \Cref{thm:qualitative}.

%% file: online-learning-intro.tex
\subsection{Online Learning} \label{subsec:online_learning_connection2}

We continue with our main application to online learning.
	Consider the setting of online prediction with binary labels;
	a learning task in this setting can be described as a guessing game between a learner and an adversary. 
	The game proceeds in rounds $t=1,\dots,T$, each consisting of the following steps:
\begin{itemize}[noitemsep]
	\item The adversary selects $(x_t,y_t) \in X\times \{0,1\}$ and reveals $x_t$ to the learner.
	\item The learner provides a prediction $\hat y_t \in \{0,1\}$ of $y_t$ and announces it to the adversary.
	\item The adversary announces $y_t$ to the learner.
\end{itemize}
The goal is to minimize the number of mistakes, $\sum_t \mathds{1}(y_t \ne \hat y_t)$. 
Given a class $\E$, the {\it regret} of the learner w.r.t.\ $\E$ 
is defined as the difference between the number of mistakes
made by the learner and the number of  mistakes made by the best $E\in \cal E$: 
\[ 
\sum_t \mathds{1}(y_t \ne \hat y_t)
- \min_{E \in \E}\sum_t \mathds{1}\Bigl(y_t \ne \mathds{1}(x_t \in E)\Bigr).\]
A class $\E$ is {\it online-learnable} if there exists an online learner 
whose (expected) regret w.r.t.\ every adversary is at most $R(T)$, where $R(T)=o(T)$. (The amortized regret $R(T)/T$ vanishes as $T\to\infty$.)
Ben-David, P\'al, and Shalev-Shwartz~\cite{BenDavidPS09} proved that for every class~$\E$, 
the optimal regret $R_T(\E)$ satisfies 
\begin{equation}\label{eq:online-regret}
\Omega(\sqrt{d\cdot T}) \leq R_T(\E) \leq O(\sqrt{d\cdot T\log T}),
\end{equation}
where $d$ is the Littlestone dimension of $\E$, and left closing that gap as their main open question.
{Subsequently, Rakhlin, Sridharan, and Tewari~\cite{rakhlin2010online,rakhlin2015online,rakhlin2015sequential}
defined the notion of \emph{Sequential Rademacher Complexity}, proved that it captures regret bounds in online learning in a general setting, and used it to re-derive \Cref{eq:online-regret}. They also asked as an open question whether the logarithmic factor in \Cref{eq:online-regret} can be removed and pointed on difficulties to achieve this using some known techniques \cite{rakhlin2014statistical,rakhlin2015sequential}.}

{We show that the sequential Rademacher complexity also captures the sample-complexity of $\epsilon$-approximations and bound it in the proof of \Cref{thm:quantitative}. This directly implies a tight bound on online learning: (See \Cref{sec:onlinetechnical} for more details.)} 

\begin{theorem}[Tight Regret Bounds in Online Learning]\label{thm:online}
	Let $\E$ be a class with Littlestone dimension~$d$. Then the optimal regret bound in online learning $\E$ is $\Theta(\sqrt{d\cdot T})$.
\end{theorem}
The lower bound was shown by \cite{BenDavidPS09}. 
We prove the upper bound in \Cref{sec:onlinetechnical}.

%% file: COLT21/extensions.tex
\subsection{Applications and Extensions}
\label{subsec:extensions}

We next discuss applications and extensions of our results.

\paragraph{Epsilon Nets.}
We also provide sample complexity bounds for producing {\it $\eps$-nets}:
a subsample~$\vec{s}$ of the stream~$\vx$ is an $\eps$-net if whenever $E \in \E$ satisfies $|E\cap\vx|\ge \epsilon n$, then $\vec{s}\cap E \ne \emptyset$. 
I.e.\ the subsample $\vec{s}$ hits every $E\in \E$ which contains at least an $\eps$-fraction of the items in the stream.

Epsilon nets are a fundamental primitive in computational geometry and in learning theory.
In computational geometry this notion underlies fundamental algorithmic techniques, and
in learning theory it is tightly linked to the learnability in the {\it realizable} setting. 
In that sense, it is analogous to  $\eps$-approximations, which correspond to learnability in the \emph{agnostic}  setting. 

In \Cref{sec:eps-nets} we show that, like $\eps$-approximations,  $\eps$-nets are also characterized by the Littlestone dimension; and similarly, our results here provide tight sample-complexity bounds.

\paragraph{Maintaining An $\eps$-Approximation Continuously.}
Some natural applications require that the sampler continuously maintains an $\eps$-approximation with respect to the prefix of the stream observed thus-far.
	To address such scenarios we slightly modify the adversarial sampling setting by allowing the sampler to delete items from its sample.
	In this modified setting, we prove that the classical {\it Reservoir sampler}~\cite{Vitter85}, $\Res(n,k)$ (see \Cref{sec:preliminaries} for the precise definition),
	enjoys similar guarantees to those of \Cref{thm:quantitative} above. Concretely, 
	the exact same bound of \Cref{thm:quantitative} is achieved by reservoir sampling if one is only interested in $\eps$-approximation at the end of the process; for continuous $\eps$-approximation, the same bound with an added term of $O(\log \log (n))$ in the numerator suffices (see \Cref{thm:continuous}).

Notably, allowing deletions does not add significant power to the sampler, 
	and in particular \Cref{thm:qualitative} still applies in this setting.

\paragraph{ALLNs for Real-Valued Function Classes}
The adversarial sampling setting naturally extends to real-valued function classes $\E$.
Moreover, much of the machinery developed in this paper readily applies in this case.
In particular, the relationship with the sequential Rademacher complexity is retained.
Therefore, since the sequential Rademacher complexity captures regret bounds in online learning,
this allows an automatic translation of regret bounds from online learning 
to sample complexity bounds in adversarial ULLNs w.r.t.\ real-valued function classes.\footnote{The reduction from bounds on $\epsilon$-approximations to bounds on the sequential Rademacher complexity appear in \Cref{sec:approx}. They rely on concentration inequalities for $\{0,1\}$ valued random variables that have analogues for $[0,1]$ valued random variables with the same guarantees. This enables a direct extension of this reduction.}


\paragraph{Algorithmic Applications}
Part of the reason that the Fundamental Theorem of PAC Learning became a corner-stone in machine learning theory
	is due to its algorithmic implications. In particular, because it justifies the {\it Empirical Risk Minimization Principle} (ERM),
	which asserts that in order to learn a VC class, it suffices to minimize the empirical loss w.r.t.\ a random sample.
	This principle reduces the learning problem (of minimizing the loss w.r.t.\ an unknown distribution)
	to an optimization problem of minimizing the loss w.r.t.\ the (known) input sample.

It will be interesting to explore such implications in the adversarial setting.
	One promising direction is to use these sampling methods to design {\it lazy streaming/online algorithms}.
	That is, algorithms that update their internal state only on a small (random) substream.
	Intuitively, if that substream represents the entire stream in an appropriate way,
	then the performance of the algorithm will be satisfactory, and the gain in efficiency can be significant.
In fact, our proof of \Cref{lem:fractional-cover-littlestone} identifies and exploits such a phenomenon in online learning:
	we use a {\it lazy online learner} that updates its predictor rarely, only in a small random subsample of examples.



%% file: technical-overview.tex
\section{Technical Overview}
\label{sec:tech_overview}

We next overview the technical parts in this work.
We outline the proofs of the main theorems, 
and try to point out which technical arguments are novel, 
and which are based on known techniques.
A more detailed overview of particular proofs is given in the dedicated sections.

\subsection{Upper Bounds}
We begin with the sample-complexity upper bound, \Cref{thm:quantitative} (which is the longest and most technical derivation in this work).

\paragraph{Reductions Between Samplers.}

Our goal is to derive an upper bound for the Bernoulli, uniform, and reservoir samplers.
In order to abstract out common arguments, we develop a general framework 
which serves to methodically transform sample-complexity bounds between the different samplers via a type of ``online reductions''.
This framework allows us to bound the sample-complexity with respect to one sampler,
and {\it automatically} deduce them for the other samplers. 
The reduction relies on transforming one sampling scheme into another in an online fashion,
and from a technical perspective, this boils down to coupling arguments, 
similar to coupling techniques in Markov Chains processes \cite{levin2017markov}.
Section~\ref{sec:reduction-sampling} contains a more detailed overview followed by the formal derivations. 


\paragraph{Upper Bounds for The Uniform Sampler.}
Thus, for the rest of this overview we focus the sampling scheme to be the uniform sampler 
which uniformly draws a $k$-index-set $I\subseteq [n]$, and selects the subsample $\vx_I = (x_i : i\in I)$. 
Our goal is to show that with probability $\geq 1-\delta$, 
\begin{equation}\label{eq:population}
	\sup_{E\in \cal E}\lp|\frac{|\vx_I\cap E|}{k} - \frac{|\vx\cap E|}{n}\rp| \leq O\Bigl(\sqrt{\frac{d+ \log(1/\delta)}{k}}\Bigr),
\end{equation}
where $d$ is the Littlestone dimension of $\cal E$ and $\vx$ is the \emph{adversarially} produced sequence.
The proof consists of two main steps which are detailed below.

\subsubsection{Step 1: Reduction to Online Discrepancy via Double Sampling}

The first step in the proof consists of an online variant of the celebrated {\it double-sampling argument} due to \cite{VC1971}.
This argument serves to replace the error w.r.t.\ the entire population by the error w.r.t.\ a small {\it test-set} of size $k$,
thus effectively restricting the domain to the $2k$ items in the union of the selected sample and the test-set.
In more detail, let $J\subseteq [n]$ be a uniformly drawn {\it ghost} subset of size $k$ 
which is \underline{disjoint} from $I$, and is \underline{not known} to the adversary.
Consider the maximal deviation between the sample $\vx_i$ and the ``test-set'' $\vx_J$:
\begin{equation}\label{eq:test}
	\sup_{E\in \cal E}\lp|\frac{|\vx_I\cap E|}{k} - \frac{|\vx_J\cap E|}{k}\rp|.
\end{equation}
The argument proceeds by showing that for a typical $J$, the deviation w.r.t.\ the entire population $\vx$ in the LHS of \Cref{eq:population}
has the same order of magnitude like the deviation w.r.t.\ the test-set~$\vx_J$ in \Cref{eq:test} above.
Hence, it suffices to bound \eqref{eq:test}.

In order to bound \Cref{eq:test}, consider sampling $I,J$ according to the following process: 
(i)~First sample the $2k$ indices in $I\cup J$ uniformly from $[n]$, 
and \underline{reveal} these $2k$ indices to {both players} (in advance).
(ii) Then, the sampler draws $I$ from these $2k$ indices in an online fashion (i.e., the adversary does not know in advance the sample $I$).
Intuitively, this modified process only helps the adversary who has the additional
information of a superset of size $2k$, which contains~$I$.
{\it What we gain is that the modified process is essentially equivalent to reducing the horizon from $n$ to $2k$.}
The case of $n=2k$ can be interpreted as an online variant of the well-studied {\it Combinatorial Discrepancy} problem,
which is described next.

\paragraph{Online Combinatorial Discrepancy.}
The online discrepancy game w.r.t.\ $\cal E$ is a sequential game played between a painter and an adversary which proceeds as follows: at each round $t=1,\ldots,2k$ 
the adversary places an item $x_t$ on the board, and the painter colors $x_t$ in either red or blue. 
The goal of the painter is that each set in $\E$ will be colored in a balanced fashion; 
i.e., if we denote by $I$ the set of indices of items colored red, her goal is to minimize the discrepancy
\[
\Disc_{2k}(\E,\vx,I) := \max_{E \in \E}
\lp| |\vx_{I} \cap E| - |\vx_{[2k]\setminus I}\cap E| \rp|.
\]
One can verify that minimizing the discrepancy is equivalent to minimizing \Cref{eq:test}.
Moreover, each of the samplers $\Ber(2k,1/2)$ and $\Uni(2k,k)$ corresponds to natural coloring strategies of the painter;
in particular, $\Uni(2k,k)$ colors a random subset of $k$ of the items in red (and the rest in blue.)
Thus, we focus now on analyzing the performance of $\Uni(2k,k)$ in the online discrepancy problem.



\subsubsection{Step 2: From Online Discrepancy to Sequential Rademacher}

\newcommand{\Rad}{\mathrm{Rad}}
Instead of analyzing the discrepancy of $\Uni(2k,k)$,
it will be more convenient to consider the discrepancy of $\Ber(2k,1/2)$, which colors each 
item in red/blue uniformly and independently of its previous choices.
Towards this end, we show that these two strategies 
are essentially equivalent, using the reduction framework described at the beginning of this section. 

The discrepancy of $\Ber(2k,1/2)$ connects directly to the \emph{Sequential Rademacher Complexity}~\cite{rakhlin2015martingale},
defined as the expected discrepancy $\Rad_{2k}(\E) = \Ex \Disc_{2k}(\E, \vx,I)$, 
where the expectation is taken according to a uniformly drawn $I\subseteq [2k]$.
(Which is precisely the coloring strategy of $\Ber(2k,1/2)$.)



\subsubsection{Step 3.1: Bounding Sequential Rademacher Complexity -- Oblivious Case}

In what follows, it is convenient to set $n=2k$. Our goal here is to bound $\Rad_{n}(\E) \le O(\sqrt{{d\cdot n}})$.
As a prelude, it is instructive to consider the oblivious setting where the items 
$x_1,\ldots,x_{n}$ are fixed in advance, before they are presented to the painter. 
Here, the analysis is exactly as in the standard i.i.d.\ setting,
and the sequential Rademacher complexity becomes the standard Rademacher complexity. 
Consider the following three approaches, in increasing level of complexity.

\paragraph{First Approach: a Union Bound.}
Assume $\E$ is finite. Then, for each $E \in \E$ it is possible to show by concentration inequalities that with high probability, the discrepancy $||\vx_I\cap E| - |\vx_{[n]\setminus I}\cap E||$ is small. By applying a union bound over all $E \in \E$, one can derive that  $\Rad_{n}(\E) \le O(\sqrt{n\log |\E|})$.

\paragraph{Second Approach: Sauer-Shelah-Parles Lemma.}

Since $\E$ can be very large or even infinite, the bound in the previous attempt may not suffice. An improved argument relies on the celebrated Sauer-Shelah-Perles (SSP) Lemma~\cite{Sauer72Lemma},
which asserts that the number of distinct intersection-patterns of sets in~$\cal E$
with $\{x_1,\ldots, x_{n}\}$ is at most ${n \choose \leq \VC(\mathcal{E})}\leq  O(n^{\VC(\E)})$. The proof then follows by union bounding the discrepancy over $\{ \vx \cap E \colon E \in \E \}$, resulting in a bound of 
\[{O\Bigl(\sqrt{n\log \bigl(n^{\VC(\mathcal{E})}\bigr)}\Bigr) \le O\left(\sqrt{\VC(\E)n\log n}\right)},\] 
which is off only by a factor of $\sqrt{\log n}$.

\paragraph{Third Approach: Using Approximate Covers and Chaining.}
Shaving the extra logarithmic factor is a non-trivial task which was achieved in the seminal work by Talagrand~\cite{Talagrand94chaining} 
using a technique called {\it chaining}~\cite{dudley:87}. 
It relies on the notion of \emph{approximate covers}:
\begin{definition}[Approximate Covers]
	A family $\cal C$ is an $\eps$-cover of~$\cal E$ with respect to $x_1,\ldots, x_{n}$
	if for every $E\in \cal E$ there exists $C\in \cal C$ such that $E$ and $C$ agree on all but at most $\epsilon\cdot n$	of the $x_i$'s.
\end{definition}
In a nutshell, the chaining approach starts by finding covers $\mathcal{C}_0,\mathcal{C}_1,\dots$ where $\mathcal{C}_i$ is a $2^{-i}$-cover for $\E$ w.r.t.\ $\vx$, then writing the telescopic sum 
\[
\Disc_{n}(\E,\vx,I) = \Disc_{n}(\mathcal{C}_0,\vx,I)+\sum_{i=1}^\infty (\Disc_{n}(\mathcal{C}_i,\vx,I) - \Disc_{n}(\mathcal{C}_{i-1},\vx,I))
\]
and bounding each summand using a union bound.

Note that the SSP Lemma provides a bound of $\lvert {\cal C}\rvert \leq {n \choose \leq \VC(\E)}$ in the case of $\epsilon=0$, where $d$ is the VC-dimension of $\cal E$.
For $\epsilon > 0$, a classical result by Haussler \cite{Haussler1992} asserts that every family
admits an $\eps$-cover of size $(1/\eps)^{O(d)}$.
The latter bound allows via chaining to remove the redundant logarithmic factor and bound $\Rad_n(\E) \le O(\sqrt{\VC(\E)n})$.

\subsubsection{Step 3.2: Bounding Sequential Rademacher Complexity -- Adversarial Case}

We are now ready to outline the last and most technical step in this proof.
Our goal is twofold: first, we discuss how previous work \cite{BenDavidPS09,rakhlin2010online} generalized the above arguments to the adversarial (or the online learning) model, culminating in a bound of the form $\Rad_n(\E) = O(\sqrt{dn\log n})$. Then, we describe the proof approach for our improved bound of $O(\sqrt{dn})$.

\paragraph{An $O(\sqrt{dn\log n})$ Bound via Adaptive SSP.}
First, the union bound approach generalizes directly to the adversarial setting. 
However, the second approach, via the SSP lemma, does not. The issue is that in the adversarial setting, 
the stream $\vx$ can \underline{depend} on the coloring that the painter chooses, 
and hence $\{ E\cap \{x_1,\ldots, x_n\}: E\in\cal E\}$ depends on the coloring as well. 
In particular, it is not possible to apply a union bound over a small number of such patterns. 
Moreover, it is known that a non-trivial bound depending only on the VC dimension and $n$ does not exist \cite{rakhlin2015online}.
To overcome this difficulty we use an adaptive variant of the SSP Lemma due to~\cite{BenDavidPS09}, 
which is based on the following notion:
\begin{definition}[Dynamic Sets]
	A {\it dynamic set}~$\dS$ 
	is an online algorithm that operates on a sequence $\vx = (x_1,\dots,x_{n})$. 
	At each time $t=1,\dots,n$, the algorithm decides whether to retain $x_t$ as a function of $x_1,\dots,x_t$. 
	Let $\dS(\vx)$ denote the set of elements retained by $\dS$ on a sequence $\vx$.\footnote{\cite{BenDavidPS09} 
		refers to dynamic-sets as experts, which is compatible with the terminology of online learning.}
\end{definition}
Ben-David, P\'al, and Shalev-Shwartz \cite{BenDavidPS09} proved that any family $\E$ whose Littlestone dimension is $d$ 
can be covered by ${n \choose \leq d}$ dynamic sets. 
That is, for every $n$ there exists a family $\mathcal{C}$ of ${n\choose \leq d}$ dynamic sets such that for every sequence $\vx=(x_1,\ldots,x_n)$ and for every $E\in \cal E$ there exists a dynamic set $\dS \in \mathcal{C}$ which agrees with $E$ on the sequence $\vx$, namely, $\dS(\vx)= E \cap \vx$.

Using this adaptive 
SSP Lemma, 
one can proceed to bound the discrepancy as in the oblivious case 
by applying a union bound over the $(2k)^d$ dynamic sets,
and bounding the discrepancy with respect to each dynamic set using Martingale concentration bounds.
Implementing this reasoning yields a bound of $\Rad_n(\E) \le O(\sqrt{dn\log n})$ which is off by a logarithmic factor.


\paragraph{Removing the Logarithmic Factor.}
To adapt the chaining argument to the adversarial setting we first need to find small $\epsilon$-covers.
This raises the following question: 
\begin{center}
	{\it Can every Littlestone family be $\epsilon$-covered by ${(1/\eps)^{O(d)}}$ dynamic sets?}
\end{center}
Unfortunately, we cannot answer this question and leave it for future work. 
In fact, \cite{rakhlin2015sequential} identified a variant of this question as a challenge
towards replicating the chaining proof in the online setting.
To circumvent the derivation of dynamic approximate covers, we introduce a fractional variant which we term {\it fractional-covers}. 
It turns out that any Littlestone family admits ``small''  approximate fractional covers and these can be used to complete the chaining argument.
\begin{definition}[Approximate Fractional-Covers]
	A probability measure $\mu$ over dynamic sets $\dS$ is called an \emph{$(\epsilon,\gamma)$-fractional cover} for $\E$ if for any $\vx=(x_1,\ldots,x_n)$ 
	and any $E \in \E$,
	\[
	\mu\lp(\lp\{ \dS \colon \text{$E$ and $\dS(\vx)$ agree on all but at most $\epsilon n$ of the $x_i$s}\rp\}\rp) \ge 1/\gamma.
	\]
\end{definition}
The parameter $\gamma$ should be thought of as the size of the cover.
Observe that fractional-covers are relaxations of covers: indeed, if $\mathcal{C}$ is an $\epsilon$-cover for $\E$
then the uniform distribution over~$\mathcal{C}$ is an $(\epsilon,\gamma)$-fractional cover for $\E$ with $\gamma=\lvert \mathcal{C}\rvert$.

\paragraph{Small Approximate Fractional-Covers Exist.}
We prove \Cref{lem:fractional-cover-littlestone} which asserts that every Littlestone family $\E$ admits an $(\epsilon,\gamma)$-fractional cover of size 
\[\gamma=(O(1)/\epsilon)^{d}.\]
This fractional cover is essentially a mixture of non-fractional covers for subsets of the sequence~$\vx$ of size $d/\epsilon$. 
In more detail, the distribution over dynamic sets is defined by the following two-step sampling process: 
(1) draw a uniformly random subset $\vec{s}$ of $\vx$ of size $d/\epsilon$, and let $\mathcal{C}_s$ denote the (non-fractional) cover of $\E$ with respect to $\vec{s}$, 
which is promised by the dynamic variant of the SSP-Lemma. 
(2) Draw $\dS$ from the uniform distribution over $\mathcal{C}_s$.

We outline the proof that this is an $(\epsilon,\gamma)$-fractional cover with $\gamma = O(1/\epsilon)^{d}$. Fixing $E$ and $\vx$, our goal is to show that with probability at least $1/\gamma$ over $\mu$, the drawn $\dS$ agrees with $E$ on all but at most~$\epsilon\cdot n$ elements of $\vx$. This relies of the following two arguments:
(1) For every $\vec{s}$ there exists $\dS_{\vec{s}}\in \mathcal{C}_s$ that agrees with $E$ on $\vec{s}$; and (2) it can be shown that with high probability over the selection of the subset $\vec{s}$, $\dS_{\vec{s}}$ agrees with $E$ on all but at most $\epsilon n$ of the stream $\vx$. We call such values of $\vec{s}$ as \emph{good}, 
and conclude from the two steps above that
\begin{align*}
	\Pr_{\dS \sim \mu}[\dS \text{ agrees with $E$ on $(1-\epsilon) n$ of the $x_i$'s}]
	& \ge
	\Pr[\vec{s} \text{ is good}]
	\Pr_{\dS\sim \mathrm{uniform}(\mathcal{C}_s)}[\dS = \dS_{\vec{s}}]\\
	&\quad \ge \frac{1}{2} \cdot \frac{1}{|\mathcal{C}_s|}
	\ge \frac{1}{2 \binom{d/\epsilon}{\le d}}
	\ge \Omega(\epsilon)^d
	\ge \frac{1}{\gamma}
	\enspace.
\end{align*}
We further comment on the proof that $\vec{s}$ is {\it good} with high probability:
the proof relies on analyzing a \emph{lazy} online learner that updates its internal state only once encountering elements from~$\vec{s}$. We show that if $\vec{s}$ is drawn uniformly, then with high probability such a learner will make $\leq\epsilon\cdot n$ mistakes and this will imply that w.h.p. $\dS_{\vec{s}}$ agrees with $E$ on $(1-\epsilon) n$ stream elements.
We refer the reader to \Cref{sec:cover-eps} for the proof.

\paragraph{Chaining with Fractional Covers: Challenges and Subtleties.}
Here, we discuss how approximate fractional covers are used to bound the sequential Rademacher complexity.
	We do so by describing how to modify the bound that uses $0$-covers to use $(0,\gamma)$-fractional covers instead. 
	Recall that this argument goes by two steps: (1) bounding the discrepancy for each dynamic set in the cover, 
	and (2) arguing by a union bound that, with high probability the discrepancies of \emph{all} dynamic sets in the cover are bounded. 
	In comparison, with fractional covers, the second step is modified to: 
	(2') arguing that with high probability (over the random coloring), 
	the discrepancies of \emph{nearly all} the dynamic sets are bounded. 
	In particular, if more than a $(1-\gamma)$-fraction of the dynamic sets have bounded discrepancies, 
	then the discrepancies of \underline{all} sets in $\E$ are bounded.
	Indeed, this follows since every $E \in \E$ is covered by at least a $\gamma$-fraction of the dynamic-sets,
	and therefore, the pigeonhole principle implies that at least one such dynamic set also has bounded discrepancy,
	and hence $E$ has bounded discrepancy as well.

We note that multiple further technicalities are required to generalize the chaining technique for fractional covers
and refer the reader to \Cref{sec:eps-approx-via-cover} for a short overview of this method followed by its adaptation to the adversarial setting.

\subsection{Lower Bounds}
Beyond the $\Omega((d+ \log(1/\delta))/\eps^2)$ lower bound for oblivious samplers, which follows immediately from the VC literature, we prove several non-trivial lower bounds in other contexts. 
We distinguish between two types of approaches used to derive our lower bounds, described below. As the proofs are shorter than those of the upper bounds and more self-contained, we omit the exact technical details of the proofs in this overview and refer the reader to \Cref{sec:LBs}.

\paragraph{Universal Lower Bound by Adversarial Arguments.}
The main lower bound in \cite{BenEliezerYogev2020} exhibits a separation between the static and adversarial setting by proving an adversarial lower bound for the family of one-dimensional \emph{thresholds}. We identify that their proof implicitly constructs a tree as in the definition of the Littlestone dimension, and generalize their argument to derive an $\Omega(d)$ lower bound for \emph{all} families of Littlestone dimension $d$. For more details, see \Cref{thm:lbforall}.

\paragraph{Lower Bounds on the Minimum Sizes of $\eps$-Approximations/Nets.}
These lower bounds actually exhibit a much stronger phenomenon, showing that small $\eps$-approximations/nets \emph{do not exist} for some families $\cal E$. Thus, obviously, these cannot be captured by a sample of the same size. 

It is natural to seek lower bounds of this type in the VC-literature.
The main challenge is that many of the known lower bounds 
apply for geometric VC classes whose Littlestone dimension is unbounded. To overcome this, we present two lower bounds where $\Ldim$ can be controlled: one for $\eps$-approximation, which carefully analyzes a simple randomized construction,
and another 
for $\eps$-nets, which combines intersection properties of lines in the projective plane with probabilistic arguments. For more details, see Theorems \ref{thm:lbexists} and \ref{thm:lbnet} respectively.

%% file: related-work.tex
\section{Related Work}
\label{sec:related-work}

\subsection{VC Theory}

As suggested by the title, the results presented by this work are inspired
	by uniform laws of large numbers in the statistical i.i.d.\ setting and in particular by VC theory.
	(A partial list of basic manuscripts on this subject include~\cite{VC1971,vapnik:74,Dudely84book,vapnik:98}.)
	Moreover, the established equivalence between online learning and adversarial laws of large numbers
	is analogous to the the equivalence between PAC learning and uniform laws of large numbers in
	the i.i.d.\ setting. (See e.g.~\cite{VC1971,Blumer89VC,bousquet2004introduction,shalev2010learnability,Shalev-Shwartz2014}.)
	From a technical perspective, our approach for deriving sample complexity upper bound
	is based on the chaining technique~\cite{dudley1973,dudley:78,dudley:87}, 
	which was analogously used to establish optimal sample complexity bounds in the statistical setting~\cite{Talagrand94chaining}.
	(The initial bounds by \cite{VC1971} are off by a $\log(1/\eps)$ factor.)

From the lower bound side, our proofs are based on ideas originated
	from combinatorial discrepancy and $\eps$-approximations.
	(E.g.,~\cite{MWW93disc}; see the book by Matou\v{s}ek~\cite{matousekDiscrepencyBook} for a text-book introduction.)

\subsection{Online Learning}

The first works in online learning can be traced back to \cite{robbins1951asymptotically,blackwell1956analog,blackwell1954controlled,hannan1957approximation}. In terms of learning binary functions, Littlestone's dimension was first proposed in \cite{Littlestone87} to characterize online learning in the realizable (noiseless) setting. The agnostic (noisy) setting was first proposed by \cite{Haussler1992} in the statistical model and later extended to the online setting by \cite{littlestone1994weighted} who studied function-classes of bounded cardinality and then by \cite{BenDavidPS09} and \cite{rakhlin2010online} who provided both upper and lower bounds with only a logarithmic gap.

We note that Rakhlin, Sridharan, and Tewari~\cite{rakhlin2010online,rakhlin2015online,rakhlin2015sequential}, in the same line of work that proved the equivalence between online learning and sequential Rademacher complexity, analyzed uniform martingales laws of large numbers in the context of online learning.
These laws of large numbers are conceptually different from ours: roughly, they
assert uniform concentration of certain properties of martingales, where the uniformity
is over a given family of martingales.
In particular, in contrast with our work, there is no aspect of sub-sampling in these laws.
Below, we compare their techniques to those of this paper:
\begin{itemize}
	\item 
	\cite{rakhlin2010online} used a symmetrization argument to reduce from Martingale quantities relating to online learning to the Rademacher complexity. This does not reduce the effective sample size, which is what we achieve using the double sampling argument.
	\item \cite{rakhlin2010online} developed methods for analyzing the sequential Rademacher complexity. In particular, they developed a notion of covering numbers that is generally more powerful than the \emph{non-fractional} cover that uses dynamic sets, which was developed by \cite{BenDavidPS09} and was the baseline for our analysis. Yet, obtaining tight bound on the sequential Rademacher of Littlestone classes remained open.
	\item Reductions between sampling schemes did not appear in the above work as they did not study sampling.
\end{itemize}



\subsection{Streaming Algorithms}

The streaming model of computation is useful when analyzing massive datasets~\cite{AlonMS99}. 
There is a wide variety of algorithms for solving different tasks. One 
common method that is useful for various approximation tasks in streaming 
is random sampling. To approximate a function $f$, each element is sampled with some small probability $p$, and at the end, the function $f$ is computed on the sample. For tasks such as computing a center point of a high-dimensional dataset, where the objective is (roughly speaking) preserved under taking an $\eps$-approximation, this can result in improved space complexity and running time.
Motivated by streaming applications, Ben-Eliezer and Yogev \cite{BenEliezerYogev2020} 
proposed the adversarial sampling model that we study in this paper, and proved preliminary bounds on it. Their main result, a weaker quantitative analogue of our \Cref{thm:quantitative}, is an upper bound of $O((\log(|\E|) + \log(1/\delta))/\eps^2)$ for any finite family $\E$.


Streaming algorithms in the adversarial setting is an emerging topic that is not well understood. Hardt and Woodruff \cite{HardtW13} showed that linear sketches are inherently \emph{non-robust} and cannot be used to compute the Euclidean norm of its input (where in 
the static setting they are used mainly for this reason). Naor and Yogev 
\cite{NaorY15} showed that Bloom filters are susceptible to attacks by an 
adversarial stream of queries. On the positive side, two recent works \cite{BenEliezerJWY20,HassidimKMMS20} present generic compilers that transform non-robust randomized streaming algorithms into efficient adversarially robust ones, for various classical problems such as distinct elements counting and $F_p$-sampling, among others. 

%% file: preliminaries.tex
\section{Preliminaries}\label{sec:preliminaries}

\subsection{Basic Definitions: Littlestone Dimension and Sampling Schemes}
\label{sec:basic_defs}

\paragraph{Littlestone Dimension}
Let $X$ be a domain and let $\cal E$ be a family of subsets of $X$. 
The definition of the \emph{Littlestone Dimension} \cite{Littlestone87}, denoted $\Ldim(\mathcal{E})$, is given using mistake-trees: 
these are binary decision trees whose internal nodes are labelled by elements of $X$. 
Any root-to-leaf path corresponds to a sequence of pairs $(x_1, y_1), \ldots,(x_d, y_d)$, where $x_i$
is the label of the $i$'th internal node in the path, and $y_i = 1$ if the $(i+1)$'th node in the path is the right child of the $i$'th node, and otherwise $y_i = 0$. 
We say that a tree $T$ is shattered by $\mathcal{E}$ if for any
root-to-leaf path $(x_1, y_1), \ldots,(x_d, y_d)$ in $T$ there is $E \in \mathcal{E}$ such that $x_i \in R \iff  y_i=+1$, for all $i \le d$.
$\Ldim({\cal E})$ is the depth of the largest complete tree shattered by $\mathcal{E}$, with the convention that $\Ldim({\emptyset})=-1$. See \Cref{fig:shatteredtree} for an illustration.
\begin{figure}
\centering
\includegraphics[scale=0.3]{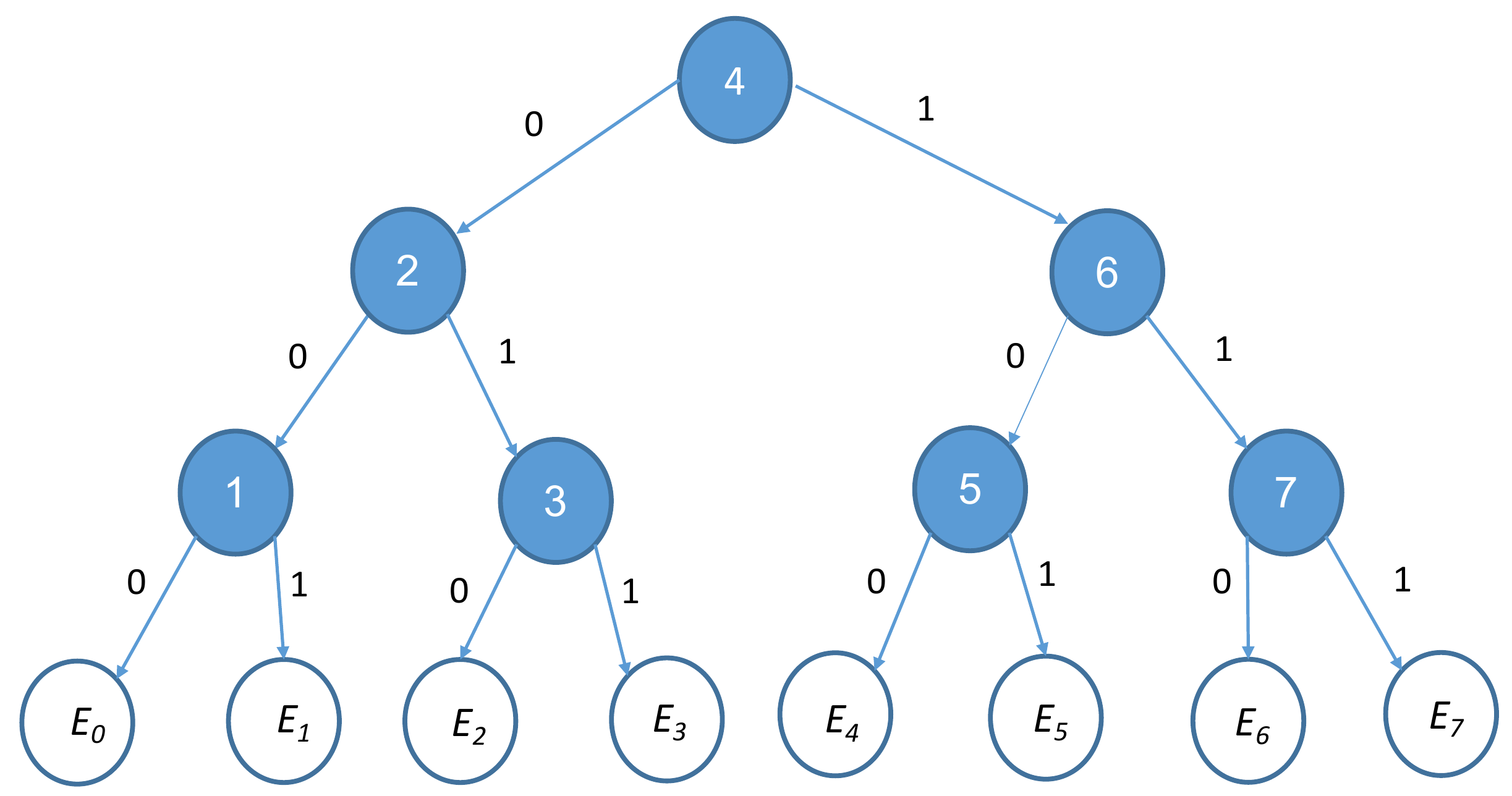}
\caption{\small{A tree shattered by the class $\mathcal{E}$ containing all thresholds $E_i\subseteq \{1,2,\ldots,7\}$, 
where $E_i=\{1,\ldots i\}$.}}\label{fig:shatteredtree}
\end{figure}

\paragraph{Sampling Algorithms}
Our results are achieved by three of the simplest and most commonly used sampling procedures: Bernoulli sampling, uniform sampling, and reservoir sampling.
\begin{itemize}
	\item \textbf{Bernoulli sampling}: $\Ber(n,p)$ samples the element arriving in each round $i \in [n]$ independently with probability $p$.
	\item   \textbf{uniform sampling}: $\Uni(n,k)$ randomly draws $k$ indices $1 \leq i_1 < \ldots < i_k \leq n$ and samples the elements arriving at rounds $i_1, \ldots, i_k$.\footnote{Note that the uniform sampler can be implement efficiently in an online way: after $i$ rounds, the probability that the next element $x_{i+1}$ will be sampled depends only on $i$, $n$, and the number of elements sampled so far.} 
	\item \textbf{Reservoir sampling}: $\Res(n,k)$ \cite{Vitter85} maintains a sample of size $k$ at all times using insertions and deletions: the first $k$ elements are always added to the sample, and for any $i > k$, with probability $k/i$ the element arriving in round $i$ is added to the sample while one of the existing elements (picked uniformly) is removed from the sample. 
\end{itemize}

\subsection{Notation}
\begin{itemize}
\item
	\textbf{Random variables} are denoted in a bold font. 
\item
	\textbf{Universal constants:} let $C,c,C',\dots$ denote universal numerical constants, that are independent of the problem parameters. Further, the values of these constants can change from the left-hand side to the right-hand side of some inequalities. We use $C_0,C_1,\dots$ to denote universal constants whose values are fixed.
\item \textbf{Sampling schemes:}
	We use $\rv{I}$ to denote the set of indices of elements sampled by the algorithm if they are sampled from a scheme with no deletions (e.g. Bernoulli or uniform sampling), and use $\rvvI = (\rvI_1,\dots,\rvI_n)$ to denote a sampling scheme with deletions, where $\rvI_j$ is the set of indices of elements retained after round $j \in [n]$.
\item \textbf{Adversaries:}
	We denote the set of adversaries that generate a stream of size $n$ by $\Adv_n$ and commonly denote adversaries by $\A$. We assume the adversary to be deterministic: since the sampler has a fixed strategy, we can make this assumption without limiting the generality of the theorem.
\item \textbf{The stream and subsets of it:}
	Let $\vx = (x_1,\dots,x_n)$ denote the stream where $x_i \in X$. We set $\vx(\A,I)$ to denote the stream presented by the adversary when the sampler samples elements indexed by $I$. Notice that $\vx(\A,I)_t$ depends only on $\A$ and $I\cap[t-1]$, since the $t$-th stream element is presented before the adversary knows if element $t$ is added to the sample. Given a subset $J \subseteq [n]$ we let $\vx_J = \{x_j \colon j \in J \}$. By abuse of notation, we may use $\vx$ also to denote the multiset $\{x_1,\dots,x_n\}$, allowing operations such as set intersection.
\end{itemize}

\subsection{Additional Central Definitions}
We define the central notions used in this proof, starting with the approximation rate in $\epsilon$-approximations:
\begin{definition}
	Given an arbitrary family $\E$ over a domain $X$, an adversary $\A\in \Adv_n$ and a subset $I \subseteq [n]$, let $\vx= \vx(\A,I)$ and define
	\[
	\App_{\A,I}(\E) = \App_{\A,I} := \max_{E \in \E} \lp| \frac{|E\cap \vx|}{n} - \frac{|E\cap \vx_I|}{k} \rp|.
	\]
\end{definition}

Secondly, define the notion of online discrepancy:
\begin{definition}
	Let $\E$ denote an arbitrary family over $X$, let $\A\in \Adv_{n}$, let $I \subseteq [n]$ and let $\vx = \vx(\A,I)$. The online discrepancy is defined by:
	\[
	\Disc_{\A,I}(\E) = \Disc_{\A,I} := \max_{E \in \E} \lp| |E\cap \vx_{I}| - |E\cap \vx_{[n]\setminus I}| \rp|.
	\]
\end{definition}

The sequential Rademacher complexity is just the expected discrepancy:
\begin{definition}
    The \emph{sequential Rademacher complexity} is defined as
    \[
    \Rad_T(\E)
    := \Ex_{\rvI\sim \Ber(n,1/2)}[\Disc_{\A,\rvI}(\E)].
    \]
\end{definition}
The next definition is used in the proof for $\epsilon$-nets. It defines an indicator to whether there exist $E \in \E$ that is well represented in the stream but not sufficiently represented in the sample.
\begin{definition}
	Fix $n,\om,\um \in \mathbb{N}$ such that $0 \le \um\le \om \le n$, let $\A \in \Adv_n$ and $I \subseteq [n]$. Denote $\vx = \vx(\A,I)$. Define
	\[
	\Net^n_{\A,I,\om,\um}(\E)
	= \Net^n_{\A,I,\om,\um}
	= \begin{cases}
	1 & \exists E \in \E,
	\ |\vx\cap E| \ge \om \text{ and }
	|\vx_I\cap E| \le \um \\
	0 & otherwise
	\end{cases}
	\]
\end{definition}
Notice that $\Net^n_{\A,I,\epsilon n,0}$ is an indicator to whether $\vx_I$ fails to be an $\epsilon$-net for $\vx$.

\subsection{Sampling Without Replacement}
Here we present technical probabilistic lemmas for sampling without replacement that are used in the proof. The proofs of these lemmas appear in Appendix \ref{sec:app-without-rep}.

For a sample $\rvI\sim \Uni(n,k)$ chosen without replacement, one would like to estimate the size of the intersection of $\rvI$ with any fixed set $U \subseteq [n]$. The next lemma bounds the variance of the intersection:
\begin{lemma} \label{lem:var-without-rep}
	Let $n,k$ such that $n\ge k$, let $U\subseteq [n]$, and let $\rvI\sim \Uni(n,k)$. Then, $\Ex[|U\cap \rvI|] = |U|k/n$ and
	$
	\mathrm{Var}(|U\cap \rvI|)
	\le |U|k/n.
	$
\end{lemma}
Further, exponential tail bounds can also be obtained:
\begin{lemma}[\cite{chatterjee2005concentration},\cite{bardenet2015concentration}]\label{lem:subset-intersection}
	Let $n,k \in \mathbb{N}$ such that $n\ge k$.
	Let $U \subseteq [n]$.
	Then, the following holds:
	\begin{enumerate}
	\item \label{itm:without-rep-big-set}
	For any $t \ge 0$,
	\[
	\Pr_{\rvI \sim \Uni(n,k)}\lp[\lp|\frac{|\rvI \cap U|}{k} - \frac{|U|}{n}\rp| \ge t \rp]
	\le 2\exp\lp(-2t^2 k\rp).
	\]
	\item \label{itm:without-rep-small-set}
	For any $\alpha\in [0,1]$,
	\[
	\Pr_{\rvI \sim \Uni(n,k)}\lp[\lp|\frac{|\rvI \cap U|}{k} - \frac{|U|}{n}\rp| \ge \alpha \frac{|U|}{n}\rp]
	\le 2\exp\lp(- \frac{\alpha^2 k|U|}{6n}\rp).
	\]
	\end{enumerate}
\end{lemma}

%% file: epsilon-approx.tex
\section{Epsilon Approximations}\label{sec:approx}

Below, we prove Theorem~\ref{thm:quantitative}. We start with a more formal statement of the theorem:
\begin{theorem} \label{thm:approx}
	Let $\E$ denote a family of Littlestone dimension $d$, let $\delta,\epsilon \in (0,1/2)$, $n\in\mathbb{N}$, $\A\in \Adv_n$ and $p \in [0,1]$.
	Define $k = \lfloor np \rfloor$.
	If $n \ge 3k$ then, for any $\delta > 0$,
	\[
	\Pr_{\rvI}\lp[\App_{\A,\rvI}(\E) \ge C \sqrt{\frac{d+\log(1/\delta)}{k}}\rp]
	\le \delta,
	\]
	where $\rvI$ is drawn either from $\Uni(n,k)$, $\Ber(n,p)$ or $\Res(n,k)$ and $C>0$ is a universal constant.
\end{theorem}
Note that the requirement $n\ge 3k$ is merely technical; when $n$ is smaller than that, one can just add all $n$ elements to the sample and obtain a $0$-approximation trivially.

Here we prove that the \emph{uniform sample} $\Uni(n,k)$ is an $\epsilon$-approximation and in Section~\ref{sec:reduction-sampling} we show a reduction to Bernoulli and reservoir sampling.
The bound for the uniform sampler consists of three  steps (see \Cref{sec:tech_overview}). The first step utilizes the double sampling argument, to bound the approximation error $\App_{\A,I}$ in terms of the discrepancy corresponding to a sampler $\Uni(2k,k)$:
\begin{lemma}\label{lem:approx-n-to-2k}
	Let $\E$ denote an arbitrary family of subsets from some universe. Fix $k,n\in\mathbb{N}$ such that $n\ge 2k$. Then, for any $t\ge 0$ and $\delta \in (0,1)$,
	\begin{equation}\label{eq:221}
	\max_{\A\in\Adv_n} \Pr_{\rvI \sim \Uni(n,k)}[\App_{\A,\rvI}(\E) > t]
	\le 2\max_{\A'\in\Adv_{2k}} \Pr_{\rvI'\sim \Uni(2k,k)}\lp[\Disc_{\A',\rvI'}(\E) > tk - \sqrt{Ck}\rp].
	\end{equation}
\end{lemma}
The technique of double sampling is presented in Section~\ref{sec:dbl-sampling} and the proof of Lemma~\ref{lem:approx-n-to-2k} is in Section~\ref{sec:pr-approx-dbl}.

Applying Lemma~\ref{lem:approx-n-to-2k}, we are left with bounding $\Disc_{\A,\rvI}$ where $\A\in \Adv_{2k}$ and $\rvI\sim \Uni(2k,k)$. However, it is easier to analyze a random sample $\rvI'\sim \Ber(2k,1/2)$ due to the independence of the coordinates. Hence, we prove the following lemma:
\begin{lemma}\label{lem:approx-comparison}
	Let $\E$ denote an arbitrary family. Then, for any $t > 0$ and $\delta > 0$,
	\begin{equation}\label{eq:31}
	\max_{\mathcal{A}\in\Adv_{2k}} \Pr_{\rvI\sim \Uni(2k,k)}\lp[\Disc_{\A,I}(\E) > t \rp]
	\le \max_{\mathcal{A}'\in\Adv_{2k}}\Pr_{\rvI'\sim\Ber(2k,1/2)}\lp[\Disc_{\A',I'}(\E) > t-\sqrt{Ck\log(1/\delta)}\rp] + \delta.
	\end{equation}
\end{lemma}
The technique of reducing between sampling schemes is in Section~\ref{sec:reduction-sampling} and the proof of Lemma~\ref{lem:approx-comparison} is in Section~\ref{sec:reduction-approx}.

The last step is to bound $\Disc_{\A,\rvI}$ for a sample $\rvI$ that is drawn Bernoulli $1/2$, for classes of bounded Littlestone dimension.
\begin{lemma} \label{lem:bnd-p-half}
	Let $\E$ be of Littleston dimension $d$ and let $\A\in \Adv_{2k}$. Then, for all $\delta > 0$,
	\[
	\Pr_{\rvI\sim \Ber(2k,1/2)}\lp[\Disc_{\A,\rvI}(\E) > C \sqrt{k(d+\log(1/\delta))} \rp]
	\le \delta.
	\]
\end{lemma}
\Cref{lem:bnd-p-half} is proved in \Cref{sec:seq-rad-proof-outline}.
Combining Lemma~\ref{lem:approx-n-to-2k}, Lemma~\ref{lem:approx-comparison} and Lemma~\ref{lem:bnd-p-half}, 
We conclude that the uniform sample as an $\epsilon$-approximation, as summarized below:
\begin{theorem}\label{thm:ind-of-n}
    Let $\E$ denote a family of Littlestone dimension $d$, let $\delta,\epsilon \in (0,1/2)$, $n\in\mathbb{N}$, $\A\in \Adv_n$ and $k \in \mathbb{N}$ such that $n \ge 2k$.
	Then, for any $\delta > 0$,
	\[
	\Pr_{\rvI\sim \Uni(n,k)}\lp[\App_{\A,\rvI}(\E) \ge C \sqrt{\frac{d+\log(1/\delta)}{k}}\rp]
	\le \delta,
	\]
	where $C>0$ is a universal constant.
\end{theorem}

%% file: epsilon-nets.tex
\section{Epsilon Nets}\label{sec:eps-nets}

We prove that the three sampling schemes discussed in this paper sample $\epsilon$-nets with high probability, as stated below:
\begin{theorem}\label{thm:eps-nets}
	Let $\E$ denote a family of Littlestone dimension $d$, let $
	\epsilon \in (0,1/2)$, $n\in\mathbb{N}$, $\A\in \Adv_n$ and $p \in [0,1]$.
	Define $k = \lfloor np \rfloor$.
	If $n \ge 3k$ then
	\[
	\Pr_{\rvI}[\Net^n_{\A,\rvI,\epsilon n, 0}(\E) =1]
	\le \lp(\frac{Ck}{d}\rp)^d \exp(-c\epsilon k)
	\]
	where $\rvI$ is drawn either $\Uni(n,k)$, $\Ber(n,p)$ or $\Res(n,k)$ and $C>0$ is a universal constant.
\end{theorem}
In this section, we prove \Cref{thm:eps-nets} for the uniform sampler.
Reductions to the other sampling schemes are given in Section~\ref{sec:reduction-sampling}.

Below the main lemmas are presented, starting with a reduction from a stream of size $n$ to $2k$, via the technique of double sampling, presented in Section~\ref{sec:dbl-sampling}:
\begin{lemma}\label{lem:dbl-nets}
	Let $\E$ be some family and $n,k\in \mathbb{N}$ be integers such that $n\ge 2k$ and $k \ge C/\epsilon$. 
	Then,
	\[
	\max_{\A\in\Adv_n}\Pr_{\rvI\sim \Uni(n,k)}\lp[\Net^n_{\A,\rvI,\epsilon\cdot n,0}(\E)=1\rp] 
	\le 2\max_{\A'\in\Adv_{2k}}\Pr_{\rvI'\sim \Uni(2k,k)}\lp[\Net^{2k}_{\A',\rvI',\epsilon/4\cdot 2k,0}(\E)=1\rp].
	\]
\end{lemma}
The proof appears in Section~\ref{sec:pr-dbl-net}. 
It would be desirable to replace the uniform sample $\Uni(2k,k)$ with a Bernoulli sample, since it selects each coordinate independently. In particular, we will show that the probability of $\Uni(2k,k)$ to fail to be an $\epsilon$-net is bounded in terms of the probability of $\Ber(2k,1/8)$. Intuitively, this follows from the fact that a sample drawn $\Uni(2k,k)$ nearly contains a sample $\Ber(2k,1/8)$ in some sense. The formal statement is below:
\begin{lemma}\label{lem:epsnet-reduction}
	Let $\E$ denote some family, $\epsilon \in (0,1)$, and $k \in \mathbb{N}$. Then,
	\begin{align*}
	&\max_{\A\in\Adv_{2k}}\Pr_{\rvI \sim \Uni(2k,k)}[\Net^{2k}_{\A,\rvI,\epsilon\cdot 2k,0}(\E) = 1] \\
	&\quad\le \max_{\A'\in\Adv_{2k}}\Pr_{\rvI'\sim \Ber(2k,1/8)}[\Net^{2k}_{\A',\rvI',\epsilon\cdot 2k, \epsilon/16 \cdot2k}(\E) = 1] + 2\exp(-c\epsilon k).
	\end{align*}
\end{lemma}
The framework to reduce between sampling schemes is presented in Section~\ref{sec:reduction-sampling} and the proof of Lemma~\ref{lem:epsnet-reduction} is presented in Section~\ref{sec:pr-reduction-epsnets}.
Lastly, we bound the error probability corresponding to Bernoulli sampling, for classes of bounded Littelstone dimension, using the technique of covering numbers presented in Section~\ref{sec:dynamic-sets}.
\begin{lemma}\label{lem:nets-final}
	Let $\E$ denote a family of Littlestone dimension $d$, let $m,k \in \mathbb{N}$ such that $m\le 2k$, and let $p \in (0,1)$. Then,
	\[
	\Pr_{\rvI\in \Ber(2k,p)}[\Net^{2k}_{\A,\rvI,m,mp/2}(\E) = 1] \le \lp(\frac{Ck}{d}\rp)^d \exp(-cmp).
	\]
\end{lemma}
The three lemmas stated above imply the following theorem, that is a special case of \Cref{thm:eps-nets} for the uniform sampler: 
\begin{theorem}\label{thm:nets-reformulation}
	Let $\E$ denote a family of Littlestone dimension $d$, let $
	\epsilon \in (0,1/2)$, $n\in\mathbb{N}$, $\A\in \Adv_n$ and $k \in \mathbb{N}$ that satisfies $n\ge 2k$.
	Then
	\[
	\Pr_{\rvI\sim \Uni(n,k)}[\Net^n_{\A,\rvI,\epsilon n, 0}(\E) =1]
	\le \lp(\frac{Ck}{d}\rp)^d \exp(-c\epsilon k),
	\]
	where $C>0$ is a universal constant.
\end{theorem}
\begin{proof}
	By Lemma~\ref{lem:dbl-nets}, 
	\[
	\max_{\A\in\Adv_n}\Pr_{\rvI\sim \Uni(n,k)}\lp[\Net^n_{\A,\rvI,\epsilon\cdot n,0}(\E)=1\rp] 
	\le 2\max_{\A'\in\Adv_{2k}}\Pr_{\rvI'\sim \Uni(2k,k)}\lp[\Net^{2k}_{\A',\rvI',\epsilon/4\cdot 2k,0}(\E)=1\rp].
	\]
	Applying Lemma~\ref{lem:epsnet-reduction} while substituting $\epsilon$ with $\epsilon/4$,
	\begin{align*}
	\max_{\A'\in\Adv_{2k}}\Pr_{\rvI'\sim \Uni(2k,k)}\lp[\Net^{2k}_{\A',\rvI',\epsilon/4\cdot 2k,0}(\E)=1\rp]
	&\le \max_{\A'\in\Adv_{2k}}\Pr_{\rvI'\sim \Ber(2k,1/8)}[\Net^{2k}_{\A',\rvI',\epsilon/4\cdot 2k, \frac{2k\epsilon}{64}} = 1] \\
	&+ 2\exp(-c\epsilon k).
	\end{align*}
	Applying Lemma~\ref{lem:nets-final} with $p=1/8$ and $m=\epsilon/4\cdot 2k$,
	\[
	\max_{\A'\in\Adv_{2k}}\Pr_{\rvI'\sim \Ber(2k,1/8)}[\Net^{2k}_{\A',\rvI',\epsilon/4\cdot 2k, \epsilon/64 \cdot2k} = 1]
	\le \lp(\frac{Ck}{d}\rp)^d \exp(-c\epsilon k).
	\]
\end{proof}

%% file: double-sampling.tex
\section{Double Sampling}\label{sec:dbl-sampling}

Let $n \in \mathbb{N}$ denote the stream length and assume that the sample is of size $k \le n/2$. 
If $n\gg k$, it may be difficult to analyze the sample directly, since each element is selected with small probability and the universe is very large. This section presents a framework to replace the stream of size $n$ with a stream of size $2k$. Then, this framework is used to prove Lemma~\ref{lem:approx-n-to-2k} and Lemma~\ref{lem:dbl-nets} in Section~\ref{sec:pr-approx-dbl} and Section~\ref{sec:pr-dbl-net}, respectively.

Let $\vx \in X^n$ denote the stream and $I \subseteq [n]$ be the index set of the sample, that has cardinality $|I|=k$.
Let $f \colon X^k\times X^n \to \{0,1\}$ denote some function and we view $f(\vx_I,\vx)$ as some indicator of whether $\vx_I$ fails to approximate the complete sample $\vx$. For example, $f(\vx_I,\vx)$ could indicate whether $\vx_I$ fails to be an $\eps$-approximation for $\vx$ with respect to some family $\E$.
Denote by $\rvvx = \vx(\A,I)$ the stream generated by the adversary $\A\in \Adv_n$ when the sample is indexed by $I$ and we would like to bound $\Pr_{\rvI \sim \Uni(n,k)}[f(\rvvx_I,\rvvx)=1]$. We would like to bound it by a different term that corresponds to only $2k$ elements. For this purpose, let $f'\colon X^k\times X^k\to\{0,1\}$ be another function, where $f(\vx_I,\vx_J)$ is an indicator of whether $\vx_I$ fails to approximate $\vx_J$. For example, $f'(\vx_I,\vx_J)$ can indicate whether $\vx_I$ fails to be an $\epsilon$-approximation for $\vx_J$.

Let $\A' \in \Adv_{2k}$, $\rvI' \sim \Uni(2k,k)$ and $\rvvx' =\vx(\A',\rvI')$ denote the stream of size $2k$ generated by $\A'$ with sample-index $\rvI'$. The following lemma gives a condition under which the probability that $f(\rvvx_{\rvI},\rvvx) = 1$ can be bounded in terms of the probability that $f(\rvvx'_{\rvI'},\rvvx'_{[n]\setminus\rvI'})=1$.

\begin{lemma}\label{lem:dbl-sampling}
	Let $\rvI\sim \Uni(n,k)$ and let $\rv{J}$ be distributed uniformly over all subsets of $[n]\setminus\rvI$ of size $k$, conditioned on $\rvI$. Let $f \colon X^k \times X^n \to \{0,1\}$. Let $\rvI' \sim \Uni(2k,k)$ and $f'\colon X^k \times X^k \to \{0,1\}$.
	
	Assume that for every $\vx$ and $I$ that satisfy $f(\vx_I,\vx)=1$, it further holds that
	\begin{equation}\label{eq:dblsamp-condition}
	\Pr_{\rv{J}}\lp[ f'(\vx_I,\vx_{\rv{J}}) = 1 \mid \rvI=I \rp]
	\ge 1/2.
	\end{equation}
	Then,
	\[
	\max_{\A\in\Adv_n}\Pr_{\rvI}[f(\vx(\A,\rvI)_{\rvI},\vx(\A,\rvI))=1]
	\le 2\max_{\A'\in\Adv_{2k}} \Pr_{\rvI'}[f'(\vx(\A',\rvI')_{\rvI'},\vx(\A',\rvI')_{[2k]\setminus \rvI'})=1].
	\]
\end{lemma}

To give some intuition on the condition \eqref{eq:dblsamp-condition}, assume again that $f(\vx_I,\vx)$ denotes an indicator of whether $\vx_I$ fails to be an $\epsilon$-approximation to $\vx$ and $f'(\vx_I,\vx_J)$ denotes whether $\vx_I$ fails to be an $\epsilon'$-approximation to $\vx_J$, where $\epsilon'$ is slightly larger than $\epsilon$. By concentration properties, if $f(\vx_I,\vx)=1$ then with high probability over $\rv{J}$, $f'(\vx_I,\vx_{\rv{J}})=1$ as well.

The proof of Lemma~\ref{lem:dbl-sampling} consists of two steps. In the first step, an index-set $\rv{J}$ is drawn uniformly at random from all the subsets of $[n]\setminus \rvI$ of size $k$, conditioned on $\rvI$. The set $\rv{J}$ is called a \emph{ghost sample}, as it is used only for the analysis and in particular, the adversary is unaware of $\rv{J}$.
The following lemma shows that under the condition \eqref{eq:dblsamp-condition}, we can bound $f(\rvvx_\rvI,\rvvx)$ in terms of $f'(\rvvx_\rvI,\rvvx_{\rv{J}})$.

\begin{lemma}\label{lem:dbl-aux1}
Let $f \colon X^k \times X^n \to \{0,1\}$, $f'\colon X^k \times X^k \to \{0,1\}$ and $\A \in \Adv_n$. Assume that for every $\vx$ and $I$ that satisfy $f(\vx_I,\vx)=1$, it further holds that
\begin{equation}\label{eq:61}
	\Pr_{\rv{J}}\lp[ f'(\vx_I,\vx_{\rv{J}}) = 1 \mid \rvI=I \rp]
	\ge 1/2.
\end{equation}
Then,
\begin{equation}\label{eq:63}
	\Pr_{\rvI}[f(\vx(\A,\rvI)_{\rvI},\vx(\A,\rvI))=1]
	\le 2 \Pr_{\rvI,\rv{J}}[f'(\vx(\A,\rvI)_{\rvI},\vx(\A,\rvI)_{\rv{J}})=1].
\end{equation}
\end{lemma}
\begin{proof}
	By \eqref{eq:61},
	\begin{align*}
	\Pr_{\rvI,\rv{J}}\lp[f'(\vx(\A,\rvI)_{\rvI},\vx(\A,\rvI)_{\rv{J}})=1\rp]
	&= \Ex_{\rvI} \lp[\Pr_{\rv{J}}\lp[ f'(\vx(\A,\rvI)_{\rvI},\vx(\A,\rvI)_{\rv{J}}) = 1 \mid \rvI \rp]\rp]\\
	&\ge \Ex_{\rvI} \lp[f(\vx(\A,\rvI)_{\rvI},\vx(\A,\rvI))
	\Pr_{\rv{J}}\lp[ f'(\vx(\A,\rvI)_{\rvI},\vx(\A,\rvI)_{\rv{J}}) = 1 \mid \rvI \rp]\rp]\\
	&\ge \Ex_{\rvI} \lp[f(\vx(\A,\rvI)_{\rvI},\vx(\A,\rvI))/2\rp]
	= \Pr_{\rvI} \lp[f(\vx(\A,\rvI)_{\rvI},\vx(\A,\rvI)_{\rv{J}})=1\rp]/2.
	\end{align*}
\end{proof}
Notice that the right hand side of \eqref{eq:63} corresponds to drawing two subsets of size $k$ from a stream of size $n$. It is desirable to bound this with a quantity that corresponds to partitioning a sample of size $2k$ to two subsets of size $k$. Essentially, this amounts to ignoring the elements out of $\rvI\cup\rv{J}$. Formally:
\begin{lemma}\label{lem:dbl-aux2}
	Let $f'\colon X^k\times X^k \to \{0,1\}$, let $\rvI$ and $\rv{J}$ be random subsets of $[n]$ as defined above and let $\rv{I}'\sim \Uni(2k,k)$. Then,
	\begin{equation}\label{eq:81}
	\max_{\A\in\Adv_n} \Pr_{\rvI,\rv{J}}[f'(\vx(\A,\rvI)_{\rvI},\vx(\A,\rvI)_{\rv{J}})]
	\le \max_{\A'\in\Adv_{2k}} \Pr_{\rvI'}[f'(\vx(\A',\rvI')_{\rvI'},\vx(\A',\rvI')_{[2k]\setminus \rvI'})=1].
	\end{equation}
\end{lemma}
\begin{proof}
	Let $\A$ denote the maximizer on the left hand side of \eqref{eq:81}.
	Let $U \subseteq [2n]$ be a set of size $2k$, and we will prove that 
	\begin{equation}\label{eq:conditioned-on-U}
	\Pr_{\rvI,\rv{J}}[f'(\vx(\A,\rvI)_{\rvI},\vx(\A,\rvI)_{\rv{J}})\mid \rvI \cup \rv{J} = U]
	\le \max_{\A'} \Pr_{\rvI'}[f'(\vx(\A',\rvI')_{\rvI'},\vx(\A',\rvI')_{[2k]\setminus \rvI'})].
	\end{equation}
	The proof of Lemma~\ref{lem:dbl-aux2} will then follow by taking an expectation over $U$. 
	
	The main idea to proving \eqref{eq:conditioned-on-U} is to match the subsets $I\subseteq U$ with the subsets $I' \subseteq [2k]$. One can define an adversary $\A_U\in \Adv_{2k}$ that simulates the behavior of $\A$ on $U$, hence matching the probability that $f'=1$.
	In particular, $\A_U$ simulates the selections of $\A$ on the set $U$, while skipping all the elements not in $U$. Formally, denote $U = (i_1,\dots,i_{2k})$ where $i_1<i_2 \cdots < i_{2k}$ and $U_{I'} = \{i_j \colon j \in I'\}$. Then $\A_U$ is defined to satisfy $\vx(\A_U,I') := \vx(\A,U_{I'})_{U}$.
	This implies that
	\[
	\vx(\A_U,I')_{I'} = \vx(\A,U_{I'})_{U_{I'}}; \quad
	\vx(\A_U,I')_{[2k]\setminus I'} = \vx(\A,U_{I'})_{U_{[2k]\setminus I'}}.
	\]
	Hence,
	\begin{equation}\label{eq:225}
	f'(\vx(\A,U_{I'})_{U_{I'}}, \vx(\A,U_{I'})_{U_{[2k]\setminus I'}})
	= f'(\vx(\A_U,I')_{I'},\vx(\A_U,I')_{[2k]\setminus I'}).
	\end{equation}
	Notice that the joint distribution of $(U_{\rvI'},U_{[n]\setminus\rvI'})$ taken over $\rvI'\sim \Uni(2k,k)$, is the same as the joint distribution of $(\rvI,\rv{J})$, conditioned on $\rvI\cup \rv{J}=U$. In combination with \eqref{eq:225}, this implies that
	\begin{align*}
	&\Pr[f'(\vx(\A,\rvI)_{\rvI}, \vx(\A,\rvI)_{\rv{J}})=1 \mid \rvI\cup \rv{J} = U] 
	= \Pr[f'(\vx(\A,U_{\rvI'})_{U_{\rvI'}}, \vx(\A,U_{\rvI'})_{U_{[2k]\setminus \rvI'}})=1] \\
	&\quad = \Pr[f'(\vx(\A_U,\rvI')_{\rvI'},\vx(\A_U,\rvI')_{[2k]\setminus \rvI'})=1]
	\le \max_{\A'} \Pr[f'(\vx(\A',\rvI')_{\rvI'},\vx(\A',\rvI')_{[2k]\setminus \rvI'})=1].
	\end{align*}
	This proves \eqref{eq:conditioned-on-U}, and concludes the proof.
\end{proof}
The proof of Lemma~\ref{lem:dbl-sampling} follow directly from Lemma~\ref{lem:dbl-aux1} and Lemma~\ref{lem:dbl-aux2}.

\subsection{Proof of Lemma~\ref{lem:approx-n-to-2k}}\label{sec:pr-approx-dbl}

We start with the following auxiliary probabilistic lemma:
\begin{lemma}\label{lem:aux-n-2k}
	Let $I \subseteq [n]$ of size $|I|=k$, let $\vx$ and let $t \ge 0$ be such that 
	\begin{equation}\label{eq:96}
	\max_{E \in \E} \lp| \frac{|E\cap \vx_I|}{k} - \frac{|E\cap \vx|}{n} \rp| \ge t.
	\end{equation}
	Let $\rv{J}$ be distributed uniformly over all subsets of $[n]\setminus I$ of size $k$. Then, with probability at least $1/2$,
	\[
	\max_{E \in \E} \lp| \frac{|E\cap \vx_I|}{k} - \frac{|E\cap \vx_{\rv{J}}|}{k} \rp| \ge t - C/\sqrt{k}.
	\]
\end{lemma}
\begin{proof}
	Let $E_0$ be a maximizer of \eqref{eq:96}.
	We will assume that
	\begin{equation}\label{eq:282}
	\frac{|E_0\cap \vx|}{n} \ge \frac{|E_0\cap \vx_I|}{k}.
	\end{equation}
	and the proof follows similarly in the other case.
	Applying Lemma~\ref{lem:subset-intersection} (item \ref{itm:without-rep-big-set}) with $U = \{ i \in [n]\setminus I \colon x_i \in E_0 \}$, $n=n-k$ and $\rv{I} = \rv{J}$, we have that with probability at least $1/2$ over $\rv{J}$,
	\[
	\frac{|E_0\cap \vx_{\rv{J}}|}{k}
	= \frac{|U\cap\rv{J}|}{k}
	\ge \frac{|U|}{n-k} - C\sqrt{1/k}
	= \frac{|E_0\cap \vx_{[n]\setminus I}|}{n-k} - C\sqrt{1/k}
	\ge \frac{|E_0\cap \vx_{[n]}|}{n} - C\sqrt{1/k},
	\]
	where the last inequality follows from \eqref{eq:282}.
	In particular,
	\[
	\frac{|E_0\cap \vx_J|}{k} - \frac{|E_0\cap \vx_I|}{k}
	\ge t - C/\sqrt{k},
	\]
	which concludes the proof.
\end{proof}

\begin{proof}[Proof of Lemma~\ref{lem:approx-n-to-2k}]
	First, apply Lemma~\ref{lem:dbl-sampling} with the function
	$f(\vx_I,\vx)$ that is the indicator of the event
	\[
	\max_{E \in \E} \lp| \frac{|E\cap \vx_I|}{k} - \frac{|E\cap \vx|}{n} \rp| \ge t
	\]
	and $f'(\vx_I,\vx_J)$ that is the indicator of
	\[
	\max_{E \in \E} \lp| \frac{|E\cap \vx_I|}{k} - \frac{|E\cap \vx_J|}{k} \rp| \ge t - C_0/\sqrt{k},
	\]
	where $C_0> 0$ corresponds to the constant $C$ in Lemma~\ref{lem:aux-n-2k}.
	The condition in Lemma~\ref{lem:dbl-sampling} is satisfied from Lemma~\ref{lem:aux-n-2k} and one derives that
	\begin{align*}
	\max_{\A \in \Adv_n} \Pr_{\rvI}[\App_{\A,\rvI} \ge t]
	&= \max_{\A\in\Adv_n} \Pr_{\rvI}[f(\vx(\A,\rvI)_{\rvI},\vx(\A,\rvI)) = 1]\\
	&\le 2 \max_{\A'\in\Adv_{2k}} \Pr[f'(\vx'(\A',\rvI')_{\rvI'}, \vx'(\A',\rvI')_{[2k]\setminus \rvI'})] \\
	&= 2 \max_{\A'\in\Adv_{2k}} \Pr[\Disc_{\A',\rvI'}/k \ge t-C_0/\sqrt{k}].
	\end{align*}
\end{proof}

\subsection{Proof of Lemma~\ref{lem:dbl-nets}}\label{sec:pr-dbl-net}
We start with the following probabilistic lemma:
\begin{lemma}\label{lem:dbl-net-aux}
	Let $\epsilon\in (0,1)$, $k \ge C/\epsilon$, $\vx\in X^n$ and $I\subseteq [n]$, $|I|=k$, be such that
	\[
	\exists E \in \E,
	\ |\vx \cap E| \ge \epsilon n \text{ and }
	\vx_I\cap E = \emptyset.
	\]
	Let $\rv{J}$ be drawn uniformly from the subsets of $[n]\setminus I$ of size $k$. Then, with probability at least $1/2$,
	\[
	\exists E \in \E,
	\ |\vx_{I\cup J}\cap E| \ge \epsilon k/2 \text{ and }
	\vx_I\cap E = \emptyset.
	\]
\end{lemma}
\begin{proof}
	Let $E_0 \in \E$ be any set such that $|\vx\cap E_0| \ge \epsilon n$ and $\vx_I\cap E_0 = \emptyset$. It suffices to show that with probability at least $1/2$, $|\vx_{\rv{J}} \cap E_0| \ge \epsilon k/2$. To prove this, apply Lemma~\ref{lem:subset-intersection} (item \ref{itm:without-rep-small-set}) with $\rvI=\rv{J}$, $U = \{ i \in [n]\setminus I \colon x_i \in E_0 \}$, $n=n-k$ and $\alpha=1/2$. Notice that $|U| \ge \epsilon n \ge \epsilon(n-k)$ and we derive that
	\begin{align*}
	&\Pr\lp[ |\vx_{\rv{J}} \cap E_0| < \frac{\epsilon k}{2}\rp]
	= \Pr\lp[\frac{|\rv{J}\cap U|}{k}< \frac{\epsilon}{2}\rp]
	\le \Pr\lp[\frac{|\rv{J}\cap U|}{k} \le \frac{|U|}{2(n-k)}\rp]\\
	&\quad \le \Pr\lp[\lp|\frac{|\rv{J}\cap U|}{k} - \frac{|U|}{n-k}\rp| \ge \frac{|U|}{2(n-k)}\rp]
	\le 2 \exp\lp(-\frac{k|U|/4}{6(n-k)} \rp)
	\le 2 \exp\lp(-\frac{k\epsilon/4}{6} \rp)
	\le 1/2,
	\end{align*}
	where the last inequality follows from the assumption that $k\ge C\epsilon$ for a sufficiently large $C>0$.
\end{proof}
\begin{proof}[Proof of Lemma~\ref{lem:dbl-nets}]
	We will apply Lemma~\ref{lem:dbl-sampling} with
	$
	f(\vx_I,\vx)
	$
	being the indicator of
	\[
	\exists E \in \E,
	\ |\vx\cap E| \ge \epsilon n \text{ and }
	\vx_I\cap E = \emptyset
	\]
	and $f'(\vx_I,\vx_J)$ the indicator of 
	\[
	\exists E \in \E,
	\ |\vx_{I\cup J}\cap E| \ge \epsilon k/2 \text{ and }
	\vx_I\cap E = \emptyset.
	\]
	The condition of Lemma~\ref{lem:dbl-sampling} holds from Lemma~\ref{lem:dbl-net-aux} and the proof follows.
\end{proof}

%% file: dynamic-sets.tex
\section{Covering Numbers}\label{sec:dynamic-sets}


\subsection{Overview}
While studying online algorithms, it is natural to consider the following objects:
\begin{definition}[Dynamic-Set.]
A dynamic set is an online algorithm $\dS$ which is defined on input sequences $\vx\in X^n$.
	At each time-step $t\leq n$, the algorithm decides whether to retain $x_t$ or to discard it.
	The decision whether to retain/discard $x_t$ may depend on the elements $x_1,\ldots,x_t$ which were observed up to time $t$. 
	The \underline{trace} of $\dS$ with respect to an input sequence $\vx$ is defined by
	\[\dS(\vx) = \{ x_t : x_t\text{ was retained by $\dS$}, t\leq n\}.\]
	and is viewed either as a set or as an ordered set.
	Further, we say that $|\dS|\le m$ if $\dS$ cannot retain more than $m$ elements.
	We stress that the decision whether to retain/discard an item is \underline{not} reversible:
	retained (discarded) items can not be discarded (retained) in the future.
\end{definition}

Given a family $\E$, we would like to cover it using dynamic sets, as defined below:
\begin{definition}
	Let $\E$ be some family and let $\mathcal{N}$ denote some finite collection of dynamic sets. We say that $\mathcal{N}$ is an $\epsilon$-cover for $\E$ if for every input sequence $\vx$ and every $E \in \E$ there exists $\dS \in \mathcal{N}$ such that
	\[
	|(E \cap \{\vx\}) \triangle \dS(\vx)| \le \epsilon^2 n
	\]
	where $\Delta$ is the symmetric difference of sets.
	Further, define the \emph{covering number at scale $\epsilon$}, $N(\E,\epsilon)$, as the smallest cardinality of an $\epsilon$-cover for $\E$.
\end{definition}
We can obtain bounds on the covering numbers at scale $0$ for Littlestone families, via a known argument:
\begin{lemma}[Covering Littlestone Families with Few Dynamic Sets]\label{prop:cover}
	Let ${\cal E}$ be a family of subsets of $X$ with $\Ldim({\cal E})=d$ and let $n\in\mathbb{N}$.
	Then, $N(\E,0) \le {n \choose \leq d}$.
	Moreover, this is tight for $n=d$, where $N(\E,0) = 2^d = \binom{n}{\le d}$.
\end{lemma}
While covering numbers at scale $0$ can be used to derive bounds for $\epsilon$-approximation and $\epsilon$-nets for the Bernoulli sampler $\Ber(n,p)$ with $p$ constant, these bounds are sub-optimal for epsilon-approximations. Improved bounds can be obtained by computing covering numbers at scale $\epsilon>0$. While we do not know how derive better than $\binom{n}{\le d}$ even for scales $\epsilon>0$, it is possible to obtain improved bounds on \emph{fractional covering numbers}, which is a notion that we define below and can replace the covering numbers:
\begin{definition}
	Let $\mu$ denote a probability measure over dynamic sets. We say that $\mu$ is an \emph{$(\epsilon,\gamma)$-fractional cover} for $\cal{E}$ if for every sequence $\vx$ and every $E \in \cal{E}$,
	\[
	\mu\lp(\lp\{ \dS \colon |(E \cap \{ \vx \})\triangle \dS(\vx)|
	\le \epsilon^2 n\rp\}\rp) \ge 1/\gamma.
	\]
	Define the \emph{fractional covering number at scale $\epsilon$},
	$N'({\cal{E}},\epsilon)$, as the minimal value of $\gamma$ such that there exists an $(\epsilon,\gamma)$ fractional cover for $\cal{E}$.
\end{definition}
Notice that $N'(\E,\epsilon) \le N(\E,\epsilon)$: if $\cal{C}$ is an $\epsilon$-cover, then $N'({\cal{E}},\epsilon) \le |\cal{C}|$, by taking $\mu$ to be the uniform distribution over ${\cal C}$.

We can obtain the following bound on the fractional covering numbers for Littlestone classes:
\begin{lemma}\label{lem:fractional-cover-littlestone}
	It holds that $N'({\cal E},\epsilon) \le (C/\epsilon)^{2d}$, for some universal $C>0$.
\end{lemma}

Next, we apply bounds on covering numbers for epsilon approximation and epsilon nets:

\subsubsection{Epsilon Approximation and Sequential Rademacher}\label{sec:seq-rad-proof-outline}

The following bound can be derived based on $0$-nets:
\begin{lemma}\label{lem:eps-cover-bound}
	Let $\A\in\Adv_{2k}$, $\rvI\sim \Ber(2k,1/2)$, $\delta \in (0,1/2)$ and let $\E$ be any family over some universe. Then, with probability $1-\delta$,
	\[
	\Disc_{\A,\rvI}(\E) \le C \sqrt{k(\log N(\E, 0)+\log 1/\delta)}.
	\]
\end{lemma}
The proof is via a simple union bound. 
In combination with the bound on the $0$-cover of Littlestone classes (Lemma~\ref{prop:cover}), this derives that with probability $1-\delta$,
\[
\Disc_{\A,\rvI} \le \sqrt{k(d\log k + \log(1/\delta))},
\]
which implies a sample complexity of 
\[
O\lp( \frac{d\log(1/\epsilon)+\log(1/\delta)}{\epsilon^2} \rp).
\]
To derive sharper bounds, one can use covering numbers at scales $\epsilon>0$. Since we only have fractional covering numbers for Littlestone classes, we present the following lemma that derives a bound based on them:
\begin{lemma}\label{lem:fractional-cover-bound}
Let $\A\in\Adv_{2k}$, $\rvI\sim \Ber(2k,1/2)$, $\delta \in(0,1/2)$ and let $\E$ be some family. Then, with probability $1-\delta$,
\[
\Disc_{\A,\rvI}(\E) \le C \sqrt{k} \lp( \int_{0}^1\sqrt{\log N'(\E, \epsilon)}d\epsilon +\sqrt{\log 1/\delta}\rp).
\]
\end{lemma}
This has the same form as the celebrated Dudley's integral but here we extend it to fractional covering numbers.

Using Lemma~\ref{lem:fractional-cover-littlestone} and Lemma~\ref{lem:fractional-cover-bound}, Lemma~\ref{lem:bnd-p-half} immediately follows. Indeed,
\begin{align*}
\frac{\Disc_{\A,I}}{\sqrt{k}} 
&\le C\lp( \int_{0}^1\sqrt{\log N'(\E,\epsilon)}d\epsilon +\sqrt{\log 1/\delta}\rp)\\
&\le C\lp( \int_{0}^1\sqrt{d\log 1/\epsilon}d\epsilon +\sqrt{\log 1/\delta}\rp)
\le C \lp(\sqrt{d} +\sqrt{\log 1/\delta}\rp).
\end{align*}

\subsubsection{Epsilon Nets}

We have the following statement:
\begin{lemma}\label{lem:concentration-nets}
	Let $\E$ be any family, $\A\in\Adv_{2k}$, $\rvI\sim \Ber(2k,p)$ and let $m \in [2k]$. Then,
	\[
	\Pr_{\rvI}\lp[\Net^{2k}_{\A,\rvI,m,mp/2}=1\rp]
	\le N(\E,0) \exp(-cmp).
	\]
\end{lemma}

Lemma~\ref{lem:nets-final} follows directly by applying the bound on the covering numbers at scale $0$ for Littlestone classes, presented in Lemma~\ref{prop:cover}. 

\subsubsection{Organization}

Section~\ref{sec:cover-littlestone} contains the proofs of Lemma~\ref{prop:cover} and Lemma~\ref{lem:fractional-cover-littlestone} on the covering numbers for Littlestone classes; Section~\ref{sec:eps-approx-via-cover} contains the proofs of Lemma~\ref{lem:eps-cover-bound} and Lemma~\ref{lem:fractional-cover-bound} on proving $\epsilon$-approximations via covering numbers; and Section~\ref{sec:eps-nets-via-cover} contains the proof of Lemma~\ref{lem:concentration-nets} on proving $\epsilon$-nets via covering numbers.

\subsection{Covering for Littlestone Classes}\label{sec:cover-littlestone}
Section~\ref{sec:cover-0} contains the proof of Lemma~\ref{prop:cover} on the covering numbers at scale $0$, that is based on a known arguments; and Section~\ref{sec:cover-eps} contains the proof of Lemma~\ref{lem:fractional-cover-littlestone}, that builds on the machinery presented in Section~\ref{sec:cover-0}.

\subsubsection{Proof of Lemma~\ref{prop:cover}}\label{sec:cover-0}

Before we prove \Cref{prop:cover}, let us make a couple of comparisons to related literature.
\begin{remark}
	It is fair to note that the proof of this proposition exploits standard and basic ideas from the online learning literature.
	In particular, the constructed dynamic-sets hinge on variants of the {\it Standard Optimal Algorithm} by \cite{Littlestone87},
	and utilize its property of being a {\it mistake-driven} algorithm\footnote{A mistake-driven algorithm updates its internal state only when it makes mistakes.}.
	However, for the benefit of readers who are less familiar with this literature, 
	we provide here a self-contained proof 
	and modify some of the terminology/notation from the language of online learning
	to the language of $\eps$-nets/approximations.
\end{remark}
\begin{remark}
This comment concerns a connection with the celebrated Sauer-Shelah-Perles (SSP) Lemma \cite{Sauer72Lemma}.
	Note that the SSP Lemma is equivalent to a variant of \Cref{prop:cover} in which two quantifiers are flipped.
	Indeed, the SSP Lemma asserts that for every $x_1,\ldots, x_n$ there are at most ${n \choose \leq d}$ sets 
	that realize all possible intersection patterns of the sets in ${\cal E}$ with  $\{x_1,\ldots, x_n\}$.
	That is, if one allows the sets $\dS_i$ to be chosen \underline{after} seeing the entire input-sequence $x_1,\ldots,x_n$
	then the conclusion in \Cref{prop:cover} extends to VC classes 
	(which can have an unbounded Littlestone dimension, as witnessed by the class of thresholds).
\end{remark}
	
\begin{proof}[Proof of \Cref{prop:cover}]
We begin with the upper bound.
The definition of the dynamic sets $\dS_i$ exploits the following property of Littlestone families.
Let $\cal E$ be a family with $\Ldim({\cal E})< \infty$, let $x\in X$, and consider the two ``half-families'' 
\[{\cal E}_{\not\ni x} = \{E\in {\cal E} : x\notin E\}, \quad {\cal E}_{\ni x} = \{E\in {\cal E} : x\in E\}.\]
The crucial observation is that if ${\cal E}\neq\emptyset$ then for every $x\in X$:
\begin{equation}\label{eq:littlestone}
\text{$\Ldim({\cal E}_{\not\ni x})<\Ldim({\cal E})$ or $\Ldim({\cal E}_{\ni x})<\Ldim({\cal E})$.}
\end{equation}
Indeed, otherwise we have $\Ldim({\cal E}_{\not \ni x})=\Ldim({\cal E}_{\ni x})=\Ldim({\cal E})=:d$
	which implies that $\cal E$ shatters the following tree of depth $d+1$:
	the root is labelled with $x$, and the left and right subtrees of the root are trees which witness that the dimensions of ${\cal E}_{\not\ni x}$ 
	and ${\cal E}_{\ni x}$ equal $d$. However, since $\Ldim({\cal E})=d$, this is not possible.

\paragraph{Littlestone Majority Vote.}
\Cref{eq:littlestone} allows to define a notion of majority-vote of a (possibly infinite) family $\cal E$ with a finite Littlestone dimension.
	The intuition is that if $x$ is such that $\Ldim({\cal E}_{\ni x}) = d$ then by \Cref{eq:littlestone}
	it must be that $\Ldim({\cal E}_{\not\ni x}) < d$ and therefore $\Ldim({\cal E}_{\ni x}) >\Ldim({\cal E}_{\not\ni x})$ 
	which we interpret as if $x$ is contained in a ``majority'' of the sets in $\cal E$.
	Similarly, $\Ldim({\cal E}_{\not\ni x}) = d$ is interpreted as if most sets in $\cal E$ do not contain $x$.
	This motivates the following definition
\begin{equation}\label{eq:majority}
\maj({\cal E}) = \{x: \Ldim({\cal E}_{\ni x}) = d\},
\end{equation}
with the convention that $\maj(\emptyset)=\emptyset$.
Observe that $\maj({\cal E})$ shares the following property with the standard majority-vote over finite families:
let $x\in X$ and assume ${\cal E}\neq \emptyset$. Then, 
\begin{align}\label{eq:majprop}
\Bigl((\forall E\in {\cal E}): x\in E\Bigr) \implies x\in \maj({\cal E}) \quad \text{and}\quad
\Bigl((\forall E\in {\cal E}): x\notin E\Bigr) \implies x\notin \maj({\cal E}).
\end{align}
That is, if the sets in ${\cal E}$ agree on $x$ unanimously, then $\maj({\cal E})$ agrees with them on $x$.
We comment that this definition is the basis of the {\it Standard Online Algorithm} which witnesses the online-learnability
of Littlestone classes in the mistake-bound model~\cite{Littlestone87}. 

\vspace{1mm}

We are now ready to define the required family of ${n \choose \leq d}$ dynamic sets.
	Each dynamic set $\dS_I$ is indexed by a subset $I\subseteq [n]$ of size $\lvert I\rvert\leq d$.
	(Hence there are ${n \choose \leq d}$ dynamic sets.) 	
	Below is the pseudo-code of $\dS_I$ for $I\subseteq [n]$.

\begin{center}
\noindent\fbox{
\parbox{.99\columnwidth}{
	\begin{center}\underline{\bf The Dynamic Set $\dS_I$}\end{center}
Let $\mathcal{E}$ be a family with $\Ldim({\cal E})=d$, and let $I\subseteq [n]$.\\
Let $x_1,\ldots, x_n$ denote the (adversarially-produced) input sequence.
\begin{enumerate}
\item Initialize ${\cal E}_0^I = {\cal E}$. 
\item For $t=1,\ldots, n$:
\begin{enumerate}
\item If $t\notin I$ then set ${\cal E}_{t}^I= {\cal E}_{t-1}^I$. 
\item Else, set
\[{\cal E}_{t}^I=
\begin{cases}
 ({\cal E}_{t-1}^I)_{\not\ni x_t} &x_t \in\maj({{\cal E}_{t-1}^I}),\\
 ({\cal E}_{t-1}^I)_{\ni x_t}       &x_t \notin \maj({{\cal E}_{t-1}^I}).
\end{cases}
\]
\item Retain $x_t$ if and only if  $x_t\in \maj({{\cal E}_{t}^I})$.
\end{enumerate}
\end{enumerate}
     }}
\end{center}
Observe the following useful facts regarding $\dS_I$:
\begin{enumerate}
\item The sequence of families $\{{\cal E}_t^I\}_{t=0}^n$ is a chain: ${\cal E}_0^I\supseteq {\cal E}_1^I\supseteq \ldots \supseteq {\cal E}_n^I$.
\item A strict containment ${\cal E}_{t-1}^I \supsetneq {\cal E}_{t}^I$ occurs only if $t\in I$.
\item Whenever a strict containment ${\cal E}_{t-1}^I \supsetneq {\cal E}_{t}^I$ occurs then also $\Ldim({\cal E}_{t-1}^I) > \Ldim({\cal E}_{t}^I)$.
(By \Cref{eq:littlestone,eq:majority}.)
\end{enumerate}

To complete the proof it remains to show that for every $\vx$ and every $E\in {\cal E}$ there exists $I\subseteq [n]$, with $\lvert I\rvert\leq d$
	such that 
	\begin{equation}\label{eq:remains}
	(\forall t\leq n): x_t\in E \iff x_t\in \dS_I(\vx).
	\end{equation}
	We construct the set $I=I(E)$ in a parallel fashion to the above process:
	
\begin{center}
\noindent\fbox{
\parbox{.99\columnwidth}{
	\begin{center}\underline{\bf The Index Set $I=I(E)$}\end{center}
Let $E\in {\cal E}$ and let $x_1,\ldots,x_n$ denote the input sequence.
\begin{enumerate}
\item Initialize ${\cal E}_0^E = {\cal E}$ and $I=\emptyset$. 
\item For $t=1,\ldots, n$:
\begin{enumerate}
\item If $E$ and $\maj({{\cal E}_{t-1}^E})$ agree on $x_t$ (i.e.\ $x_t\in E \iff x_t\in \maj({{\cal E}_{t-1}^E})$) then set ${\cal E}_{t}^E= {\cal E}_{t-1}^E$.
\item Else, add $t$ to $I$ and set
\[{\cal E}_{t}^E=
\begin{cases}
 ({\cal E}_{t-1}^E)_{\not\ni x_t} &x_t \in\maj({{\cal E}_{t-1}^E}) \land x_t\notin E,\\
 ({\cal E}_{t-1}^E)_{\ni x_t}       &x_t \notin\maj({{\cal E}_{t-1}^E}) \land x_t\in E.
\end{cases}
\]
\end{enumerate}
\item Output $I=I(E)$.
\end{enumerate}

     }}
\end{center}	
Note that by construction, ${\cal E}_{t}^E = {\cal E}_{t}^I$ for every $t\leq n$, and $E\in{\cal E}_t^E$ for all $t$.
We need to show that the constructed set $I$ satisfies \Cref{eq:remains} and that $\lvert I\rvert \leq d$.
For the first part, note that for every $t\leq n$ :
\begin{align*}
x_t \in \dS_I(\vx) &\iff x_t\in \maj({{\cal E}_{t}^I}) \tag{by definition of $\dS_I$}\\
		   &\iff x_t\in \maj({\cal E}_{t}^E)\tag{since ${\cal E}_{t}^E = {\cal E}_{t}^I$}\\
		   &\iff x_t\in E,\tag{see below}
\end{align*}
where the last step follows because all the sets $E'\in{{\cal E}_{t}^E}$ agree with $E$ on $x_t$.
Thus, by\footnote{Note that ${\cal E}_{t}^E\neq\emptyset$ because $E\in{\cal E}_t^E$.} Equation \eqref{eq:majprop} also $\maj({{\cal E}_{t}^E})$
agrees with $E$ on $x_t$, which amounts to the last step.

To see that $\lvert I\rvert \leq d$, consider the chain
\[{\cal E}_{0}^E \supseteq {\cal E}_{1}^E\supseteq\ldots \supseteq {\cal E}_{n}^E.\]
Note that strict containments ${\cal E}_{t-1}^E \supsetneq {\cal E}_{t}^E$ occurs only if $t\in I$,
and that whenever such a strict containment occurs, we have $\Ldim({\cal E}_{t-1}^E) > \Ldim({\cal E}_{t}^E)$.
Therefore, since $\Ldim({\cal E}_{0}^E) = d$ and $\Ldim({\cal E}_{n}^E) \geq 0$,
it follows that $\lvert I\rvert\leq d$ as required.

\vspace{2mm}
It remains to prove the lower bound. Let $D=\{\dS_i :1\leq i<2^d\}$ be a family of less than $2^d$ dynamic sets. 
	Pick a tree $\cal T$ of depth~$d$ which is shattered by ${\cal E}$ and define an adversarial sequence $x_1,\ldots,x_d$ as follows:
	\begin{center}
\noindent\fbox{
\parbox{.99\columnwidth}{
	\begin{center}\underline{\bf The Adversarial Sequence $x_1,\ldots,x_d$}\end{center}
\begin{enumerate}
\item Set $\mathcal{T}_1=\mathcal{T}, D_1=D, {\cal E}_1 = {\cal E}$, and $i=1$. 
\item For $i=1,\ldots , d$
\begin{itemize}
\item[(i)] Set $x_i$ to be the item labelling the root of $\mathcal{T}_i$.
\item[(ii)] If less than half of the dynamic sets $\dS_j\in D_i$ retain $x_i$
	then continue the next iteration with $\mathcal{T}_{i+1}$ being the right subtree of $\mathcal{T}_i$ 
	(which corresponds to the sets containing $x_i$), 
	and with ${\cal E}_{i+1} = \{E\in {\cal E}_i : x\in E \}$ and $D_{i+1}=\{\dS_j \in D_i : x_i\in \dS_j\}$.
\item[(iii)] Else, continue to the next iteration with $\mathcal{T}_{i+1}$ being the left subtree of $\mathcal{T}_i$, 
	and with ${\cal E}_{i+1} = \{E\in {\cal E}_i : x\notin E \}$ and $D_{i+1}=\{\dS_j \in D_i : x_i\notin \dS_j\}$.
\end{itemize}
\end{enumerate}
     }}
\end{center}
Note that ${\cal E}_i$ contains all the sets in ${\cal E}$ that are consistent\footnote{$\dS_j$ is consistent with the path corresponding to $x_1,\ldots, x_d$ means that $\dS_j(x_1,\ldots,x_n)$ contains $x_i$ if and only if $x_{i+1}$ labels the right child of the node labelled $x_i$. } with the path corresponding to $x_1,\ldots, x_{i-1}$,
	and similarly $D_i$ contains all dynamic sets in $D$ which are consistent with that path.
	Thus, since $\lvert D_1\rvert = \lvert D \rvert < 2^d$, it follows by construction that $\lvert D_i \rvert <2^i$ for every $i < d$,
	and in particular that $D_d=\emptyset$ at the end of the process. 
	Thus, the set $E\in {\cal E}$ which is consistent with the path corresponding to $x_1,\ldots ,x_d$  
	satisfies $E\cap \{x_1,\ldots x_d\}\neq \dS_i(x_1,\ldots x_d)$ for every $i<2^d$, as required.
\end{proof}

\subsubsection{Proof of Lemma~\ref{lem:fractional-cover-littlestone}}\label{sec:cover-eps}
	For convenience, let us bound $N'(\E,\sqrt{\epsilon})\le (C/\epsilon)^d$.
	We start by defining the fractional cover $\cal{B}$ and then prove its validity. Let $p = 3d/(\epsilon n)$ for a sufficiently large universal constant $C_0$, and $\dS \sim {\cal B}$ is sampled as follows:
	\begin{enumerate}
		\item Select a random subset $\rvI' \subseteq [n]$, where each $i \in [n]$ is selected independently with probability $p$.
		\item Select a subset $\rvI \subseteq \rvI'$ of size $|\rvI|\le d$, uniformly at random from the set of all $\binom{|\rvI'|}{\le d}$ subsets.
		\item Output $\dS = \dS_{\rvI}$.
	\end{enumerate}
	
	To prove that ${\cal B}$ is an $(\epsilon,(C/\epsilon)^d)$-fractional cover, fix some $E \in {\cal E}$ and sequence $\vx$. From the proof of Lemma~\ref{prop:cover}, for any $I'$, there exists $I^* = I^*(I') \subseteq I'$ of size $|I^*| \le d$ such that $(\dS_{I^*}(\vx))_{I'} = E \cap \vx_{I'}$. 
	Denote $\rvI^* = I^*(\rvI')$ where $\rvI'$ is distributed as above.
	
	We will bound for below the probability that $\rvI^*$ satisfies
	\[
	|\dS_{\rvI^*}(\vx) \triangle (E\cap \vx)| 
	\le \epsilon n.
	\]
	Further, we will bound from below the probability that $\rvI = \rvI^*$. Combining this two bounds, this will give a lower bound on the probability that 
	\[
	|\dS_{\rvI}(\vx) \triangle (E\cap \vx)| 
	\le \epsilon n,
	\]
	that suffices to complete the proof.
	
	
	To begin with the first step, notice that
	\begin{equation}\label{eq:346}
	|\dS_{\rvI^*}(\vx) \triangle (E\cap \vx)| 
		= |\{t\in [n]\colon x_t \in \maj({{\cal E}_{t}^{\rvI^*}}) \triangle E \}|. 
	\end{equation}
	To analyze the right hand side of the above quantity, place each time $t \in [n]$ in one of four categories:
	\begin{enumerate}
		\item $x_t \in \maj(\E_{t-1}^{\rvI^*}) \triangle E$ and $t \in \rvI'$.
		\item $x_t \in \maj(\E_{t-1}^{\rvI^*}) \triangle E$ and $t \notin \rvI'$.
		\item $x_t \notin \maj(\E_{t-1}^{\rvI^*}) \triangle E$ and $t \notin \rvI'$
		\item $x_t \notin \maj(\E_{t-1}^{\rvI^*}) \triangle E$ and $t \in \rvI'$.
	\end{enumerate}
	Notice that the above properties apply:
	\begin{itemize}
		\item It holds that $\maj(\E_{t}^{\rvI^*}) \ne \maj(\E_{t-1}^{\rvI^*})$ if and only if $t$ is in category (1).
		\begin{itemize}
		\item For $t$ in category (1) it holds that $t \in \rvI'$ which implies that $x_t \notin \maj(\E_{t}^{\rvI^*}) \triangle E$, since $\dS_{\rvI^*}$ is defined to agree with $E$ on all $t \in \rvI'$. While any $t$ in category (1) satisfies $x_t \in \maj(\E_{t-1}^{\rvI^*}) \triangle E$, this implies that $\maj(\E_{t}^{\rvI^*}) \ne \maj(\E_{t-1}^{\rvI^*})$.
		\item For $t$ in categories (2) and (3) it holds that $t \notin \rvI'$. Since $\rvI^*\subseteq \rvI'$, then $t \notin \rvI^*$. By definition of $\dS_{\rvI^*}$ it holds that $\maj(\E_{t}^{\rvI^*}) = \maj(\E_{t-1}^{\rvI^*})$ whenever $t \notin \rvI^*$.
		\item For $t$ in category (4), it holds that $x_t \notin \maj(\E_{t-1}^{\rvI^*}) \triangle E$. Since $\dS_{\rvI^*}$ agrees with $E$ on $\rvI'$ and since $t\in \rvI'$, it holds that $\dS_{\rvI^*}$ agrees with $E$ on $x_t$, namely, $x_t \notin \maj(\E_{t}^{\rvI^*}) \triangle E$. This implies that $x_t \notin \maj(\E_{t-1}^{\rvI^*}) \triangle \maj(\E_{t}^{\rvI^*})$. By definition of the dynamic set $\dS_{\rvI^*}$ it holds that $\maj(\E_{t-1}^{\rvI^*}) = \maj(\E_{t}^{\rvI^*})$ if and only if $x_t \notin \maj(\E_{t-1}^{\rvI^*}) \triangle \maj(\E_{t}^{\rvI^*})$. In particular, $\maj(\E_{t-1}^{\rvI^*}) = \maj(\E_{t}^{\rvI^*})$ as required.
		\end{itemize}
		\item It holds that $x_t \in \maj(\E_{t}^{\rvI^*}) \triangle E$ if and only if $t$ is in category (2):
		\begin{itemize}
			\item For categories (1) and (4) it holds that $t \in \rvI'$. By definition of $\rvI^*$, $\dS_{\rvI^*}$ and $E$ agree for any $\vx_t$ for $t \in \rvI'$. This implies that $x_t \notin \maj(\E_{t}^{\rvI^*}) \triangle E$.
			\item For categories (2) and (3), it holds that $t \notin \rvI'$ hence $t \notin \rvI^*$ which implies by definition of $\dS_{\rvI^*}$ that $\maj(\E_{t}^{\rvI^*})=\maj(\E_{t-1}^{\rvI^*})$. Since for category (2) we have $x_t \in \maj(\E_{t-1}^{\rvI^*}) \triangle E$, we further have $x_t \in \maj(\E_{t}^{\rvI^*}) \triangle E$. Similarly, in category (3) we have $x_t \notin \maj(\E_{t-1}^{\rvI^*}) \triangle E$ hence $x_t \notin \maj(\E_{t}^{\rvI^*}) \triangle E$. 
		\end{itemize}
	\end{itemize}

	Let $\rv{H}$ (hit) denote the set of all indices $t$ that correspond to case (1) and $\rv{M}$ (miss) denote the set of indices in case (2). We view $\rv{H}$ and $\rv{M}$ as random variables that are functions of the random variable $\rvI'$ (where $\vx$ and $E$ are fixed).
	Since only the elements in $t\in \rv{M}$ satisfy $x_t \in \maj(\E_{t}^{I^*}) \triangle E$, the goal is to upper bound $|\rv{M}|$. In fact, we will upper bound its expected value and then use Markov's inequality to derive tail bounds.
	
	Before bounding $\Ex[|\rv{M}|]$, notice that $|\rv{H}| \le d$. This holds due to the fact that, as described above, for each $t \in \rv{H}$, $\maj(\E_{t}^{\rvI^*}) \ne \maj(\E_{t-1}^{\rvI^*})$. And this can happen at most $d$ times, since $\Ldim(\maj(\E_{t}^{\rvI^*})) <\Ldim(\maj(\E_{t-1}^{\rvI^*}))$ for any such $t$, as described in the proof of Lemma~\ref{prop:cover}.
	
	We proceed with proving that $\Ex|\rv{M}|\le d/p$. Denote $\rv{H} = \{\rv{t}_1,\dots,\rv{t}_{|\rv{H}|}\}$ let $\rv{t}_0 = 0$, and let $\rv{y}_j$ denote the number of elements of $\rv{M}$ between $\rv{t}_{j-1}$ and $\rv{t}_j$: $\rv{y}_j = |\rv{M} \cap \{\rv{t}_{j-1}+1,\rv{t}_{t-1}+2,\dots,\rv{t}_j-1\}|$. Note that $|\rv{M}| = \sum_{j=1}^d \rv{y}_j$. 
	
	We claim that $\mathbb{E}[\rv{y}_j] \le 1/p$. Define 
	\[
	\rv{U}_j := \{t \colon t > \rv{t}_{j-1}, x_j \in \maj(\E^{\rvI^*}_{\rv{t}_{j-1}})\triangle E \}; \quad
	\rv{U}_j^- = \rv{U}_j \cap [t-1].
	\]
	Notice that for any $t$ that satisfies $\rv{t}_{j-1} < t < \rv{t}_j$ it holds that $\maj(\E^{\rvI^*}_{\rv{t}_{j-1}}) = \maj(\E^{\rvI^*}_{t})$, since, as stated above, $\maj(\E^{\rvI^*}_t)$ only changes at iterations $t \in \rv{H}$. This will imply the following:
	\begin{equation}\label{eq:345}
	\rv{M} \cap \{\rv{t}_{j-1}+1,\dots, \rv{t}_j-1 \}
	=\rv{U}_j^-.
	\end{equation}
	For the first direction of \eqref{eq:345}, any $t\in \rv{U}_j^-$ satisfies $t \notin \rv{H}$ by definition, and further it satisfies $x_t \in \maj(\E^{\rvI^*}_{\rv{t}_{j-1}})\triangle E = \maj(\E^{\rvI^*}_{t})\triangle E$ which implies that it is in $\rv{M} \cup \rv{H}$. However, it cannot be in $\rv{H}$ since $\rv{H} = \{\rv{t}_1,\dots,\rv{t}_{|\rv{H}|}\}$. Further, it satisfies $\rv{t}_{j-1} < t < \rv{t}_j$ by definition. For the second direction, any $t$ in the left hand side satisfies $x_t \in \maj(\E^{\rvI^*}_{t})\triangle E = \maj(\E^{\rvI^*}_{\rv{t}_{j-1}})\triangle E$ which implies that it is in $\rv{U}_j$. We derive \eqref{eq:345} which implies that $\rv{y}_j = |\rv{U}_j^-|$.
	
	To estimate $|\rv{U}_j^-|$,
	notice that the first element of $\rv{U}_j$ that is also in $\rvI'$ is $\rv{t}_j$. This follows from the fact that $\maj(\E^{\rvI^*}_{t})$ changes only once an element of $\rv{H}$ is observed. Further, conditioned on $\rv{t}_{j-1}$ and $\rvI'\cap[\rv{t}_{j-1}]$, the set $\rv{U}_j$ is fixed, and conditionally, since any element of $\rv{U}_j$ is in $\rvI'$ with probability $p$, the expected number of elements in $\rv{U}_j$ that are encountered before the first element of $\rvI'$ is bounded by $1/p$. This quantity is exactly $\Ex[\rv{y}_j] = |\rv{U}_j^-|\le 1/p$, and we derive that $\Ex[|\rv M|] \le d/p$.
	From Markov's inequality, $\Pr[|\rv{M}| \le 3d/p] \ge 2/3$.
	
	We have proved that with probability $2/3$, $|\rv{M}|\le 3d/p$. This, from \eqref{eq:346} and from the definitions of $\rv{M}$ and $p$, implies that with probability $2/3$, $|\dS_{\rvI^*}(x) \triangle (E\cap \vx)| \le 3d/p \le \epsilon n$. Further, we want to lower bound the probability that $\rvI = \rvI^*$.
	Notice that $\Pr[\rvI = \rvI^* \mid \rvI'] = 1/\binom{|\rvI'|}{\le d}$, hence it is desirable to show that $|\rvI'|$ is small with high probability. Indeed, since $\mathbb{E}|\rvI'|=np$, by Markov's inequality, $\Pr[|\rvI'| \le 3np] \ge 2/3$. By a union bound,
	\[
	\Pr[|\rvI'| \le 3np, |\dS_{\rvI^*}(x) \triangle (E\cap \vx)| \le \epsilon n] \ge 1/3.
	\]
	We conclude that
	\begin{align*}
	&\Pr\lp[|\dS_\rvI(\vx)\triangle (E\cap \vx)| \le \epsilon n\rp]
	\ge \Pr\lp[|\rvI'| \le 3np, |\dS_{\rvI^*}\triangle (E\cap \vx)| \le \epsilon n, \rvI = \rvI^*\rp]\\
	&\quad= \Pr\lp[|\rvI'| \le 3np, |\dS_{\rvI^*}\triangle (E\cap \vx)| \le \epsilon n\rp] \Pr\lp[\rvI = \rvI^* \mid |\rvI'| \le 3np, |\dS_{\rvI^*}\triangle (E\cap \vx)|\rp]
	\ge \frac{1}{3} \cdot \binom{3np}{\le d}^{-1} \\
	&\quad = \frac{1}{3} \cdot \binom{9d/\epsilon}{\le d}^{-1} \ge \lp(\frac{C}{\epsilon}\rp)^{-d},
	\end{align*}
	using the fact that $\rvI^*$ is a function of $\rvI'$, and conditioned on any value of $\rvI'$, the probability that $\rvI=\rvI^*$ is $\binom{|\rvI'|}{\le d}^{-1}$; and further, that $\binom{n}{k} \le (Cn/k)^k$ for a universal $C>0$.

\subsection{Deriving Bounds on $\epsilon$-Approximation via Fractional Covering Numbers} \label{sec:eps-approx-via-cover}
In this section we prove the concentration results based on covering numbers, starting with results based on deterministic $0$-covers and moving to fractional $\epsilon$-covers.
The following definition will be useful: for any $E \in \E$ define 
$\rv{Y}_E = |E\cap \vx_{\rvI}| - |E\cap \vx_{[2k]\setminus \rvI}|$. Similarly, for any $\dS$, define
$\rv{Y}_{\dS} = |\dS\cap \vx_{\rvI}| - |\dS\cap \vx_{[2k]\setminus \rvI}|$. Notice that
\[
\Disc_{\A,\rvI} = \max_{E\in\E} |\rv{Y}_E|.
\]
\subsubsection{Basic Lemmas for Deterministic Covers and Proof of Lemma~\ref{lem:eps-cover-bound}}
We start with concentration of a single dynamic set:
\begin{lemma}\label{lem:single-DS}
	Let $\dS$ be a dynamic set with $|\dS|\le m$. Let $\rvI\sim \Ber(n,1/2)$. Then, for any $t\ge 0$,
	\[
	\Pr\lp[\lp|\rv{Y}_{\dS} \rp|\ge t\rp]
	\le 2\exp(-t^2/(2m)).
	\]
	Consequently, for any $\delta \in(0,1/2)$, with probability $1-\delta$ it holds that
	\[
	\lp|\rv{Y}_{\dS} \rp|
	\le C \sqrt{m\log(1/\delta)}. 
	\]
\end{lemma}
For the proof of Lemma~\ref{lem:single-DS}, we need the following Martingale lemma (notice that an overview on Martingales is given in Section~\ref{sec:martingales})
\begin{lemma}[\cite{victor1999general}, Theorem 6.1] \label{lem:martingale-victor}
	Let $\rv{y}_0,\dots,\rv{y}_n$ be a Martingale adapted to the filtration $F_0,\dots,F_n$, such that $\sum_{i=1}^n |\rv{y}_i - \rv{y}_{i-1}|^2\le s$ holds almost surely for some $s > 0$. Assume that for all $i \in [n]$, conditioned on $F_{i-1}$, $\rv{y}_i - \rv{y}_{i-1}$ is a symmetric random variable (namely, it has the same conditional distribution as $\rv{y}_{i-1}-\rv{y}_i$). Then, for any $t>0$,
	\[
	\Pr\lp[|\rv{y}_n - \rv{y}_0|\ge t \rp]
	\le 2\exp(-t^2/(2s)).
	\]
\end{lemma}
\begin{proof}[Proof of Lemma~\ref{lem:single-DS}]
	This follows directly from Lemma~\ref{lem:martingale-victor}. Indeed, we apply this lemma with $\rv{y}_i = |\dS\cap I \cap [i]| - |\dS\cap ([n]\setminus I)\cap [i]|$ and $s = m$.
\end{proof}

We are ready to prove Lemma~\ref{lem:eps-cover-bound}.

\begin{proof}[Proof of Lemma~\ref{lem:eps-cover-bound}]
	Notice that it suffices to prove that for any $t \ge 0$,
	\[
	\Pr[\Disc_{\A,\rvI} > t] \le 
	2N(\E, 0)\exp(-t^2/(4k)).
	\]
	Let $\mathcal{N}$ be a minimal $0$-net for $\E$. Applying Lemma~\ref{lem:single-DS} with $m=2k$, for any $\dS \in \mathcal{N}$,
	\[
	\Pr[|\rv{Y}_E| > t] \le 2\exp(-t^2/4k)
	\]
	Applying a union bound over $\dS\in \mathcal{N}$,
	\[
	\Pr[\Disc_{\A,\rvI} > t]
	= \Pr[\max_{E\in\E} |\rv{Y}_E|>t]
	\le \Pr[\max_{\dS\in\mathcal{N}} |\rv{Y}_\dS|>t]
	\le 2N(\E, 0)\exp(-t^2/(4k)).
	\]
\end{proof}

\subsubsection{Basic Lemmas for Fractional Covers}
To give intuition about fractional covers, we prove a variant of Lemma~\ref{lem:eps-cover-bound}. First, we start with an auxiliary lemma the replaces the union bound:
\begin{lemma}\label{lem:aux-union}
	Let $\{\rv{y}_j\}_{j\in J}$ denote random variables over $\{0,1\}$ where $J$ is some index set, and assume that for any $j \in J$, $\Pr[\rv{y}_j = 1] \le p$, for some $p > 0$. Let $\mu$ denote some probability measure over $J$ and let $\alpha > 0$. Then,
	\[
	\Pr_{\{\rv{y}_j\}_{\colon j\in J}}\lp[\mu(\{j \colon \rv{y}_j = 1\})
	\ge \alpha\rp] \le p/\alpha.
	\]
\end{lemma}
\begin{proof}
	Notice that by linearity of expectation,
	\[
	\Ex[\mu(\{j \colon \rv{y}_j = 1\})]
	= \Ex \int \rv{y}_j d\mu
	= \int \Ex[\rv{y}_j] d\mu
	\le \int p d\mu
	= p.
	\]
	The proof follows by Markov's inequality.
\end{proof}
The following lemma applies Lemma~\ref{lem:aux-union} specifically for distributions over dynamic sets.
\begin{lemma}\label{lem:aux-union-ds}
	Let $\mu$ be a probability measure over dynamic sets with $|\dS|\le m$ for all $\dS$ in the support of $\mu$. Then, for any $\delta' > 0$, with probability at least $1-\delta'$ over $\rvI\sim \Ber(n,1/2)$ it holds that
	\[
	\mu\lp( \lp\{\dS\colon \lp|\rv{Y}_\dS \rp| \le C \sqrt{m\log (1/(\delta'\alpha))}\rp\} \rp)\ge 1-\alpha.
	\]
\end{lemma}
\begin{proof}
	Apply Lemma~\ref{lem:aux-union} with $\rv{y}_{\dS}$ being the indicator that $\lp|\rv{Y}_{\dS} \rp| \le C \sqrt{m\log (1/(\delta'\alpha))}$ and $p = \alpha\delta'$. If $C$ is a sufficiently large constant, it follows from Lemma~\ref{lem:single-DS} that $\Pr[\rv{y}_{\dS}=1]\le \alpha\delta'$ and Lemma~\ref{lem:aux-union} can be applied to derive the desired result.
\end{proof}
Using Lemma~\ref{lem:aux-union-ds}, one can derive bounds for epsilon approximation based on fractional covering numbers at scale $0$:
\begin{lemma}
	Let $\A\in\Adv_n$, $\rvI\sim \Bin(n,1/2)$ and $\delta > 0$. Then, with probability $1-\delta$,
	\[
	\Disc_{\A,\rvI} \le C\sqrt{n\lp(\log N'(\E,0) + \log (1/\delta)\rp)}.
	\]
\end{lemma}
\begin{proof}
	Let $\mu$ be a $(0,N'(\E,0))$-fractional cover for $\E$.
	We apply Lemma~\ref{lem:aux-union-ds} with $m=n$, $\alpha = 1/(2N'(\E,0))$, $\delta'=\delta$ and $\mu=\mu$ to get that with probability $1-\delta$, 
	\[
	\mu\lp( \dS\colon |\rv{Y}_{\dS}| \le C \sqrt{n (\log N'(\E,0)+\log(1/\delta))} \rp) \ge \frac{1}{2N'(\E,0)}.
	\]
	Whenever this holds, for every $E \in \E$ there exists $\dS \in \E$ that $E\cap \vx = \dS(\vx)$ which implies that $\rv{Y}_{\dS} = \rv{Y}_E$. Hence,
	\[
	\Disc_{\A,\rvI}
	= \max_{E\in\E} |\rv{Y}_E|
	\le \max_{\dS\in\mathcal{N}} |\rv{Y}_\dS|
	\le C \sqrt{n (\log N'(\E,0)+\log(1/\delta))}.
	\]
\end{proof}
\subsubsection{Chaining for Non-Fractional Covers}
The proof of Lemma~\ref{lem:fractional-cover-bound} follows the technique of chaining. Before presenting the proof for fractional covers, we start by presenting an outline of the proof for non-fractional covers, while obtaining a similar bound with $N(\E,\epsilon)$ instead of $N'(\E,\epsilon)$. The proof follows from standard techniques. Some technicalities are ignored for the sake of presentation.

Let $1 = \epsilon_0 > \epsilon_1 > \cdots$ be a non-increasing sequence of values with $\lim_{i\to\infty}\epsilon_i = 0$.
We take nets $\mathcal{N}_0,\mathcal{N}_1,\dots$, where each $\mathcal{N}_i$ is an $\epsilon_i$ net for $\E$. Each $E \in \E$ we approximate using elements from the different nets: for any $E \in \E$, $i \ge 0$ and $\vx \in X^n$, let $\dS_{E,i,\vx}$ denote an arbitrarily chosen dynamic set $\dS \in \mathcal{N}_i$ that satisfies $|\dS_{E,i,\vx}(\vx) \triangle (E \cap \vx)| \le \epsilon_i^2 n$. Further, define the random variable over dynamic sets $\rv{B}_{E,i} = \dS_{E,i,\rvvx}$ and notice that $|\rv{B}_{E,i}(\rvvx) \triangle (E \cap \rvvx)| \le \epsilon_i^2 n$. Since $\epsilon_i\to 0$, we have $\rv{Y}_{\rv{B}_{E,i}}\to \rv{Y}_E$, hence,
\begin{equation}\label{eq:72}
|\rv{Y}_E| = \lp| \rv{Y}_{\rv{B}_{E,0}} + \sum_{i=1}^\infty (\rv{Y}_{\rv{B}_{E,i}} - \rv{Y}_{\rv{B}_{E,i-1}}) \rp|
\le \lp|\rv{Y}_{\rv{B}_{E,0}}\rp| + \sum_{i=1}^\infty \lp|\rv{Y}_{\rv{B}_{E,i}} - \rv{Y}_{\rv{B}_{E,i-1}} \rp|.
\end{equation}
The hope is for the sum in the right hand side of \Cref{eq:72} to converge. Indeed, notice that 
\begin{equation}\label{eq:73}
|\rv{B}_{E,i} \triangle \rv{B}_{E,i-1}| \le |\rv{B}_{E,i} \triangle E|+ |E\triangle\rv{B}_{E,i-1}| \le (\epsilon_i^2 + \epsilon_{i-1}^2)n \le 2 \epsilon_{i-1}^2 n,
\end{equation}
hence, as $i$ increases, the differences $\lp|\rv{Y}_{\rv{B}_{E,i}} - \rv{Y}_{\rv{B}_{E,i-1}} \rp|$ tend to decrease.
Taking a maximum over $E\in \E$ in \eqref{eq:72}, we have
\begin{equation}\label{eq:75}
\max_{E \in \E} |\rv{Y}_E|
\le \max_{E \in \E} \lp|\rv{Y}_{\rv{B}_{E,0}}\rp|
+  \sum_{i=1}^\infty \max_{E \in \E} \lp|\rv{Y}_{\rv{B}_{E,i}} - \rv{Y}_{\rv{B}_{E,i-1}} \rp|.
\end{equation}
We show how to bound the summand corresponding to any $i\ge 1$ while term $\max_{E \in \E} \lp|\rv{Y}_{\rv{B}_{E,0}}\rp|$ can be similarly bounded. Since $\rv{Y}_{\rv{B}_{E,i}} \in \mathcal{N}_i$ and $\rv{Y}_{\rv{B}_{E,i-1}} \in \mathcal{N}_{i-1}$, there can be at most $|\mathcal{N}_{i-1}||\mathcal{N}_i| = N(\E,\epsilon_{i-1})N(\E,\epsilon_i)$ distinct differences $\rv{Y}_{\rv{B}_{E,i}} - \rv{Y}_{\rv{B}_{E,i-1}}$.
Intuitively, as $i$ increases, the maximum is taken over more elements, however, the individual differences are smaller, and the hope is that the sum in \Cref{eq:75} would converge.

Using \Cref{eq:73}, we can apply Lemma~\ref{lem:single-DS} with $m=2\epsilon_{i-1}^2 n$ to obtain that each distance $\rv{Y}_{\rv{B}_{E,i}} - \rv{Y}_{\rv{B}_{E,i-1}}$ is bounded by $C\epsilon_i \sqrt{n \log(1/\delta')}$ with probability $1-\delta'$. Taking $\delta'$ smaller than $1/N(\E,\epsilon_{i-1})N(\E,\epsilon_i)$ and applying a union bound over at most $N(\E,\epsilon_{i-1})N(\E,\epsilon_i)$ elements, we derive that with high probability,
\begin{equation}\label{eq:76}
\max_{E \in \E} \lp|\rv{Y}_{\rv{B}_{E,i}} - \rv{Y}_{\rv{B}_{E,i-1}} \rp| 
\le O\lp(\epsilon_{i-1}\sqrt{n \log N(\E,\epsilon_{i})}\rp).
\end{equation}

We further take a union bound over $i \ge 0$ and derive by \eqref{eq:75} and \eqref{eq:76} that w.h.p,
\[
\max_{E\in \E}|\rv{Y}_E| \le O\lp(\sqrt{n} \sum_{i=1}^\infty \epsilon_{i-1} \sqrt{ \log N(\E,\epsilon_{i})}\rp).
\]
A standard choice for $\epsilon_i$ is $\epsilon_i = 2^{-i}$, and this yields
\begin{equation} \label{eq:dudleys}
\max_{E\in \E}|\rv{Y}_E| \le O\lp(\sqrt{n} \sum_{i=1}^\infty 2^{-i} \sqrt{ \log N(\E,2^{-i})}\rp)
\le O\lp(\sqrt{n} \int_{\epsilon=0}^1 \sqrt{ \log N(\E,\epsilon)}\rp),
\end{equation}
where the last inequality is by approximating the sum with an integral. This, in fact, is the celebrated \emph{Dudley's integral}.
\begin{remark}
	The choice of $\epsilon_i  = 2^{-i}$ can generally only yield bounds on $\Ex \max_{E\in \E}|\rv{Y}_E|$, rather than high probability bounds. It is common in literature to obtain high probability bounds by first bounding the expectation $\Ex \max_{E\in \E}|\rv{Y}_E|$ and then showing that this maximum concentrates around its expectation, via McDiarmid-like inequalities. However, such concentration inequalities cannot be applied in the adversarial setting. Hence we use instead a different well-studied choice of $\epsilon_i$ that directly gives high probability bounds.
\end{remark}

\subsubsection{Proof of Lemma~\ref{lem:fractional-cover-bound}}
\label{sec:proof-of-fractional-cover}

The proof is by the standard technique of chaining, with adaptations to handle fractional covering numbers.
Our goal is to bound $\max_{E \in \E} \rv{Y}_E$.
The main ideas is to create finer and finer approximations for $\E$ using fractional epsilon nets.
Formally, let $\epsilon_j$ denote the minimal value of $\epsilon$ such that $N'(\E,\epsilon_j) \le 2^{2^j}$.
We will derive the following bound: with probability $1-\delta$,
\begin{equation} \label{eq:equivalent-chaining-bd}
\max_{E\in \E} |\rv{Y}_E|
\le C\sqrt{n}\lp( \sqrt{\log(1/\delta)}
+ \sum_{j=0}^\infty \epsilon_j \sqrt{\log N'(\E,\epsilon_j)}\rp)
\le C\sqrt{n}\lp( \sqrt{\log(1/\delta)}
+ \sum_{j=0}^\infty \epsilon_j 2^{j/2}\rp).
\end{equation}
This series in the right hand side is known to be equal, up to constant factors, to the following integral, which is known as Dudley's integral. We include the proof for completeness.
\begin{lemma}\label{lem:dudley-equivalence}
\[
\sum_{j=0}^\infty \epsilon_j 2^{j/2}
\le \frac{\sqrt{2}}{\sqrt{2}-1}\int_{0}^1 \sqrt{\log N'(\E, \epsilon)} d\epsilon.
\]
\end{lemma}
\begin{proof}
Notice that
\begin{align*}
\sum_{j=0}^\infty \epsilon_j 2^{j/2}
&\le \sum_{j=0}^\infty \sum_{i=j}^\infty (\epsilon_i - \epsilon_{i+1}) 2^{j/2}
= \sum_{i=0}^\infty \sum_{j \le i} (\epsilon_i - \epsilon_{i+1}) 2^{j/2} \\
&= \sum_{i=0}^\infty (\epsilon_i - \epsilon_{i+1}) \frac{\sqrt{2}^{i+1}-1}{\sqrt{2}-1} 
\le \frac{\sqrt{2}}{\sqrt{2}-1}\sum_{i=0}^\infty (\epsilon_i - \epsilon_{i+1}) 2^{i/2}.
\end{align*}
Further, by definition of $\epsilon_i$ we have that for any $\epsilon< \epsilon_i$, $N'(\E,\epsilon) > 2^{2^i}$, hence
\[
\sqrt{\log N'(\E,\epsilon)} > 2^{i/2}.
\]
This implies that
\[
\sum_{i=0}^\infty (\epsilon_i - \epsilon_{i+1}) 2^{i/2}
\le \sum_{i=0}^\infty \int_{\epsilon_{i+1}}^{\epsilon_i} \sqrt{\log N'(\E,\epsilon)}
\le \int_{0}^1 \sqrt{\log N'(\E,\epsilon)}.
\]
\end{proof}

We start with some definitions; for any dynamic sets $\dS$ and $\dS'$ and $E \in \E$:
\begin{itemize}
	\item Let $\dS\setminus\dS'$ be defined by $(\dS\setminus\dS')(x_1,\dots,x_n) = \dS(x_1,\dots,x_n) \setminus \dS'(x_1,\dots,x_n)$. Notice that $\dS\setminus\dS'$, as defined, is a dynamic set.
	\item Define $\dS_{\le m}$ as the dynamic sets that simulates $\dS$ up to the point where $\dS$ has $m$ elements, and then it stops adding elements.
	\item Let $\rvvx = \vx(\A,\rvI)$ denote the stream that is output by the adversary $\A$ in interaction with the sampler that samples the coordinates from $\rvI$. 
\end{itemize}

Let $\mu_j$ be a probability measure over dynamic sets that is a fractional $(\epsilon_j,N'(\E, \epsilon_j))$-cover for $\E$. 
We will approximate each $E \in \E$ using dynamic sets $\pi_{E,0}, \pi_{E,1},\pi_{E,2},\dots$, where $\pi_{E,j} \in \support(\mu_j)$ is an $\epsilon$-approximation for $E$, namely,
\[
|(E \cap \rvvx)\triangle \pi_{E,j}(\rvvx)| \le \epsilon^2 n,
\]
Notice that by definition of (fractional) covers, $\pi_{E,j}$ may depend on the stream $\rvvx$, 
which is a random variable, hence $\pi_{E,j}$ is also a random variable. The following lemma shows that for a sufficiently large $j$, $\pi_{E,j}(\rvvx)$ equals $E\cap\rvvx$.
\begin{lemma}\label{lem:j1}
	Assume that $N'(\E,\epsilon)<\infty$ for all $\epsilon>0$. Then there exists $j_1>0$ such that for all $E \in \E$ and all $j\ge j_1$, $E\cap \rvvx = \pi_{E,j}(\rvvx)$ holds with probability $1$ over $\rvvx$. Consequently, $\rv{Y}_{E} = \rv{Y}_{\pi_{E,j}}$ for all $j \ge j_1$.
\end{lemma}
\begin{proof}
The assumption of the lemma implies that for some $j$, $\epsilon_j < n^{-1/2}$ and let $j_1$ be the minimal such value. By definition of $\pi_{E,j}$ we have that for all $j\ge j_1$, $|(E \cap \rvvx)\triangle \pi_{E,j}(\rvvx)| < 1$, hence $E \cap \rvvx=\pi_{E,j}(\rvvx)$ as required.
\end{proof}
We can assume that $N'(\E,\epsilon)<\infty$ for all $\epsilon>0$ otherwise Dudley's integral (appearing in Lemma~\ref{lem:fractional-cover-bound}) would diverge. Hence, by Lemma~\ref{lem:j1}, for any $j_0\ge 0$,
\begin{align}
|\rv{Y}_E| 
&:= \lp| \rv{Y}_{\pi_{E,j_0}} + \sum_{j=j_0}^\infty \lp(\rv{Y}_{\pi_{E,j+1}}-\rv{Y}_{\pi_{E,j}}\rp) \rp|
\le \lp| \rv{Y}_{\pi_{E,j_0}}\rp| + \sum_{j=j_0}^\infty \lp| \rv{Y}_{\pi_{E,j+1}}-\rv{Y}_{\pi_{E,j}} \rp|\\
&= \lp| \rv{Y}_{\pi_{E,j_0}}\rp|+ \sum_{j=j_0}^\infty \lp| \rv{Y}_{\pi_{E,j+1}\setminus \pi_{E,j}}-\rv{Y}_{\pi_{E,j}\setminus\pi_{E,j+1}} \rp|
\le \lp| \rv{Y}_{\pi_{E,j_0}}\rp|+\sum_{j=j_0}^\infty \lp| Y_{\pi_{E,j+1}\setminus \pi_{E,j}}\rp| + \sum_{j=j_0}^\infty\lp|\rv{Y}_{\pi_{E,j}\setminus\pi_{E,j+1}} \rp|.\label{eq:deviation-decomp}
\end{align}
We bound the supremum over $E\in \E$ by taking the supremum over each term separately:
\begin{equation}\label{eq:chain-general}
\sup_{E\in\E}|\rv{Y}_E|
\le \sup_{E\in\E}\lp| \rv{Y}_{\pi_{E,j_0}}\rp| + \sum_{j=j_0}^\infty \sup_{E\in\E}\lp| \rv{Y}_{\pi_{E,j+1}\setminus \pi_{E,j}}\rp| + \sum_{j=j_0}^\infty \sup_{E\in\E}\lp|\rv{Y}_{\pi_{E,j}\setminus\pi_{E,j+1}} \rp|.
\end{equation}
Each supremum will be bounded using the generalized union bound Lemma~\ref{lem:aux-union}. 

Next, we define measures over differences of dynamic sets, that will be used to bound the right hand side of \eqref{eq:chain-general}.
For any $j \ge 1$ we let $\mu_{j,0}$ denote a probability measure over dynamic sets such that $\dS\sim \mu_{j,0}$ is drawn by first drawing $\dS_j \sim \mu_j$ and $\dS_{j+1} \sim \mu_{j+1}$ and then outputting $\dS = (\dS_j \setminus \dS_{j+1})_{\le 2\epsilon_j^2}$. Similarly, let $\mu_{j,1}$ denote the measure that outputs $(\dS_{j+1} \setminus \dS_{j})_{\le 2\epsilon_j^2 n}$. By the generalized union bound, we have the following:
\begin{lemma} \label{lem:aux-chaining}
	Let $j_0 = \lceil\log_2\log_2 (1/\delta)\rceil$. With probability at least $1-\delta$, the following holds:
	\begin{itemize}
	\item It holds that
	\[
	\mu_{j_0}\lp( \lp\{ \dS\in \support(\mu_{j_0}) 
	\colon |\rv{Y}_{\dS}| > C_0 \sqrt{\log(1/\delta)}\sqrt{n}
	\rp\} \rp)
	\le \frac{1}{3 \cdot 2^{2^{j_0}}}.
	\]
	\item For all $j \ge j_0$ and all $b \in \{0,1\}$,
	\[
	\mu_{j,b}\lp( \lp\{ \dS\in \support(\mu_{j,b}) 
	\colon |\rv{Y}_{\dS}| > C_0 \epsilon_{j} 2^{j/2}\sqrt{n}
	\rp\} \rp)
	\le \frac{1}{18 \cdot 2^{2^j}\cdot 2^{2^{j+1}}}.
	\]
	\end{itemize}
\end{lemma}
\begin{proof}
	The proof follows directly from Lemma~\ref{lem:aux-union-ds}. First, we show that the first item holds with probability $1-\delta/2$: it follows by substituting in Lemma~\ref{lem:aux-union-ds} the values $m=n$, $\delta'=\delta/2$ and $\alpha=\frac{1}{3 \cdot 2^{2^{j_0}}}$, and notice that 
	\[
	\log_2(1/\alpha)
	= \log_2 3 + 2^{j_0}
	\le C \cdot \log(1/\delta),
	\]
	by definition of $j_0$.
	
	For the second item, we show that the term corresponding to a specific $j\ge j_0$ and $b\in \{0,1\}$ holds with probability $1- 2^{-j-3} \delta$.
	Indeed, we can substitute $m=2\epsilon_j^2 n$, $\delta'=2^{-j-3}\delta$ and $\alpha=\frac{1}{18 \cdot 2^{2^j}\cdot 2^{2^{j+1}}}$, and notice that
	\[
	\log_2\frac{1}{\delta'}
	= \log_2\frac{1}{\delta} + j + 3
	\le C \cdot 2^j,
	\]
	since $j \ge j_0$ and by definition of $j_0$.
	Further,
	\[
	\log_2\frac{1}{\alpha}
	= \log_2 18 +2^{j+2}
	\le C \cdot 2^j.
	\]
	
	By a union bound, the failure probability is bounded by
	\[
	\delta/2 + 2 \cdot \sum_{j=j_0}^\infty 2^{-j-3}\delta 
	\le \delta/2+ 2^{-j_0-1}\delta 
	\le \delta/2 + \delta/2
	= \delta.
	\]
\end{proof}
For the remainder of the proof, we fix some stream $\vx$ such that the condition in Lemma~\ref{lem:aux-chaining} holds when $\rvvx=\vx$. This fixes values $\{Y_E\}_{E\in\E}$ and $\{Y_{\dS}\}_{\dS \text{ dynamic set}}$ such that $\rv{Y}_E = Y_E$ and $\rv{Y}_\dS = Y_\dS$ for all $E$ and $\dS$. We will show that for any $E \in \E$
\[
|\rv{Y}_E| \le C \sqrt{n \log(1/\delta)} + C \sum_{j=0}^\infty 2^{j/2} \epsilon_j \sqrt{n},
\]
which suffices to complete the proof from Lemma~\ref{lem:dudley-equivalence}.
Fix $E \in \E$; we show how to define $\pi_{E,j}$. First, for any $j \ge j_0$, let $A_j$ denote the set of elements $\dS\in \support(\mu_j)$ such that $|(E \cap \vx)\triangle \dS(\vx)| \le \epsilon_j^2 n$. By the property of the fractional cover $\mu_j$, we know that
\begin{equation} \label{eq:measure-Aj}
\mu_j(A_j) \ge 1/2^{2^j}.
\end{equation}
The set $A_j$ contains the candidates for $\pi_{E,j}$. Notice that in order to bound \eqref{eq:chain-general}, we would like to bound $Y_{\pi_{E,j}\setminus \pi_{E,j-1}}$ and $Y_{\pi_{E,j-1}\setminus \pi_{E,j}}$. For this purpose,
we now define for any $j \ge j_0$ the function $R \colon \support(\mu_j) \times \support(\mu_{j+1}) \to \{0,1\}$, that indicates which pairs of elements $\dS_j,\dS_{j+1}$ are not suitable to be defined as $\pi_{E,j}$ and $\pi_{E,j+1}$:
\[
R(\dS_j,\dS_{j+1})
= \begin{cases}
 1 & \max\lp(|Y_{\dS_j\setminus \dS_{j+1}}|,
 |Y_{\dS_{j+1}\setminus \dS_{j}}|\rp) 
 > C_0\epsilon_j 2^{j/2}\sqrt{n},\ \dS_j \in A_j,\ \dS_{j+1} \in A_{j+1} \\
 0 & otherwise 
\end{cases},
\]
where $C_0$ is the constant from Lemma~\ref{lem:aux-chaining}.
Next, we further restrict the set of candidates by creating a set $A'_j\subseteq A_j$, that contains only dynamic sets $\dS_j\in A_j$ such that for many elements $\dS_{j+1}\in A_{j+1}$, the pair $(\dS_j,\dS_{j+1})$ is suitable, which is formally defined as:
\[
A'_{j}
= \lp\{ \dS_j \in A_{j} \colon 
	\mu_{j+1}\lp(\lp\{ \dS_{j+1}\colon R(\dS_j,\dS_{j+1})=1 \rp\} \rp)
	\le \frac{1}{3\cdot 2^{2^{j+1}}}
\rp\}.
\]
Next, we lower bound the measure $\mu_j(A'_j)$:
\begin{lemma}\label{lem:bnd-Ajprime}
Assume that the high probability event from Lemma~\ref{lem:aux-chaining} holds. Then, for any $j\ge j_0$,
\[
\mu_j(A'_j)
\ge \frac{2}{3\cdot 2^{2^{j}}}.
\]
\end{lemma}
\begin{proof}
	Fix some stream $\vx$ such that the high probability event of Lemma~\ref{lem:aux-chaining} holds; this fixes the values of $Y_{\dS}$ for all $\dS$. If $\dS_j \sim \mu_j$ and $\dS_{j+1}\sim \mu_{j+1}$ are drawn independently, then
	\begin{align}
	&\Ex_{\dS_j\sim \mu_j}\Ex_{\dS_{j+1}\sim \mu_{j+1}}[R(\dS_j,\dS_{j+1})]=
	\Pr_{\dS_j\sim \mu_j,\dS_{j+1}\sim \mu_{j+1}}[R(\dS_j,\dS_{j+1}) = 1]\notag\\
	&\quad= \Pr_{\dS_j,\dS_{j+1}}\lp[\max\lp(|Y_{\dS_j\setminus \dS_{j+1}}|,|Y_{\dS_{j+1}\setminus \dS_{j}}|\rp) 
	> C_0\epsilon_j 2^{j/2}\sqrt{n},\ \dS_j \in A_j,\ \dS_{j+1} \in A_{j+1} \rp]\notag\\
	&\quad\le \Pr_{\dS_j,\dS_{j+1}}\lp[|Y_{\dS_{j}\setminus \dS_{j+1}}|
	> C_0\epsilon_j 2^{j/2}\sqrt{n},\ \dS_j \in A_j,\ \dS_{j+1} \in A_{j+1}\rp]\label{eq:242}\\
	&\qquad+\Pr_{\dS_j,\dS_{j+1}}\lp[|Y_{\dS_{j+1}\setminus \dS_{j}}|
	> C_0\epsilon_j 2^{j/2}\sqrt{n},\ \dS_j \in A_j,\ \dS_{j+1} \in A_{j+1}\rp],\label{eq:243}
	\end{align}
	We focus on bounding \eqref{eq:242}; \eqref{eq:243} is bounded in a similar fashion. For any $\dS_j\in A_j$ and $\dS_{j+1}\in A_{j+1}$ we have that by definition of $A_j$ and $A_{j+1}$,
	\begin{align*}
	|(\dS_j\setminus \dS_{j+1})(\vx)|
	&= |\dS_j(\vx) \setminus \dS_{j+1}(\vx)|
	\le |\dS_j(\vx) \triangle \dS_{j+1}(\vx)|
	\le |\dS_j(\vx) \triangle E| + |E \triangle \dS_{j+1}(\vx)|\\
	&\le \epsilon_j^2n + \epsilon_{j+1}^2 n
	\le 2 \epsilon_j^2 n,
	\end{align*}
	This implies that
	\[
	|(\dS_j\setminus \dS_{j+1})(\vx)| = 
	|(\dS_j\setminus \dS_{j+1})_{\le 2\epsilon_j^2 n}(\vx)|,
	\]
	hence by definition of $\mu_{j,0}$ and by Lemma~\ref{lem:aux-chaining},
	\begin{align*}
	&\Pr_{\dS_j\dS_{j+1}}\lp[|Y_{\dS_{j}\setminus \dS_{j+1}}|
	> C_0\epsilon_j 2^{j/2}\sqrt{n},\ \dS_j \in A_j,\ \dS_{j+1} \in A_{j+1}\rp]\\
	&\quad\le \Pr_{\dS_j,\dS_{j+1}}\lp[\lp|Y_{(\dS_{j}\setminus \dS_{j+1})_{\le 2\epsilon_j^2}}\rp|
	> C_0\epsilon_j 2^{j/2}\sqrt{n}\rp]
	=  \Pr_{\dS\sim\mu_{j,0}}\lp[\lp|Y_{\dS}\rp|
	> C_0\epsilon_j 2^{j/2}\sqrt{n}\rp]\\
	&\quad\le \frac{1}{18 \cdot 2^{2^j}\cdot 2^{2^{j+1}}}.
	\end{align*}
	Similarly, \eqref{eq:243} is bounded by the same quantity, hence by \eqref{eq:242} and \eqref{eq:243} above,
	\[
	\Ex_{\dS_j\sim \mu_j}\Ex_{\dS_{j+1}\sim \mu_{j+1}}[R(\dS_j,\dS_{j+1})]
	\le \frac{1}{9 \cdot 2^{2^j}\cdot 2^{2^{j+1}}}.
	\]
	By Markov's inequality,
	\begin{align*}
	\Pr_{\dS_j \sim \mu_j}\lp[
		\dS_j \in A_j\setminus A_j'
	\rp]
	&=\Pr_{\dS_j \sim \mu_j}\lp[
	\mu_{j+1}\lp(\lp\{ \dS_{j+1}\colon R(\dS_j,\dS_{j+1})=1 \rp\} \rp)
	> \frac{1}{3\cdot 2^{2^{j+1}}}
	\rp]\\
	&=\Pr_{\dS_j \sim \mu_j}\lp[
	\Ex_{\dS_{j+1}\sim \mu_{j+1}}[R(\dS_j,\dS_{j+1})]
	> \frac{1}{3\cdot 2^{2^{j+1}}}
	\rp]\\
	&\le \frac{\Ex_{\dS_j\sim \mu_j}\Ex_{\dS_{j+1}\sim \mu_{j+1}}[R(\dS_j,\dS_{j+1})]}{1/(3\cdot 2^{2^{j+1}})}
	\le \frac{1}{3\cdot 2^{2^j}}.
	\end{align*}
	By \eqref{eq:measure-Aj},
	\[
	\mu_j(A_j')
	\ge
	\mu_j(A_j) - \mu_j(A_j\setminus A_j')
	\ge \frac{1}{2^{2^j}}
	- \frac{1}{3\cdot 2^{2^j}}
	= \frac{2}{3\cdot 2^{2^j}}.
	\]
\end{proof}
To complete the proof, we show how to define $\pi_{E,j}$ inductively on $j$ such that $R(\pi_{E,j},\pi_{E,j+1})=0$ for all $j\ge j_0$. First, we select $\pi_{E,j_0}$ to be any element of $A_j'$ such that $|Y_{\pi_{E,j_0}}| \le C_0 \sqrt{\log(1/\delta)}\sqrt{n}$; such an element exists from Lemma~\ref{lem:aux-chaining} and Lemma~\ref{lem:bnd-Ajprime}. Next, assume for $j\ge j_0$ that $\pi_{E,j}$ was already selected and select $\pi_{E,j+1}$ to be any element $\dS_{j+1}\in A_{j+1}'$ such that $R(\pi_{E,j},\dS_{j+1})=0$. Such an element exists since $\pi_{E,j} \in A_{j}'$, hence by definition of $A_{j}'$,
\[
\mu_{j+1}\lp(\lp\{ \dS_{j+1}\colon R(\pi_{E,j},\dS_{j+1})=1 \rp\} \rp)
\le \frac{1}{3\cdot 2^{2^{j+1}}},
\]
while $\mu_{j+1}(A_{j+1}') \ge 2/(3\cdot 2^{2^{j+1}})$ by Lemma~\ref{lem:bnd-Ajprime}.
By \eqref{eq:deviation-decomp} and since we defined $\pi_{E,j}$ such that $|Y_{\pi_{E,j_0}}| \le C_0 \sqrt{\log(1/\delta)}\sqrt{n}$, $R(\pi_{E,j},\pi_{E,j+1}) = 0$ and $\pi_{E,j} \in A_j$ for all $j$, we have that
\[
|Y_E|
\le \lp| Y_{\pi_{E,j_0}}\rp|+\sum_{j=j_0}^\infty \lp| Y_{\pi_{E,j+1}\setminus \pi_{E,j}}\rp| + \sum_{j=j_0}^\infty\lp|Y_{\pi_{E,j}\setminus\pi_{E,j+1}} \rp|
\le C_0\sqrt{n}\lp( \sqrt{\log(1/\delta)}
+ 2 \sum_{j=j_0}^\infty \epsilon_j 2^{j/2}\rp).
\]
This proves \eqref{eq:equivalent-chaining-bd} as required.

\subsection{Bounds on $\epsilon$-Nets via Fractional Covering Numbers} \label{sec:eps-nets-via-cover}
The goal of this section is to prove Lemma~\ref{lem:concentration-nets}.
First, we use the following Martingale bound (see Section~\ref{sec:martingales} for an introduction to Martingales):
\begin{lemma}[Freedman's inequality \cite{freedman1975tail}]\label{lem:freedman}
	Let $\rv{y}_0,\dots,\rv{y}_n$ be a Martingale adapted to the filtration $F_0,\dots,F_n$, such that $\sum_{i=1}^n \Ex[(\rv{y}_i - \rv{y}_{i-1})^2\mid F_{i-1}]\le s$ holds almost surely for some $s > 0$. Further, assume that $|\rv{y}_i - \rv{y}_{i-1}|\le M$ for all $i$. Then, for any $t>0$,
	\[
	\Pr\lp[\rv{y}_n - \rv{y}_0\ge t \rp]
	\le \exp\lp(-\frac{t^2}{2(s+Mt)}\rp).
	\]
\end{lemma}

We derive the following concentration bound for a dynamic set:
\begin{lemma}\label{lem:net-concentration-aux}
	Let $\dS$ be a dynamic set with $|\dS|\le m$. Let $\rvI\sim \Ber(n,p)$. Let $\rvvx=\rvvx(\A,\rvI)$. Then, for any $t\ge 0$,
	\[
	\Pr\lp[|\dS(\rvvx)\cap \rvvx_{\rvI}| \le p|\dS(\rvvx)| - t \rp]
	\le \exp(-ct^2/(mp+t)).
	\]
\end{lemma}
\begin{proof}
	This follows from Lemma~\ref{lem:freedman}. We apply this lemma with $\rv{y}_i = |\dS(\rvvx)\cap \rvvx_{\rvI\cap[i]}| - p|\dS(\rvvx) \cap \rvvx_{[i]}|$ and $F_i = \sigma(\rvI\cap[i])$. We can substitute $s = c'mp$, due to the following reason: conditioned on $F_{i-1}$, we know $\rvvx_{[i]}$, which implies that we know whether $\rvx_i \in \dS(\rvvx)$. If this holds true, then,
	\[
	\rv{y}_i - \rv{y}_{i-1} = \begin{cases}
		1 - p & \text{if $i \in \rvI$ (holds with probability $p$)} \\
		-p & \text{otherwise (holds with probability $1-p$)}
	\end{cases}.
	\]  
	By simple calculations we have $\Ex[(\rv{y}_i - \rv{y}_{i-1})^2\mid F_{i-1}]\le c'p$ in this case. If $\rvx_i \notin \dS(\rvvx)$, then $\rv{y}_i = \rv{y}_{i-1}$ and we have $\Ex[(\rv{y}_i - \rv{y}_{i-1})^2\mid F_{i-1}]= 0$. Since $\vx_i \in \dS(\rvvx)$ can hold true for at most $m$ values of $i$, then $\sum_{i=1}^n \Ex[(\rv{y}_i - \rv{y}_{i-1})^2\mid F_{i-1}] \le c'mp$. Further, we can substitute $M=1$ in Lemma~\ref{lem:freedman}, and the result follows.
\end{proof}

We are ready to prove Lemma~\ref{lem:concentration-nets}:
\begin{proof}[Proof of Lemma~\ref{lem:concentration-nets}]
	Let $\mathcal{N}$ be a $0$-net for $\E$ with minimal cardinality.
	For each $\dS \in \mathcal{N}$, let $\dS_{\le m}$ be the dynamic set that simulates $\dS$ up to the point that it has retained $m$ elements, and then it discards all the remaining elements. We apply Lemma~\ref{lem:net-concentration-aux} on $\dS_{\le m}$ to obtain that
	\begin{align*}
	\Pr[|\dS(\rvvx)|\ge m,\ |\dS(\rvvx)\cap\rvvx_{\rvI}| \le mp/2]
	\le \Pr[|\dS_{\le m}(\rvvx)|= m,\ |\dS_{\le m}(\rvvx)\cap\rvvx_{\rvI}| \le mp/2]\\
	\le \Pr[|\dS_{\le m}(\rvvx)\cap\rvvx_{\rvI}| \le p|\dS_{\le m}(\rvvx)| - mp/2]
	\le \exp(-cmp).
	\end{align*}
	
	The proof follows by a union bound over $\dS \in \mathcal{N}$.
\end{proof}

%% file: reductions-sampling-schemes.tex
\section{Reductions Between Different Sampling Schemes}\label{sec:reduction-sampling}

This section establishes a framework that enables one to obtain bounds with respect to one sampler in terms of bounds with respect to a different sampler. In particular, this shows how, given bounds on $\epsilon$-nets and $\epsilon$-approximations for the uniform sampler, one can obtain bounds for the Bernoulli and the reservoir sampler. Section~\ref{sec:reduction-intuition} presents an overview of an abstract method to reduce between two sampling schemes, that is formally presented in Section~\ref{sec:reduction-abstract}. Section~\ref{sec:uniform-to-reservoir} shows how to obtain bounds with respect to the reservoir sampler given bounds for the uniform sampler. 
Section~\ref{sec:uniform-to-bernoulli} presents bounds on the Bernoulli sampler based on bounds for the uniform sampler. 
Finally, Section~\ref{sec:bernoulli-to-uniform} provides bounds with respect to the uniform sampler based on the Bernoulli sampler, which proves the auxiliary Lemma~\ref{lem:approx-n-to-2k} and Lemma~\ref{lem:dbl-nets}. Both Section~\ref{sec:uniform-to-reservoir} and Section~\ref{sec:bernoulli-to-uniform} use the abstract reduction method of Section~\ref{sec:reduction-abstract} while Section~\ref{sec:uniform-to-bernoulli} utilizes the fact that the Bernoulli sampler can be presented as a mixture of uniform samplers $\Uni(n,k)$ for different values of $k$.

\subsection{Intuition for the Reduction Method} \label{sec:reduction-intuition}

For convenience, we consider samplers with no deletions, that are characterized by some distribution over subsets of $[n]$, such as $\Uni(n,k)$, that is the uniform distribution over subsets of $[n]$ of size $[k]$. We use $\rvI$ and $\rvI'$ to denote such random variables over subsets of $[n]$ (e.g., they can be distributed $\Ber(n,p)$ or $\Uni(n,k)$).

In these reductions, our goal is to show that one sampling scheme $\rvI$ is at least as good as a different scheme $\rvI'$. For example, that $\rvI$ attains $\epsilon$-approximations for values of $\epsilon$ smaller than those attained by $\rvI'$. In other words, we would like to say that $\rvI$ is resilient to the adversary at least as well as $\rvI'$. The above is equivalent to saying that the \emph{worst} adversary for $\rvI'$ \emph{is at least as bad} as the worst adversary for $\rvI$. 
The above can be shown by reduction: given an adversary $\A$ that is \emph{bad} for $\rvI$, we will construct an adversary $\A'$ that is bad for $\rvI'$.

Here we define $\A'$, that plays against a sampler that samples $\rvI'$. The general idea for $\A'$ is to simulate $\A$. However, $\A$ is known to be bad against $\rvI$ while $\A'$ plays against $\rvI'$. To tackle this issue, $\A'$ will simulate a sample $\rv{J}$ that has the same distribution as $\rvI$ and then simulate the actions of $\A$ against $\rv{J}$. Then, $\A'$ will output the same stream output by the simulated $\A$.

We would like to show that $\A'$ is bad against $\rvI'$. In order to show that, we will have to assume that the simulated sample $\rv{J}$ is very close to the true sample $\rvI'$ with high probability (say, in symmetric difference of sets). Since $\A$ is bad against $\rvI$ and $\rv{J}\sim \rvI$, the simulated actions of $\A$ are bad against the simulated sample $\rv{J}$. Since further $\rv{J}$ is very close to $\rvI'$, then the simulated $\A$ is bad also against $\rvI'$. Since $\A'$ outputs the same stream as the simulated $\A$, this implies that $\A'$ is bad against $\rvI'$, as required.

Notice that since $\A'$ would like the simulated sample $\rv{J}$ to be similar to the true sample $\rvI'$, it has to construct $\rv{J}$ based on $\rvI'$ and this defines a joint probability distribution between $\rvI'$ and $\rv{J}$. Such a joint distribution is called \emph{coupling}.
Further, the simulation has to be performed in an online fashion: once $\A'$ receives the actions taken by the sampler $\rvI'$ that it plays against, it has to immediately simulate the actions of the simulated sample $\rv{J}$. In particular, once $\A'$ knows whether $t \in \rvI'$, it has to decide whether $t \in \rv{J}$. Since the coupling between $\rv{J}$ and $\rv{I}'$ is constructed in an online fashion, we denote it an \emph{online coupling}.

Using the notation above, the goal of $\A'$ is to construct an online coupling of $\rv{J}$ and $\rvI'$ such that $\rv{J}\sim\rvI$ and such that with high probability, the symmetric set difference between $\rv{J}$ and $\rvI'$ is small. This can be done, for example, if $\rv{J}\sim \Uni(2k,k)$ and $\rvI'\sim\Ber(2k,1/2)$: there, $\A'$ will have to omit or add a approximately $O(\sqrt{k})$ elements to $\rvI'$ to create $\rv{J}$.

\subsection{Abstract Reduction Method} \label{sec:reduction-abstract}

We refer to sampling schemes that are oblivious to the adversary, namely that the choice to retain or discard an element is independent of the stream. Formally, we denote by $\rvI_t$ the set of indices of elements retained by the algorithm after seeing $\vx_{[t]}$, for $t \in [n]$. An \emph{oblivious sampling scheme} is one where $\rvvI= (\rvI_1,\dots,\rvI_n)$ are random variables jointly distributed, that are not a function of the adversary $\A$. This section compares one oblivious sampling scheme $\rvvI:=(\rvI_1,\dots,\rvI_n)$ with another, $\rvvI'=(\rvI'_1,\dots,\rvI'_n)$. It will be shown that if the adversary, given an input stream $\rvvI'$ can simulate a stream that is distributed according to $\rvvI$ such that with high probability, the input stream is close to the output stream in some sense, then the sampling scheme $\rvvI$ is at least as resilient to the adversary as $\rvvI'$. We begin with the following definition of online simulation:
\begin{definition}
	An \emph{online simulator} $\mathcal{S}$ is an algorithm that receives a stream $\rvI_1, \rvI_2,\dots, \rvI_n$ of subsets of $[n]$ and an unlimited pool of independent random bits and outputs a stream $\rvI'_1,\dots,\rvI'_n$, such that $\rvI'_t$ has to be computed before seeing $\rvI_{t+1}$, for $t \in [n]$. In other words, $\rvI'_t$ depends only on $\rvI_1,\dots,\rvI_t,\rvI'_1,\dots,\rvI'_{t-1}$ and on the randomness of the simulator.
	The joint distribution of $\rvvI$ and $\rvvI'$ is called an \emph{online coupling} of $\rvvI$ to $\rvvI'$. Equivalent, we can say that $\rvvI$ is online coupled to $\rvvI'$.
\end{definition}
Notice that an online coupling of $\rvvI$ and $\rvvI'$ is also a \emph{coupling} of these two random variables, which is any joint distribution between them. Further, notice that an online coupling is not a symmetric notion: an online coupling of $\rvvI$ to $\rvvI'$ is not necessarily an online coupling of $\rvvI'$ to $\rvvI$.

In cases that the sampler cannot delete elements from its sample, notice that $\rvI_t =\rvI_n \cap [t]$. To simplify the notation, we can write $\rvI = \rvI_n$ and $\rvI\cap[t] = \rvI_t$. Hence, we have the following definition of online coupling for no-deletion samplers:
\begin{definition}
	Let $\rvI$ and $\rvI'$ be jointly distributed random variables over $[n]$.
	We say that $\rvI$ is \emph{online coupled} to $\rvI'$ if $(\rvI\cap[1],\rvI\cap[2],\dots,\rvI\cap[n])$ is online coupled to 
	$(\rvI'\cap[1],\rvI'\cap[2],\dots,\rvI'\cap[n])$.
\end{definition}
In order to show that $\rvvI$ is more resilient to the adversary than $\rvvI'$, it suffices to find an online coupling of $\rvI$ to $\rvI'$ such that $\rvI_n$ is similar to $\rvI'_n$ in some sense. To be more formal, 
let $f(I_n,\vx)$ be some $\{0,1\}$-valued function that we view as an indicator denoting whether the sub-sample $\vx_{I_n}$ fails to represent the full stream $\vx$. For example, $f$ can be an indicator of whether $\vx_{I_n}$ is \emph{not} an $\epsilon$-approximation of $\vx$. Our goal is to bound the failure probability with the worst adversary. Namely, to bound $\max_{\A\in \Adv_n} \Pr_{\rvvI}[f(\rvI_n, \vx(\A,\rvvI))=1]$. Say that we already know how to bound a similar quantity for a different sampling scheme $\rvvI'$. If we can online couple $\rvvI$ to $\rvvI'$, then we can reduce between these two bounds:
\begin{lemma}\label{lem:coupling-reduction}
	Let $\rvvI$ be online coupled to $\rvvI'$. Let $f,g \colon \{0,1\}^n \times X^n \to \{0,1\}$. Then,
	\begin{align}
	&\max_{\A\in\Adv_n} \Pr_{\rvvI}[f(\rvI_n,\rvvx(\A,\rvvI))=1]\notag\\
	&\quad \le \max_{\A'\in\Adv_n} \Pr_{\rvvI'}[g(\rvI'_n,\rvvx(\A',\rvvI'))=1]
	+ \Pr_{\rvvI,\rvvI'}[\exists \vx \in X^n \text{ s.t } f(\rvI_n,\vx) = 1 \text{ and } g(\rvI'_n,\vx) = 0].\label{eq:91}
	\end{align}
\end{lemma}
Notice the second term in the right hand side of \eqref{eq:91}: it equals zero if $\rvI_n = \rvI_n'$, and, it is expected to be small if $\rvI_n\approx \rvI_n'$ with high probability.
\begin{proof}
	Let $\A_{\max}$ be the adversary that achieves the maximum on the left hand side of \eqref{eq:91}. Let $\mathcal{S}$ denote the simulation adversary that given $\rvvI'$ and some additional random string $\rv{r}$, outputs $\rvvI = \mathcal{S}(\rvvI',\rv{r})$. We will create the following adversary $\A'_{\rv{r}}$ that operates on the stream $\rvvI'$ and has additional randomness $\rv{r}$: it creates the sample $\rvvI = \mathcal{S}(\rvvI',\rv{r})$, simulates $\A_{\max}$ on this sample and outputs the same stream as the simulated $\A_{\max}$. In particular, we have 
	\begin{equation}\label{eq:92}
	\vx(\A_{\rv{r}}, \rvvI') = \vx(\A_{\max},\rvvI).
	\end{equation}
	We view $\A'_{\rv{r}}$ as a distribution over deterministic adversaries $\{\A'_{r}\}_{r \in \mathrm{support}(\rv{r})}$. Further, notice that each $\A'_r$ defines an appropriate adversary. By \eqref{eq:92},
	\begin{align*}
		&\max_{\A\in\Adv_n} \Pr_{\rvvI}\lp[f(\rvI_n,\vx(\A,\rvvI))=1\rp]
		= \Pr_{\rvvI}\lp[f(\rvI_n,\vx(\A_{\max},\rvvI))=1\rp] \\
		&\quad = \Pr_{\rvvI,\rvvI',\rv{r}}\lp[f(\rvI_n,\vx(\A_{\max},\rvvI))=1 \text{ and } g(\rvI'_n,\vx(\A'_{\rv{r}},\rvvI'))=1\rp] \\ 
		& \qquad + \Pr_{\rvvI,\rvvI',\rv{r}}\lp[f(\rvI_n,\vx(\A_{\max},\rvvI))=1 \text{ and } g(\rvI'_n,\vx(\A'_{\rv{r}},\rvvI'))=0\rp] \\
		&\quad= \Pr_{\rvvI,\rvvI',\rv{r}}\lp[f(\rvI_n,\vx(\A_{\max},\rvvI))=1 \text{ and } g(\rvI'_n,\vx(\A'_{\rv{r}},\rvvI'))=1\rp] \\ 
		&\qquad+ \Pr_{\rvvI,\rvvI',\rv{r}}\lp[f(\rvI_n,\vx(\A_{\max},\rvvI))=1 \text{ and } g(\rvI'_n,\vx(\A_{\max},\rvvI))=0\rp] \\
		&\quad\le \Pr_{\rvvI',\rv{r}}\lp[g(\rvI'_n,\vx(\A'_{\rv{r}},\rvvI'))=1\rp]
		+ \Pr_{\rvvI,\rvvI'}\lp[\exists \vx \in X^n \text{ s.t } f(\rvvI,\vx) = 1 \text{ and } g(\rvvI',\vx) = 0\rp],
	\end{align*}
	as required.
\end{proof}

\subsection{Bounds for Reservoir Sampling via Uniform Sampling} \label{sec:uniform-to-reservoir}

Next, we show how to obtain bounds for reservoir sampling based on uniform sampling. The intuition is that the reservoir sampler gives less information than the uniform sampler: indeed, when the reservoir sampler selects an element, the adversary does not know whether this element will remain for the final sample, while this is not the case for the uniform sampler. 
The following holds:
\begin{lemma} \label{lem:uniform-to-reservoir}
	The reservoir sampler $\rvvI = (\rvI_1,\dots,\rvI_n) \sim \Res(n,k)$ can be online coupled to the Uniform sampler $\rvI' \sim \Uni(n,k)$ such that $\rvI_n = \rvI'$ with probability $1$.
\end{lemma}
Notice that we describe the uniform sampler using one index-set as it is an insertion-only scheme, while the reservoir sample has deletions hence we describe it using $n$ index-sets.
Lemma~\ref{lem:uniform-to-reservoir}, in combination with Lemma~\ref{lem:coupling-reduction}, immediately implies that any high probability bound obtained for the uniform sampling, also holds true for the reservoir sampler:
\begin{proof}[Proof of Theorem~\ref{thm:approx}, reservoir sampling]
Let $f(I,\vx) = g(I,\vx)$ denote an indicator of whether $\vx_I$ fails to be an $\epsilon$-approximation for $\vx$. Let $\rvvI$ denote the reservoir sampler and let $\rvI'$ denote the uniform sampler. Then, by \Cref{lem:uniform-to-reservoir}, we can online couple $\rvvI$ to $\rvI'$ such that $\rvI_n = \rvI'$. By Lemma~\ref{lem:coupling-reduction}, we have
\[
\max_{\A\in\Adv_n} \Pr_{\rvvI}\lp[f(\rvI_n,\vx(\A,\rvvI))=1\rp]
\le \max_{\A'\in\Adv_n} \Pr_{\rvI'}\lp[g(\rvI',\vx(\A',\rvI'))=1\rp].
\]
By Theorem~\ref{thm:ind-of-n}, the right hand side is bounded by $\delta$, for a suitable value of $\delta$. Hence, the left hand side is bounded by the same quantity.

Lastly, notice that the assumption $n\ge 2k$ in \Cref{thm:ind-of-n} translates to $n\ge 3k$ in the reduction \Cref{lem:uniform-to-reservoir}.
\end{proof}
\begin{proof}[Proof of Theorem~\ref{thm:eps-nets}, reservoir sampling]
	The proof follows the same steps as the proof for Theorem~\ref{thm:approx}, while replacing $\eps$-approximations with $\eps$-nets and using Theorem~\ref{thm:nets-reformulation} for the bound on the uniform sampler.
\end{proof}
Finally, we prove Lemma~\ref{lem:uniform-to-reservoir}. First, an auxiliary lemma:

\begin{lemma} \label{lem:conditonal-ind}
	Let $\rvvI \sim \Res(n,k)$ and fix $t \in [n]$. Then, conditioned on $\rvI_1,\dots,\rvI_{t-1}, \rvI_n\cap [t]$, it holds that $\rvI_t$ is independent of $\rvI_n$.
\end{lemma}
\begin{proof}
	The proof follows from the following steps:
	\begin{itemize}
		\item First, the conditional distribution of $\rvI_n$ conditioned on $\rvI_1=I_1,\dots,\rvI_t=I_t$ is only a function of $I_t$. That is due to the fact that $\rvI_1\to \rvI_2 \to \cdots \to \rvI_n$ is a Markov chain.
		\item This implies that the conditional distribution of $\rvI_n$ conditioned on  $\rvI_1=I_1,\dots,\rvI_t=I_t, \rvI_n\cap[t]=S$ is only a function of $S$ and $I_t$.
		\item Conditioned on $\rvI_1=I_1,\dots,\rvI_t=I_t, \rvI_n\cap[t]=S$, we can write $\rvI_t = S \cup (I_t \setminus S)$. Notice that due to the symmetry of deletion, namely, that the deleted element is chosen uniformly at random, one derives that the conditional probability of $\rvI_n$ conditioned on $\rvI_1=I_1,\dots,\rvI_t=I_t, \rvI_n\cap[t]=S$ is not dependent on $I_t \setminus S$, hence it is only a function of $S$.
		\item The above implies that conditioned on $\rvI_n\cap[t] = S$, the random vector $(\rvI_1,\dots,\rvI_t)$ is independent of $\rvI_n$.
		\item This further implies that conditioned on $\rvI_1,\dots,\rvI_{t-1},\rvI_n\cap[t]$, it holds that $\rvI_t$ is independent of $\rvI_n$.
	\end{itemize}
\end{proof}
The following is a well-known fact that can be proved by induction:
\begin{lemma}\label{lem:res-is-uni}
	Let $\rvvI=(\rvI_1,\dots,\rvI_n)\sim \Res(n,k)$. Then, $\rvI_n\sim \Uni(n,k)$.
\end{lemma}

Using only Lemma~\ref{lem:conditonal-ind} and Lemma~\ref{lem:res-is-uni}, we can prove Lemma~\ref{lem:uniform-to-reservoir}:

\begin{proof}[Proof of Lemma~\ref{lem:uniform-to-reservoir}]
	Let $\rvI'\sim\Uni(n,k)$ and $\rvvI\sim \Res(n,k)$. We will define a random varible $\rvv{J}$ that is online coupled to $\rvI'$ and show that both $\rv{J}_n = \rvI'$ with probability $1$ and that $\rvv{J}$ has the same distribution as $\rvvI$. The sample $\rvv{J}$ is created using the following inductive argument: for $t = 1,\dots,n$, assume that we have already set $\rv{J}_1 = J_1,\dots,\rv{J}_{t-1}=J_{t-1}$, and that $\rvI' = I'$, and recall that by definition of online simulation, we can set $\rv{J}_t$ to be any randomized function of $J_1,\dots,J_{t-1},I'\cap[t]$. Specifically, $\rv{J}_t$ is drawn from the following conditional distribution: for any $J_t$,
	\begin{align}
	&\Pr\lp[\rv{J}_t = J_t \mid \rv{J}_1=J_1,\dots,\rv{J}_{t-1}=J_{t-1},\rvI'\cap[t]=I'\cap[t] \rp]\notag\\
	&\quad = \Pr\lp[\rvI_t = J_t \mid \rvI_1=J_1,\dots,\rvI_{t-1}=J_{t-1},\rvI_n\cap[t]=I'\cap[t] \rp]\notag
	\end{align}

	Using the fact that the random coins used by the algorithm are independent of the sample $\rvI'$, we derive that conditioned on $\rv{J}_1,\dots,\rv{J}_{t-1},\rvI'\cap[t]$, it holds that $\rv{J}_t$ is independent of $\rvI'$. In combination with Lemma~\ref{lem:conditonal-ind}, it follows that
	\begin{align}\label{eq:26}
	&\Pr\lp[\rv{J}_t = J_t \mid \rv{J}_1=J_1,\dots,\rv{J}_{t-1}=J_{t-1},\rvI'=I' \rp]\notag\\
	&\quad =\Pr\lp[\rv{J}_t = J_t \mid \rv{J}_1=J_1,\dots,\rv{J}_{t-1}=J_{t-1},\rvI'\cap[t]=I'\cap[t] \rp]\notag\\
	&\quad = \Pr\lp[\rvI_t = J_t \mid \rvI_1=J_1,\dots,\rvI_{t-1}=J_{t-1},\rvI_n\cap[t]=I'\cap[t] \rp]\notag\\
	&\quad = \Pr\lp[\rvI_t = J_t \mid \rvI_1=J_1,\dots,\rvI_{t-1}=J_{t-1},\rvI_n=I' \rp].
	\end{align}
	
	The following inductive argument shows that for $t \in \{0,\dots,n\}$, the conditional distribution $\rv{J}_1,\dots,\rv{J}_t$ conditioned on $\rvI'=I'$ equals the conditional distribution of $\rvI_1,\dots,\rvI_t$ conditioned on $\rvI_n = I'$.	
	For the base of the induction, $t=0$, there is nothing to prove. For the induction step, assume that the above holds for $t-1$ and we will prove for $t$. Indeed, by the chain rule, the induction hypothesis, and \eqref{eq:26},
	\begin{align}
	&\Pr[\rv{J}_1=J_1,\dots,\rv{J}_t=J_t \mid \rvI'=I']\\
	&\quad= \Pr[\rv{J}_1=J_1,\dots,\rv{J}_{t-1}=J_{t-1} \mid \rvI'=I']
	\Pr[\rv{J}_t=J_t \mid \rv{J}_1=J_1,\dots,\rv{J}_{t-1}=J_{t-1}, \rvI'=I']\\
	&\quad= \Pr[\rvI_1=J_1,\dots,\rvI_{t-1}=J_{t-1} \mid \rvI_n=I']
	\Pr[\rvI_t=J_t \mid \rvI_1=J_1,\dots,\rvI_{t-1}=J_{t-1}, \rvI_n=I']\\
	&\quad=\Pr[\rvI_1=J_1,\dots,\rvI_t=J_t \mid \rvI_n=I'].\label{eq:77}
	\end{align}
	This concludes the induction. It follows that $\rv{J}_n = \rvI'$ with probability $1$, since the conditional distribution of $\rv{J}_n$ conditioned on $\rvI'=I'$ equals the conditional distribution of $\rvI_n$ conditioned on $\rvI_n = I'$, which constantly equals $I'$. This proves one of the guarantees on $\rvv{J}$. Further, it implies that $\rv{J}_n$ has the same distribution as $\rvI'$, which, by Lemma~\ref{lem:res-is-uni} implies that $\rv{J}_n$ is distributed as $\rvI_n$. In combination with \eqref{eq:77}, we derive that
	\begin{align*}
	\Pr[\rv{J}_1=J_1,\dots,\rv{J}_n=J_n]
	&=\Pr[\rv{J}_n=J_n]
	\Pr[\rv{J}_1=J_1,\dots, \rv{J}_{n-1}=J_{n-1} \mid \rv{J}_n = J_n]\\
	&=\Pr[\rv{J}_n=J_n]
	\Pr[\rv{J}_1=J_1,\dots, \rv{J}_{n-1}=J_{n-1} \mid \rvI' = J_n]\\
	&=\Pr[\rv{I}_n=I_n]
	\Pr[\rv{I}_1=J_1,\dots, \rvI_{n-1}=J_{n-1} \mid \rv{I}_n = J_n]\\
	&=\Pr[\rvI_1=J_1,\dots,\rvI_n=J_n],
	\end{align*}
	as required.
\end{proof}

\subsection{Bounds for Bernoulli Sampling via Uniform Sampling} \label{sec:uniform-to-bernoulli}
Here, our goal is to show that concentration guarantees on uniform sampling imply guarantees on Bernoulli sampling. Notice that the latter can be viewed as a mixture of uniform sampling schemes for different values of $k$. To be more precise, a Bernoulli sample $\Ber(n,p)$ can be obtained by first drawing $\rv{k} \sim \mathrm{Bin}(n,p)$ and then drawing a uniform sample $\Uni(n,\rv{k})$, where $\mathrm{Bin}$ denotes the binomial distribution. For this reason, it is \emph{harder} to be adversarial against a Bernoulli sampler, because the adversary there does not know in advance what value of $\rv{k}$ is drawn.

To formalize this notion, we use $f(I,\vx)$ as some indicator of failure of the sample $\vx_I$ to represent $\vx$, for example, an indicator of whether $\vx_I$ is not an $\epsilon$-approximation of $\vx$. One would like to minimize the probability that $f=1$, against any adversary. The following lemma compares the failure probability of the Bernoulli sampling with that of the uniform sampling:
\begin{lemma}\label{lem:bernoulli-from-uniform}
	Let $n \in \mathbb{N}$ and $p \in [0,1]$.
	Let $f \colon \{0,1\}^n\times X^n \to \{0,1\}$. Then,
	\begin{align*}
	&\max_{\A\in\Adv_n} \Pr_{\rvI \sim \Ber(n,p)}\lp[
	f(\rvI,\vx(\A,\rvI)) = 1
	\rp]\\
	&\quad\le \max_{\substack{k \in \mathbb{N} \colon \\np/2 \le k \le 3np/2}}
	\max_{\A^k\in \Adv_n} \Pr_{\rvI^k \sim \Uni(n,k)}\lp[
	f(\rvI^k,\vx(\A^k,\rvI^k)) = 1
	\rp]
	+ 2\exp(-cnp).
	\end{align*}
\end{lemma}

The proof relies on a variant of Bernstein's inequality, on the tail of the binomial random variable:
\begin{lemma}[Bernstein's inequality]\label{lem:mult-chernoff}
	Let $\rv{k} \sim \mathrm{Bin}(n,p)$. Then, for any $\epsilon \in [0,1]$,
	\[
	\Pr[|\rv{k} -np|\ge \epsilon np]
	\le 2\exp(-\epsilon^2 np/3).
	\]
\end{lemma}

\begin{proof}[Proof of Lemma~\ref{lem:bernoulli-from-uniform}]
	Let $\A$ denote the maximizer with respect to the Bernoulli sample. First, we decompose the Bernoulli sample into a mixture of uniform samples. Let $\rv{k} \sim \Bin(n,p)$ drawn from a binomial distribution; we have
	\begin{align*}
	\Pr_{\rvI \sim \Ber(n,p)} \lp[f(\rvI,\vx(\A,\rvI)) = 1\rp]
	= \sum_{k=0}^n \Pr[\rv{k} = k] \Pr_{\rvI^k \sim \Uni(n,k)} \lp[f(\rvI^k,\vx(\A,\rvI^k)) = 1\rp].
	\end{align*}
	First, summing the terms corresponding to $np/2 \le k \le 3np/2$, we have
	\begin{align*}
	&\sum_{k=\lceil np/2\rceil}^{\lfloor 3np/2\rfloor} \Pr[\rv{k} = k] \Pr_{\rvI^k \sim \Uni(n,k)} \lp[f(\rvI^k,\vx(\A,\rvI^k)) = 1\rp]
	\le \max_{k \colon np/2\le k \le 3np/2} \Pr_{\rvI^k \sim \Uni(n,k)} \lp[f(\rvI^k,\vx(\A,\rvI^k)) = 1\rp]\\
	&\quad\le \max_{k \colon np/2\le k \le 3np/2} \max_{\A^k\in\Adv_n}\Pr_{\rvI^k \sim \Uni(n,k)} \lp[f(\rvI^k,\vx(\A^k,\rvI^k)) = 1\rp].
	\end{align*}
	Next, the sum in the remaining terms is bounded by the probability that $\rv{k} \notin [np/2,3np/2]$, or equivalently, the probability that $|\rv{k}-np|> np/2$. From Lemma~\ref{lem:mult-chernoff}, this is bounded by $2\exp(-cnp)$.
\end{proof}

As a direct application, we derive Theorem~\ref{thm:approx} and Theorem~\ref{thm:eps-nets} for the Bernoulli sampling, given the bounds corresponding to the uniform sampling:
\begin{proof}[Proof of Theorem~\ref{thm:approx}, Bernoulli sampling]
	Let $\delta \in (0,1/2)$. Define by $f(I,\vx)$ the indicator of whether $\vx_I$ fails to be an $\epsilon$-approximation for $\vx$, where $\epsilon = C_0\sqrt{(d+\log(1/\delta))/(np)}$ and $C_0>0$ is a sufficiently large constant. From Theorem~\ref{thm:ind-of-n}, for any $k$ such that $np/2 \le k \le 3np/2$ and any $\A\in\Adv_n$, it holds that 
	\[
	\Pr_{\rvI\sim\Uni(n,k)}[f(\rvI,\vx(\A,\rvI))=1]
	\le \delta/2.
	\]
	By Lemma~\ref{lem:bernoulli-from-uniform}, it follows that for any $\A\in\Adv_n$,
	\begin{equation}\label{eq:41}
	\Pr_{\rvI\sim\Uni(n,k)}[f(\rvI,\vx(\A,\rvI))=1]
	\le \delta/2 + \exp(-c_0np),
	\end{equation}
	for some universal constant $c_0>0$.
	Notice that if $\epsilon > 1$ then the result trivially follows and otherwise, we have that 
	\[
	np \ge C_0^2 \log(1/\delta).
	\]
	Assuming that $C_0$ is sufficiently large, we have that
	\[
	c_0np \ge \log(2/\delta),
	\]
	which implies that 
	\[
	\exp(-c_0np) \le \delta/2.
	\]
	In combination with \eqref{eq:41}, this concludes the proof. Notice that the condition that $2k\le n$ in Theorem~\ref{thm:ind-of-n} translates here to $3np\le n$.
\end{proof}
\begin{proof}[Proof of Theorem~\ref{thm:eps-nets}, Bernoulli sampling]
	The proof follows similar steps as the proof for Theorem~\ref{thm:approx}, while replacing $\eps$-approximations with $\eps$-nets and using Theorem~\ref{thm:nets-reformulation} for the bound on the uniform sampler.
\end{proof}

\subsection{Bounds for Uniform Sampling via Bernoulli Sampling} \label{sec:bernoulli-to-uniform}

This section reduces bounds for the uniform sampler $\rvI\sim \Uni(2k,k)$ to bounds for the Bernoulli sampler $\rvI'\sim \Ber(2k,p)$. This is done via the method of online coupling, presented in Section~\ref{sec:reduction-abstract}. First, we present some simple well known properties of binary random variables, and then, we describe an online coupling of $\rvI$ to $\rvI'$:
\begin{lemma}\label{lem:aux-monotone}
	If $\rv{y}$ and $\rv{z}$ are two random variables over $\{0,1\}$, then there exists a coupling (i.e. joint distribution) of them such that:
	\begin{enumerate}
		\item
			$\Pr[\rv{y}\ne \rv{z}] = |\Pr[\rv{y}=1] - \Pr[\rv{z}=1]|$. 
		\item
			If $\Pr[\rv{y}=1] \ge \Pr[\rv{z}=1]$ then $\rv{y} \ge \rv{z}$ with probability $1$.
	\end{enumerate}
\end{lemma}
\begin{proof}
	One can couple the following way: first, draw a random variable $\rv{\xi}$ uniformly in $[0,1]$ and set $\rv{y} = 0$ if $\rv{\xi}\le \Pr[\rv{y}=0]$ and $\rv{z} = 0$ if $\rv{\xi} \le \Pr[\rv{z}=0]$. This satisfies the requirements of the lemma.
\end{proof}
To online couple $\rvI$ to $\rvI'$ we use the coupling guaranteed from the next lemma:
\begin{lemma}\label{lem:monotone-coupling}
	Fix $k \in \mathbb{N}$ and $p \in (0,1)$, and let $\rvI\sim \Uni(2k,k)$ and $\rvI' \sim \Ber(2k,p)$. Then, $\rvI$ can be online-coupled to $\rvI'$ such that for any $t \in [n]$ and any $I_{t-1}, I_{t-1}'\subseteq[t-1]$, the following holds:
	\begin{enumerate}
	\item
	\begin{align}
	&\Pr\lp[t \in \rvI \triangle \rvI' \mid \rvI\cap [t-1]=I_{t-1}, \rvI'\cap[t-1]=I_{t-1}' \rp]\notag\\
	&\quad= \lp| \Pr\lp[t \in \rvI \mid \rvI\cap [t-1]=I_{t-1}\rp] 
	- \Pr\lp[t \in \rvI' \mid \rvI'\cap[t-1]=I_{t-1}' \rp]
	\rp|\notag\\
	&\quad= \lp|\frac{k-|I_{t-1}|}{2k-(t-1)} - p \rp|, \label{eq:monotone}
	\end{align}
	where $\triangle$ denotes the symmetric set difference $A \triangle B = (A\setminus B)\cup (B\setminus A)$.
	\item
	For any $t\in [n]$ such that $p \le (k-|\rv{I} \cap [t-1]|)/(2k-t+1)$, it holds that $\rvI'\cap\{t\} \subseteq \rvI\cap\{t\}$.
	\end{enumerate}
\end{lemma}
\begin{proof}[Proof of Lemma~\ref{lem:monotone-coupling}]
	By Lemma~\ref{lem:aux-monotone}, conditioned on $\rvI\cap [t-1]=I_{t-1}, \rvI'\cap[t-1]=I_{t-1}'$, there is a joint probability distribution between the indicators $\mathds{1}_{t \in \rvI}$ and $\mathds{1}_{t\in \rvI'}$ such that
	\begin{align*}
	&\Pr\lp[t \in \rvI \triangle \rvI' \mid \rvI\cap [t-1]=I_{t-1}, \rvI'\cap[t-1]=I_{t-1}'\rp]\\
	&\quad=\Pr\lp[\mathds{1}_{t \in \rvI}\ne \mathds{1}_{t \in \rvI'} \mid \rvI\cap [t-1]=I_{t-1}, \rvI'\cap[t-1]=I_{t-1}'\rp]\\
	&\quad= \lp| \Pr\lp[\mathds{1}_{t \in \rvI}=1 \mid \rvI\cap [t-1]=I_{t-1}\rp] 
	- \Pr\lp[\mathds{1}_{t \in \rvI'}=1 \mid \rvI'\cap[t-1]=I_{t-1}' \rp]
	\rp|\\
	&\quad= \lp| \Pr\lp[t \in \rvI \mid \rvI\cap [t-1]=I_{t-1}\rp] 
	- \Pr\lp[t \in \rvI' \mid \rvI'\cap[t-1]=I_{t-1}' \rp]
	\rp|.
	\end{align*}
	Define the online sampler to obey this property, namely,
	while receiving the value of $\mathds{1}_{t\in \rvI'}$, it can sample $\mathds{1}_{t \in \rvI}$ from its conditional distribution, conditioned on the obtained value of $\mathds{1}_{t\in \rvI'}$ in the above coupling.
	This proves property 1. Property 2 follows from property 2 in Lemma~\ref{lem:aux-monotone}.
\end{proof}

Notice that there is a unique way to define a coupling that satisfies the properties above and we term it the \emph{online monotone coupling}. We proceed with applying the monotone online coupling to prove Lemma~\ref{lem:approx-comparison} and Lemma~\ref{lem:epsnet-reduction}.

\subsubsection{Proof of Lemma~\ref{lem:approx-comparison}} \label{sec:reduction-approx}
We will in fact prove a more general lemma. We let $\varphi(I,\vx)$ be some real valued function, that we view as some loss corresponding to how the sample $\vx_I$ represents the complete stream $\vx$. We have the following:
\begin{lemma}\label{lem:Lipschitz}
	Let $\varphi \colon \{0,1\}^{2k} \times X^{2k} \to \mathbb{R}$ be a function that is $L$-Lipschitz in each coordinate of $I$, namely, for all $\vx \in X^n$,
	\[
	|\varphi(I,\vx)-\varphi(I',\vx)| \le L |I\triangle I'|.
	\]
	Then, for any $t\ge 0$ and $\delta \in (0,1/2)$,
	\begin{align*}
	& \sup_{\A\in\Adv_{2k}}\Pr_{\rvI\sim \Uni(2k,k)}[\varphi(\rvI,\vx(\A,\rvI))\ge t]\\
	&\quad\le \sup_{\A'\in\Adv_{2k}}\Pr_{\rvI'\sim \Ber(2k,1/2)}\lp[\varphi(\rvI',\vx(\A',\rvI'))\ge t - CL\sqrt{k\log (1/\delta)}\rp]+\delta.
	\end{align*}
\end{lemma}
This directly implies Lemma~\ref{lem:approx-comparison}:
\begin{proof}[Proof of Lemma~\ref{lem:approx-comparison}]
	Apply Lemma~\ref{lem:Lipschitz} with 
	\[
	\varphi(I,\vx)
	= \max_{E \in \E} \lp| \frac{|E\cap \vx_I| - |E\cap \vx_{[2k]\setminus I}|}{k} \rp|.
	\]
	This function is $L=1/k$-Lipschitz with respect to each coordinate of $I$, as the maximum of $L$-Lipschitz functions is $L$-Lipschitz itself. This suffices to conclude the proof.
\end{proof}

To prove Lemma~\ref{lem:Lipschitz}, we start with the following auxiliary property:
\begin{lemma} \label{lem:reduction-aux-approx}
	Let $\rvI \sim \Uni(2k,k)$ and $\rvI' \sim \Ber(2k,1/2)$ be coupled according to the monotone online coupling. Then, for every $\delta \in (0,1/2)$,
	\[
	\Pr\lp[|\rvI' \triangle \rvI| \ge C\sqrt{k \log(1/\delta)}\rp]
	\le \delta.
	\]
\end{lemma}

First, one can bound the expected symmetric difference between $\rvI$ and $\rvI'$:
\begin{lemma}\label{lem:ber-uni-expected-diff}
	Let $\rvI \sim \Uni(2k,k)$ and $\rvI' \sim \Ber(2k,1/2)$ be coupled according to the monotone online coupling. Then,
	\[
	\Ex\lp[|\rvI' \triangle \rvI|\rp]
	\le C \sqrt{k}.
	\]
\end{lemma}
\begin{proof}
	Summing \eqref{eq:monotone} over all $t$ and taking expectation:
	\[
	\Ex[|\rvI\triangle \rvI'|]
	= \sum_{t=1}^{2k} \Pr[t \in \rvI \triangle \rvI']
	= \sum_{t=1}^{2k} \Ex\lp|\frac{1}{2} - \frac{k-|\rvI\cap[t-1]|}{2k-(t-1)}\rp|.
	\]
	By using Jenssen's inequality and applying Lemma~\ref{lem:var-without-rep} with $U = \{t,\dots,2k\}$, $k=k$ and $n=2k$, we have
	\begin{align*}
	&\Ex\lp|\frac{1}{2} - \frac{k-|\rvI\cap[t-1]|}{2k-(t-1)}\rp|
	\le \sqrt{\Ex\lp[\lp(\frac{1}{2} - \frac{k-|\rvI\cap[t-1]|}{2k-(t-1)} \rp)^2\rp]} \\
	&\quad= \sqrt{\mathrm{Var}\lp[\frac{|\rvI\cap[t-1]|}{2k-(t-1)}\rp]}
	= \sqrt{\mathrm{Var}\lp[\frac{|\rvI\cap\{t,\dots,2k\}|}{2k-(t-1)}\rp]}
	\\
	&\quad= \frac{1}{2k-(t-1)}\sqrt{\mathrm{Var}\lp[|\rvI\cap\{t,\dots,2k\}|\rp]}
	\le \frac{1}{2k-(t-1)}\sqrt{\frac{(2k-(t-1)) k}{2k}}\\
	&\quad= \frac{1}{2\sqrt{2k-(t-1)}}.
	\end{align*}
	Summing over $t=1,\dots, 2k$, we have
	\[
	\Ex[|\rvI\triangle \rvI'|]
	\le \sum_{t=1}^{2k} \frac{1}{2\sqrt{2k-(t-1)}}
	= \sum_{t=1}^{2k} \frac{1}{2\sqrt{t}}
	\le \sqrt{2k}.
	\]
\end{proof}
The next step is to show that with high probability, $|\rvI\cap\rvI'|$ is close to its expectation. For that, the notion of Martingales is used. We use a standard notation, that is presented in Section~\ref{sec:martingales}. In particular, we use the following commonly used corollary of Azuma's inequality (Lemma~\ref{lem:azuma}):
\begin{lemma}\label{lem:doob}
	Let $F_0 \subseteq F_1 \subseteq \cdots F_n$ be a filtration such that $F_0$ is the trivial $\sigma$-algebra. Let $\rv{y}$ be a random variable that is $F_n$ measurable, and assume that $a_1,\dots,a_n$ are numbers such that $|\Ex[\rv{y} \mid F_i] - \Ex[\rv{y}\mid F_{i-1}]| \le a_i$ holds for all $i$ with probability $1$. Then, for all $t>0$,
	\[
	\Pr[\rv{y}-\Ex[\rv{y}] > t] \le 
	\exp\lp(\frac{-t^2}{2\sum_i a_i^2}\rp).
	\]
\end{lemma}
Notice that, as described in Section~\ref{sec:martingales}, $\Ex[\rv{y} \mid F_i]$ is a random variable, and a bound of $|\Ex[\rv{y} \mid F_i] - \Ex[\rv{y}\mid F_{i-1}]| \le a_i$ states that the information that is present in $F_i$ and not in $F_{i-1}$ does not significantly affect the conditional expectation of $\rv{y}$.
\begin{proof}[Proof of Lemma~\ref{lem:doob}]
	We define the following Martingale, which is known as \emph{Doob's Martingale}: $\rv{y}_i = \Ex[\rv{y}\mid F_i]$, for $i=0,\dots,n$. Then, $\rv{y}_0 = \Ex \rv{y}$ and $\rv{y}_n = \rv{y}$, and the proof follows directly from Lemma~\ref{lem:azuma}.
\end{proof}

We are ready to bound the deviation of $|\rvI\triangle\rvI'|$:
\begin{lemma} \label{lem:ber-uni-deviation}
	Let $\rvI \sim \Uni(2k,k)$ and $\rvI' \sim \Ber(2k,1/2)$ be coupled according to the monotone online coupling. Then, for every $\delta \in (0,1/2)$,
	\[
	\Pr\lp[|\rvI' \triangle \rvI| - \Ex[|\rvI'\triangle \rvI|] \ge C\sqrt{k \log(1/\delta)}\rp]
	\le \delta.
	\]
\end{lemma}
\begin{proof}
	For any $t = 0,\dots,n$,
	let $F_t$ denote the $\sigma$-field that contains all the information up to (and including) round $t$, $F_t = \sigma(\rvI\cap[t],\rvI'\cap[t])$. 
	In order to apply Lemma~\ref{lem:doob}, it is desirable to bound the differences $\lp|\Ex[|\rvI\triangle \rvI'|\mid \mathcal{F}_t] - \Ex[|\rvI\triangle \rvI'|\mid \mathcal{F}_{t-1}]\rp|$ for all $t\in [2k]$.
	Notice that
	\begin{align}\label{eq:412}
	&\lp|\Ex\lp[|\rvI\triangle \rvI'|\mid \mathcal{F}_t\rp] - \Ex\lp[|\rvI\triangle \rvI'|\mid \mathcal{F}_{t-1}\rp] \rp|\notag\\
	&\quad=\lp|\Ex\lp[|\rvI\triangle \rvI'|\mid \rvI\cap[t],\rvI'\cap[t]\rp] - \Ex\lp[|\rvI\triangle \rvI'|\mid \rvI\cap[t-1],\rvI'\cap[t-1]\rp] \rp|.
	\end{align}
	We would like to bound \eqref{eq:412} for any realization of $\rvI$ and $\rvI'$, namely, bounding for any $S,S'\subseteq [t]$ the quantity
	\begin{align}\label{eq:413}
	&\lp|\Ex\lp[|\rvI\triangle \rvI'|\mid \rvI\cap[t]=S,\rvI'\cap[t]=S'\rp] \rp.\\
	&\quad\lp. - \Ex\lp[|\rvI\triangle \rvI'|\mid \rvI\cap[t-1]=S\cap[t-1],\rvI'\cap[t-1]=S'\cap[t-1]\rp] \rp|.
	\end{align}
	Define four random variables, $\rv{J}_t,\rv{J}_t',\rv{J}_{t-1},\rv{J}_{t-1}'$ in a joint probability space such that $(\rv{J}_t,\rv{J}_{t}')$ is distributed according to the joint distribution of $(\rvI,\rvI')$ conditioned on $\rvI\cap[t]=S,\rvI'\cap[t]=S'$ and $(\rv{J}_{t-1},\rv{J}_{t-1}')$ is distributed according to the joint distribution of $(\rvI,\rvI')$ conditioned on $\rvI\cap[t-1]=S\cap[t-1],\rvI'\cap[t-1]=S'\cap[t-1]$.
	Then, we derive that \eqref{eq:413} equals
	\begin{equation}\label{eq:414}
	|\Ex[|\rv{J}_t\triangle\rv{J}'_t|]
	-\Ex[|\rv{J}_{t-1}\triangle\rv{J}'_{t-1}|]|
	\le \Ex[||\rv{J}_t\triangle\rv{J}'_t|-|\rv{J}_{t-1}\triangle\rv{J}'_{t-1}||]
	\end{equation}
	and our goal is to bound the right hand side of \eqref{eq:414}.
	The joint distribution is defined by an inductive argument, defining for $j\ge t$ the intersection of the above four random variables with $[j]$ given the intersection with $j-1$.
	Begin with $j=t$: here, $\rv{J}_t$ and $\rv{J-1}_t$ are fixed to $S$ and $S'$, respectively. Further, the intersections of $\rv{J}_{t-1}$ and $\rv{J}_{t-1}'$ with $[t-1]$ are fixed and equal $S\cap[t-1]$, and their intersections with $[t]$ are random drawn according to the monotone online coupling. Further, for $j>t$, we start by drawing a random variable $\rv{\xi}$ uniformly in $[0,1]$, and for any $\rv{U} \in \{\rv{J}_t,\rv{J}_t',\rv{J}_{t-1},\rv{J}_{t-1}'\}$ we set $j \in \rv{U}$ if and only if $\rv{\xi} \le \Pr[j \in \rv{U} \mid \rv{U}\cap[j-1]]$. 
	The above defined joint distribution satisfies the following properties:
	\begin{itemize}
	\item From the proofs of Lemma~\ref{lem:aux-monotone} and Lemma~\ref{lem:monotone-coupling}, it follows that $(\rv{J}_t,\rv{J}_{t}')$ is distributed according to the joint distribution of $(\rvI,\rvI')$ conditioned on $\{\rvI\cap[t]=S,\rvI'\cap[t]=S'\}$ and $(\rv{J}_{t-1},\rv{J}_{t-1}')$ is distributed according to the joint distribution of $(\rvI,\rvI')$ conditioned on $\{\rvI\cap[t-1]=S\cap[t-1],\rvI'\cap[t-1]=S'\cap[t-1]\}$.
	\item For any $\rv{U},\rv{U}'\in \{\rv{J}_t,\rv{J}_t',\rv{J}_{t-1},\rv{J}_{t-1}'\}$ and any $j > t$, if $\Pr[j \in \rv{U} \mid \rv{U}\cap[j-1]] \le \Pr[j \in \rv{U}' \mid \rv{U}'\cap[j-1]]$ then $j \in \rv{U}$ implies $j \in \rv{U}'$.
	\item Notice that for any for any $j > t$ and any $\rv{U} \in \{\rv{J}_t',\rv{J}_{t-1}'\}$ it holds that $\Pr[j \in \rv{U} \mid \rv{U}\cap[j-1]] = 1/2$, hence $j \in \rv{J}_t'$ if and only if $j \in \rv{J}_{t-1}'$.
	\item For any $j > t$ and any $\rv{U} \in \{\rv{J}_t,\rv{J}_{t-1}\}$, it holds that 
	\[
	\Pr[j \in \rv{U} \mid \rv{U}\cap[j-1]] = \frac{k-|\rv{U}\cap[j-1]|}{2k-(j-1)},
	\]
	which is a monotone decreasing function of $|\rv{U}\cap[j-1]|$. This implies the following properties:
	\begin{itemize}
		\item For any $j > t$ such that $|\rv{J}_t \cap[j-1]| = |\rv{J}_{t-1}\cap[j-1]|$, it holds that $j \in \rv{J}_t$ if and only of $j \in \rv{J}_{t-1}$. 
		\item For any $j > t$ such that  $|\rv{J}_t \cap[j-1]| \ge |\rv{J}_{t-1}\cap[j-1]|$, it holds that $\Pr[j \in \rv{J}_t\mid \rv{J}_t\cap[j-1]] \le \Pr[j \in \rv{J}_{t-1}\mid \rv{J}_{t-1}\cap[j-1]]$, hence, $j \in \rv{J}_t$ implies $j\in \rv{J}_{t-1}$.
		\item For any $j>t$ such that $|\rv{J}_t \cap[j-1]| \le |\rv{J}_{t-1}\cap[j-1]|$, due to a similar argument, $j \in \rv{J}_{t-1}$ implies $j\in \rv{J}_{t}$.
	\end{itemize}
	\item From the above arguments, for any $j$, if $||\rv{J}_t \cap[j-1]|-|\rv{J}_{t-1} \cap[j-1]||=1$ then either $j \notin \rv{J}_t\triangle \rv{J}_{t-1}$ or $|\rv{J}_t \cap[j]|=|\rv{J}_{t-1} \cap[j]|$.
	\item It holds that $||\rv{J}_t \cap[t]|-|\rv{J}_{t-1} \cap[t]||\le 1$. From the above arguments, at the first $j>t$ such that $j \in \rv{J}_t \triangle \rv{J}_{t-1}$, it holds that $|\rv{J}_t \cap[j]|=|\rv{J}_{t-1} \cap[j]|$ and from that point onward, $j \notin \rv{J}_t \triangle \rv{J}_{t-1}$. In particular, there is at most one $j>t$ such that $j \in \rv{J}_t \triangle \rv{J}_{t-1}$.
	\item It follows that there are at most two values of $j$ such that $j \in \rv{J}_t \triangle \rv{J}_{t-1}$: one (possibly) for $j=t$ and one (possibly) for $j>t$. Since the only possible $j$ where $j \in \rv{J}'_t \triangle \rv{J}'_{t-1}$ is $j=t$, it follows that 
	\begin{equation}\label{eq:415}
	||\rv{J}_t\triangle\rv{J}'_t|-|\rv{J}_{t-1}\triangle\rv{J}'_{t-1}||\le 2.
	\end{equation}
	\end{itemize}
	From \eqref{eq:414} and \eqref{eq:415} we derive that 
	\[
	\lp|\Ex\lp[|\rvI\triangle \rvI'|\mid \mathcal{F}_t\rp] - \Ex\lp[|\rvI\triangle \rvI'|\mid \mathcal{F}_{t-1}\rp] \rp| \le 2.
	\]
	Applying Lemma~\ref{lem:doob}, the proof follows.
\end{proof}

Lemma~\ref{lem:ber-uni-expected-diff} and Lemma~\ref{lem:ber-uni-deviation} together imply Lemma~\ref{lem:reduction-aux-approx}.

\begin{proof}[Proof of Lemma~\ref{lem:Lipschitz}]
	We apply Lemma~\ref{lem:coupling-reduction} with
	\[
	f(\vx, I)
	= \mathds{1}\lp( \varphi(I,\vx)  \ge t \rp)
	\]
	and
	\[
	g(\vx, I') = \mathds{1}\lp( \varphi(I',\vx) \ge t - C_0 L\sqrt{k\log(1/\delta)}\rp).
	\]
	Here $C_0>0$ is the universal constant guaranteed from Lemma~\ref{lem:reduction-aux-approx} such that with probability $1-\delta$,
	\[
	|\rvI\triangle \rvI'| \le C_0 \sqrt{k\log(1/\delta)}.
	\]
	Notice that for any $I$ and $I'$ such that $|I\triangle I'| \le C_0 \sqrt{k\log(1/\delta)}$ and any $\vx\in X^n$, it holds that
	\[
	|\varphi(I,\vx) - \varphi(I',\vx)|
	\le L|I\triangle I'|
	\le C_0 L\sqrt{k\log(1/\delta)}.
	\]
	Hence, for any such $I,I'$ and any $\vx$ such that $f(I,\vx) = 1$, it holds that $g(I',\vx) = 1$. Applying Lemma~\ref{lem:coupling-reduction}, we derive that for $\rvI\sim \Uni(2k,k)$ and $\rvI'\sim\Ber(2k,1/2)$,
	\begin{align*}
	&\max_{\A\in\Adv_{2k}}\Pr[\varphi(\rvI,\vx(\A,\rvI)) \ge t]
	= \max_{\A\in\Adv_{2k}} \Pr[f(\rvI,\vx(\A,\rvI)) = 1]\\
	&\quad\le \max_{\A'\in\Adv_{2k}}\Pr[g(\rvI',\vx(\A',\rvI')) = 1] +\delta\\
	&\quad= \max_{\A'\in\Adv_{2k}}\Pr\lp[\varphi(\rvI',\vx(\A',\rvI')) \ge t - C_0L \sqrt{k\log(1/\delta)}\rp] + \delta.
	\end{align*}
\end{proof}

\subsubsection{Proof of Lemma~\ref{lem:epsnet-reduction}}\label{sec:pr-reduction-epsnets}
We prove the following property of the online monotone coupling:
\begin{lemma}\label{lem:uni-ber-eighth}
	Let $\rvI \sim \Uni(2k,k)$ be coupled to $\rvI' \sim \Ber(2k,1/8)$ according to the monotone online coupling and fix $m \le n$. Then,
	\[
	\Pr[\rvI' \subseteq \rvI \cup \{n-m+1,\dots,n\}]
	\ge 1- 2\exp(-cm).
	\]
\end{lemma}
\begin{proof}
	From Lemma~\ref{lem:monotone-coupling}, it suffices to show that with probability $1-e^{-cm}$, for all $t \le 2k-m$, 
	\begin{equation}\label{eq:244}
	1/8 \le \frac{k-|\rvI\cap[t-1]|}{2k-(t-1)}.
	\end{equation}
	Let $m_0,m_1,\dots,m_r$, such that $m_i = m \cdot 2^i$ and $k < m_r \le 2k$.
	For any $i=0,\dots,r$, apply Lemma~\ref{lem:subset-intersection} with $n=2k$, $k=k$, $U=\{n-m_i+1,\dots,2k\}$, $\alpha=1/2$ and $\rvI=\rvI$, deriving
	\begin{align}\label{eq:38}
	\Pr\lp[|\rvI' \cap (\{n-m_i+1,\dots,2k\})| \le \frac{m_i}{4} \rp]
	\le \Pr\lp[\lp| \frac{|\rvI \cap U|}{k}-\frac{|U|}{2k}\rp| \ge \frac{m_i}{4k} \rp]	
	\le 2\exp\lp(-c m_i\rp).
	\end{align}
	Summing the failure probabilities over $i=0,\dots,r$, the sum is dominated by the first summand, and we derive that with probability at least $1-2\exp(-cm)$, \eqref{eq:38} fails to hold for all $i=0,\dots,r$. Fix a realization $I$ of $\rvI$ such that \eqref{eq:38} fails for all $i$, and we will prove \eqref{eq:244}, to complete the proof. Fix $t \le 2k-m$, and let $i_0$ be the maximal $i$ such that $t \le 2k-m_i+1$. Then, $2k-(t-1) < m_{i_0+1} = 2m_{i_0}$. Further,
	\[
	k - |I \cap [t-1]| 
	= |I \cap \{t,\dots,n\}|
	\ge |I \cap \{n-m_i+1,\dots,n\}|
	\ge m_i/4.
	\]
	We derive that
	\[
	\frac{k - |I \cap [t-1]|}{2k-(t-1)} \ge \frac{1}{8}
	\]
	as required.
\end{proof}

\begin{proof}[Proof of Lemma~\ref{lem:epsnet-reduction}]
	Let $\rvI$ be coupled to $\rvI'$ according to the monotone online coupling. We wish to apply Lemma~\ref{lem:coupling-reduction} with 
	\[
	f(I,\vx) = \begin{cases}
	1 & \exists E \in \E,\ |\vx_{I\cup J} \cap E|\ge \epsilon \cdot 2k,\ 
	\vx_I \cap E = \emptyset \\
	0 &\text{otherwise}
	\end{cases},
	\]
	and
	\[
	g(I',\vx)= \begin{cases}
	1 & \exists E \in \E,
	\ |\vx\cap E| \ge \epsilon \cdot 2k  \text{ and }
	|\vx_I\cap E| \le \epsilon/16 \cdot 2k  \\
	0 & otherwise
	\end{cases}
	\]
	Applying Lemma~\ref{lem:uni-ber-eighth} with $m=\lfloor\epsilon/16 \cdot 2k\rfloor$, it holds with probability $1-2\exp(-c\epsilon k)$ that
	\begin{equation} \label{eq:71}
	\rvI'\subseteq \rvI\cup\{2k - \lfloor \epsilon/16 \cdot 2k \rfloor+1, n \}.
	\end{equation}
	For any values $I,I'$ such that \eqref{eq:71} holds and any $\vx\in X^n$ such that $f(I,\vx) = 1$, it also holds that $g(I',\vx) = 1$. 
	Indeed, let $E\in \E$ be a set such that $|E \cap \vx|\ge \epsilon \cdot 2k$ and $\vx_I \cap E = \emptyset$. From \eqref{eq:71}, 
	\[
	|\vx_{I'}\cap E| \le
	|\vx_I \cap E| + |\{2k - \lfloor \epsilon/16\cdot 2k \rfloor+1, n \} \cap E|
	\le \epsilon/16\cdot 2k,
	\]
	which implies that $g(I',\vx) = 1$. From Lemma~\ref{lem:coupling-reduction},
	\begin{align*}
	&\max_{\A\in\Adv_{2k}}\Pr\lp[\Net^{2k}_{\A,\rvI,\epsilon\cdot 2k,0}(\E) = 1\rp]
	= \max_{\A\in\Adv_{2k}} \Pr_{\rvI}[f(\rvI,\vx(\A,\rvI))=1]\\
	&\quad\le \max_{\A'\in\Adv_{2k}} \Pr_{\rvI'}[g(\rvI',\vx(\A',\rvI'))=1] + 2\exp(-c\epsilon k)\\
	&\quad= \max_{\A'\in\Adv_{2k}}\Pr\lp[\Net^{2k}_{\A',\rvI',\epsilon\cdot 2k, \epsilon/16 \cdot2k}(\E) = 1\rp] + 2\exp(-c\epsilon k).
	\end{align*}
\end{proof}

\section{Continuous $\eps$-Approximation}
\label{subsec:continuous}
In the adversarial model we discuss in this paper, the general goal is that in the end of the process (after all elements have been sent by the adversary), the obtained sample would be an $\eps$-approximation of the entire adversarial sequence. 
However, in many practical scenarios of interest, one might want the sample obtained so-far to be an $\eps$-approximation of the current adversarial sequence \emph{at any point along the sequence} (and not just at the end of the sequence). We call this condition a \emph{continuous $\eps$-approximation}. Note that such a requirement only makes sense for sampling procedures that allow deletions, like reservoir sampling. (For insertion-only samplers, like Bernoulli and uniform sampling, one cannot hope for the sample to approximate the stream until there is sufficient ``critical mass'' collected in the sample; this is not an issue with reservoir sampling, which overcomes this by sampling the first elements in the sequence with higher probability, but may also delete them later.)

Obtaining upper bounds for continuous $\eps$-approximation can be done easily by plugging-in our upper bounds for reservoir sampling to a block-box argument by Ben-Eliezer and Yogev \cite[Section 6]{BenEliezerYogev2020}. There, it is shown that if one ensures that the current sample approximates the current sequence at $O(\log n)$ carefully located ``checkpoints'' along the stream (while setting the error parameter to be $\delta' = \Theta(\delta / \log n)$), then with probability $1-\delta$, the sample is a continuous $\eps$-approximation for the sequence. That is, we have the following.
 
\begin{theorem}[Adversarial ULLNs -- Quantitative Characterization]\label{thm:continuous}
	Let $\E$ be a family with Littlestone dimension $d$. 
	Then, the sample size $k(\E,\epsilon,\delta)$, which suffices to produce a continuous $\epsilon$-approximation w.r.t $\E$ satisfies:
	\[
	k(\E,\epsilon,\delta) \le O\lp(\frac{d+\log(1/\delta)+\log \log n}{\epsilon^2}\rp).
	\]
	This bound is attained by the reservoir sampler $\Res(n,k)$.
\end{theorem}
Compared with the standard setting (as summarized in \Cref{thm:quantitative}), the bound here has an additional $\log \log n$ term in the numerator. 

%% file: online.tex
\section{Online Learning}\label{sec:onlinetechnical}

In this section, we prove an optimal bound on the regret of online classification. We first provide the formal definitions and then proceed with the formal statement and the proof.

\subsection{Formal Defintions}
Consider the setting of online prediction with binary labels;
	a learning task in this setting can be described as a guessing game between a learner and an adversary. 
	The game proceeds in rounds $t=1,\dots,T$, each consisting of the following steps:
\begin{itemize}[noitemsep]
	\item The adversary selects $(x_t,y_t) \in X\times \{0,1\}$ and reveals $x_t$ to the learner.
	\item The learner provides a prediction $\hat y_t \in \{0,1\}$ of $y_t$ and announces it to the adversary.
	\item The adversary announces $y_t$ to the learner.
\end{itemize}
Notice that both the learner and the adversary are allowed to use private randomness. 

The goal of the learner is to minimize the number of mistakes, $\sum_t \mathds{1}(y_t \ne \hat y_t)$. 
Given a class $\E$, a learner $\mathcal{L}$ and an adversary $\A$, the {\it regret} of the learner w.r.t $\E$ 
is defined as the expected difference between the number of mistakes
made by the learner and the number of  mistakes made by the best $E\in \cal E$: 
\[ 
R_T(\E,\mathcal{L},\A) := \Ex\lp[\sum_t \mathds{1}(y_t \ne \hat y_t)
- \min_{E \in \E}\sum_t \mathds{1}\Bigl(y_t \ne \mathds{1}(x_t \in E)\Bigr)\rp].\]
The \emph{optimal regret} is defined as the value of the the regret achieved by the best sampler against its worst adversary:
\[
R_T(\E) = \min_{\mathcal{L}}\max_{\A} R_T(\E,\mathcal{L},\A).
\]

\subsection{Statement and Proof}

We prove the following theorem:
\begin{theorem}\label{thm:online-formal}
    Let $\E$ denote a class of Littlestone dimension $d$. Then, the expected regret $R_T(\E)$ for a $T$-round online learner is bounded by
    \[
    R_T(\E) \le C\sqrt{dT}
    \enspace,
    \]
    where $C>0$ is a universal constant.
\end{theorem}
We use a bound by \cite{rakhlin2015online} on the regret based on the sequential Rademacher complexity:
\begin{theorem}[\cite{rakhlin2015online}, Theorem~7]
\label{thm:online-rademacher}
    The expected regret satisfies
    \[
    R_T(\E) \le 2\Rad_T(\E)
    \enspace.
    \]
\end{theorem}
We combine this with the bound on the sequential Rademacher complexity from \Cref{lem:bnd-p-half}, to complete the proof:
\begin{proof}[Proof of \Cref{thm:online-formal}]
By \Cref{thm:online-rademacher}, by definition of the sequential Rademacher complexity and by \Cref{lem:bnd-p-half},
\[
R_T(\E)
\le 2\Rad_T(\E)
    = 2\Ex_{\rvI\sim \Ber(n,1/2)}[\Disc_{\A,\rvI}(\E)]
    \le C \sqrt{dT}
    \enspace.
\]
This concludes the proof.
\end{proof}

%% file: lower-bounds.tex
\section{Lower Bounds}
\label{sec:LBs}

In this section we state and prove our lower bounds. Our first lower bound applies to \emph{any} family $\E$, showing that the linear dependence of our upper bounds in the Littlestone dimension is universally tight.

\begin{theorem}[A universal lower bound]\label{thm:lbforall}
Let $\cal E$ be a family with Littlestone dimension $d$.
Then, there exists a (deterministic) adversary such that the following holds.
For any algorithm that retains at most $k\leq d$ items (without deletions),
the adversary presents $d$ items $x_1,\ldots, x_d$ such that 
\[(\exists E\in {\cal E}): \bar s \cap E = \emptyset\ \quad \text{and}\quad \frac{\lvert \bar x \cap E \rvert}{\lvert\bar x\rvert}=1-\frac{k}{d},\]
with probability 1 over the algorithm's randomness, where $\bar x$ denotes the adversatial stream and $\bar s$ is the sample.
In particular, any subset of the sample of $k$ items retained by the algorithm does not form an $\eps$-approximation with respect to $x_1,\ldots,x_n$ unless $\eps\geq 1- \frac{k}{d}$.
\end{theorem}
Our second result in this section shows the existence of families $\cal E$ of Littlestone dimension $d$ in which all $\eps$-approximations are of size ${\Omega}(d/\eps^2)$, so long as $d = \Omega(\log(1/\eps))$. Interestingly, the requirement that $d$ is large enough is necessary: classical results in discrepancy theory \cite{MWW93disc, Matousek95} imply that when $d = o(\log 1/\eps)$, smaller $\eps$-approximations exist.

\begin{theorem}[$\eps$-approximation: quadratic lower bound]\label{thm:lbexists}
Let $d\in\mathbb{N}$ and $\eps > 0$ where $d \geq C \log(1/\eps)$ for a large absolute constant $C > 0$.
Then, there exists a family $\cal E$ with Littlestone dimension at most $d$
and a subset $\{x_1,\ldots, x_n\} \subset X$ for which no subset of size less than $c \cdot \frac{\Ldim(\mathcal{E})}{\eps^2}$ is an $\eps$-approximation, where $c>0$ is a small absolute constant.
\end{theorem}
We also prove similar results for $\eps$-nets (without the requirement that $d$ is large enough).
\begin{theorem}[$\eps$-net: a super linear lower bound]\label{thm:lbnet}
Let $d\in\mathbb{N}$ and $\eps > 0$.
Then, there exists a family $\cal E$ with Littlestone dimension $\leq d$
and a subset $\{x_1,\ldots, x_n\} \subset X$ for which no subset of length less than $c \cdot \frac{\Ldim(\mathcal{E}) \log(1/\eps)}{\eps}$ is an $\eps$-net, where $c>0$ is a small absolute constant.
\end{theorem}

\subsection{Proofs}

\begin{proof}[Proof of \Cref{thm:lbforall}]
The proof generalizes the construction from \cite{BenEliezerYogev2020}, which provided a lower bound for the family of one-dimensional thresholds.\footnote{We note that the proof from \cite{BenEliezerYogev2020} would give a lower bound of $\Omega(\log d)$ for any family of Littlestone dimension $d$ (as compared to the $\Omega(d)$ lower bound we prove here); this follows since, roughly speaking, any such family ``contains'' a class of thresholds of dimension logarithmic in $d$.}
Let $\mathcal{T}$ be a tree of depth $d$ which is shattered by $\cal E$. 
The tree $\mathcal{T}$ can be thought of as a strategy for the adversary as follows:

\begin{center}
\noindent\fbox{
\parbox{.99\columnwidth}{
\begin{enumerate}
\item Set $\mathcal{T}_1=\mathcal{T}$ and $i=1$. 
\item For $i=1,\ldots , d$
\begin{itemize}
\item[(i)] Pick $x_i$ to be the item labelling the root of $\mathcal{T}_i$ and present it to the algorithm.
\item[(ii)] If $x_i$ was retained by the algorithm then continue to the next iteration with $\mathcal{T}_{i+1}$ 
	being the left subtree of $\mathcal{T}_i$ (corresponding to the sets in ${\cal E}_{\not\ni x_i}$).
\item[(iii)] Else, continue to the next iteration with $\mathcal{T}_{i+1}$ 
	being the left subtree of $\mathcal{T}_i$ (corresponding to the sets in ${\cal E}_{\ni x_i}$).
\end{itemize}
\end{enumerate}
     }}
\end{center}
Thus, the adversary picks the elements $x_1,\ldots, x_d$ according to a  path on the tree
such that whenever $x_i$ is retained by the algorithm then a left turn is taken
and whenever $x_i$ is not retained by the algorithm then a right turn is taken.
Thus, since the tree is shattered, there exists a set $E\in\mathcal{E}$
such that
\[E\cap\{x_1,\ldots, x_n\} = \{x_i : x_i\text{ was not sampled by the algorithm}\}. \]
In particular, $\bar s \cap E = \emptyset$, and if the algorithm samples $m\leq d$ points
then $\frac{\lvert \bar x \cap E \rvert}{\lvert\bar x\rvert} = 1-\frac{m}{d}$, as required.
\end{proof}

\begin{proof}[Proof of \Cref{thm:lbexists}]
The proof follows from standard probabilistic arguments, and shows that most families in a certain setting have bounded Littlestone dimension yet do not admit a small $\eps$-approximation.
Suppose that $d \geq \log(1/\eps)$ and let $n = d/6\eps^2$. 
Let $F$ be a family of $2^d \cdot d / \eps^2$ subsets of $[n]$ of size $n/2$, picked uniformly at random among all such families, and note that (by definition and since $d \geq \log(1/\eps)$) the Littlestone dimension of $F$ is at most $\log |F| = O(d)$.

We now claim that with high probability, there is no $\eps$-approximation of size less than $n/2$ for $F$. Indeed, fix any subset $S$ of size $m \leq n/2$. By a simple counting argument, the number of sets $A$ of size $n/2$ for which $|d_A(S) - d_A([n])| \geq \epsilon$ is at least
$$
\binom{m}{\left(\frac{1}{2}-\eps\right) m} \binom{n-m}{\frac{n-m}{2}-\eps m} \geq \binom{m}{\frac{m}{2}} \binom{n-m}{\frac{n-m}{2}} \cdot (1-2\eps)^{2 \eps m} \geq \frac{2^n}{2n} \cdot e^{-3\eps^2 n} = \frac{2^n}{2n} \cdot e^{-d/2},
$$
where the second inequality holds for $\eps < 1/10$. 

Plugging in the right hand side above, and noting the negative correlation between the events at hand, the probability that $F$ does not contain any such $A$ with $|d_A(S) - d_A([n])| \geq \epsilon$ is bounded by $$
\left(1- \frac{2^n \cdot e^{-d/2}}{2n \binom{n}{n/2}}\right)^{|F|} \leq e^{-2^d \cdot e^{-d/2} } \leq e^{-1.2^d \cdot d / \eps^2}.
$$  
Taking a union bound over all (less than $2^n = 2^{d/6\eps^2}$) possible subsets $S \subseteq [n]$ of size at most $[n]/2$, it follows that with high probability (as a function of $d$), no $\eps$-approximation exists.
	\end{proof}

\begin{proof}[Proof of \Cref{thm:lbnet}]
The proof extends a simple probabilistic construction in the projective plane, suggested by Alon, Kalai, Matou\v{s}ek, and Meshulam \cite{AKMM2002}. 
	
Consider the projective plane of order $p$, where we pick $p = C/\eps$ for a suitable constant $C$. Recall that this projective plane has $p^2 + p + 1$ points and lines, where each line consists of exactly $p+1$ points, and every two points are contained in exactly one line. For each line $L$, pick uniformly at random (and independently from choices for other lines) a subset $H_L$ containing exactly half the elements of $L$; we call such a subset a \emph{half line}. Consider the family consisting of all such half lines $H_L$.
As was shown in \cite{AKMM2002}, with high probability every $\eps$-net for this family has size $\Omega(p \log p)$, whilst the VC dimension is at most $2$. We claim that the same bound also holds for the Littlestone dimension. 
\begin{claim}\label{clm:LB_Littlestone2}
The Littlestone dimension of the family consisting of all half lines as above is at most $2$.
\end{claim}
\begin{proof}
Suppose to the contrary a depth-$3$ tree exists as in the definition of the Littlestone dimension, and consider the elements $x,y,z$ appearing in the internal nodes of the all-$1$ branch in this tree. By definition, all three elements must belong to some half line $H_L$ from the family. However, since any two lines $L_1 \neq L_2$ in the projective plane intersect in exactly one point, we have $|H_{L_1} \cap H_{L_2}| \leq 1$. It follows that there does not exist $L' \neq L$ where $x,y \in L'$, and thus, no half line corresponds to the $(1,1,0)$-branch of the tree, a contradiction. 
\end{proof}
The proof that no small $\eps$-net exists is a straightforward probabilistic proof similar in spirit to that of Theorem \ref{thm:lbexists}, and is given in detail in \cite{AKMM2002}. 
The proof bounds from above the probability of any fixed set of size (say) $0.1 p \log p$ to intersect all half lines, and then takes a union bound over all such sets. 

Next, we show how to generalize the above to get a lower bound with linear dependence in $d$. Let $p$ be as above, consider $d$ copies of the projective plane of order $p$ and let  $\mathcal{C}_1, \ldots, \mathcal{C}_d$ be collections of half lines generated as above, one in each plane. Now let 
$\mathcal{C}$ be the collection of all unions of exactly $d/2$ half lines coming from different planes, namely, $\mathcal{C}$ contains all sets of the form 
$$
H = H^{i_1}_{L_{j_1}} \cup H^{i_2}_{L_{j_2}} \cup\ldots \cup H^{i_{d/2}}_{L_{j_{d/2}}},
$$
where $i_1 < i_2 < \ldots < i_{d/2} \in [d]$,
$H^{i_t}_{L_{j_t}}$ is a half line from the $i_t$ copy corresponding to the line $L_{j_t}$ in that copy of the plane. 

Consider the family $\mathcal{C}$ with the underlying universe with $d(p^2 + p + 1)$ points, containing all points from all $d$ planes. 
\begin{claim}
	The Littlestone dimension of $\mathcal{C}$ is at most $d$.
\end{claim}
\begin{proof}
The proof is a straightforward extension of the proof of Claim \ref{clm:LB_Littlestone2}. Suppose to the contrary that the Littlestone dimension is $t > d$. Let $T$ be a labeled tree of depth $t$ as in the definition of Littlestone dimension and consider its all-$1$ branch. This branch corresponds to some set $H = H^{i_1}_{L_{j_1}} \cup H^{i_2}_{L_{j_2}} \cup\ldots \cup H^{i_{d/2}}_{L_{j_{d/2}}}$. In particular, all elements labeling nodes along the branch are contained in $H$.

By the pigeonhole principle, there exist three elements $x,y,z$ along the branch (in this order) contained in the same half line $H_L$ from one of the plane copies. We claim that there is no set in $\mathcal{C}$ that corresponds to any branch which is all-$1$ up until (and not including) $z$, and takes the value $0$ at $z$. Indeed, such a set $H$, if exists, will contain $x,y$ bot not $z$. However, this is a contradiction as in Claim \ref{clm:LB_Littlestone2}: any set $H$ that contains $x,y$ must also contain all elements in the half line $H_L$ containing them both, and thus $z \in H$. 
\end{proof}

It remains to prove that there is no $\eps$-net of size $o(d \eps^{-1} \log{\eps^{-1}})$. But this follows easily from the $\Omega(\eps^{-1} \log \eps^{-1})$ lower bound for each of the planes separately: there exists some absolute constant $C > 0$ so that for each of the planes at hand, no $\eps$-net of size $C \eps^{-1} \log{\eps^{-1}}$ exists. Consider now any set $S$ of less than $Cd \eps^{-1} \log{\eps^{-1}} / 2$ points in our universe, the union of all planes; since each point belongs to exactly one plane, there exist $d/2$ planes with less than $C \eps^{-1} \log \{\eps^{-1}\}$ points. Let $i_1 < i_2 < \ldots < i_{d/2}$ denote their indices. It follows that there exists some set $H = H^{i_1}_{L_{j_1}} \cup H^{i_2}_{L_{j_2}} \cup\ldots \cup H^{i_{d/2}}_{L_{j_{d/2}}} \in \mathcal{C}$ not intersecting $S$. This completes the proof.
\end{proof}

%% file: prob-material.tex
\section{Probabilistic Material}

\subsection{Filtration and Martingales}\label{sec:martingales}

In this section we give a brief probability background to Martingales, considering only finite probability spaces. Recall that a probability space consists of a sample space $\Omega$, a $\sigma$-field $F \subseteq \{0,1\}^\Omega$ that contains all measurable events and a probability measure $\mu$ over $\Omega$. With finite probability spaces, it is possible for $F$ to contain all subsets of $\Omega$, however, smaller sets can be considered as well. For instance, if $\rv{y}_1,\dots,\rv{y}_n$ are random variables over the finite space $Y$, then $\sigma(\rv{y}_1)$, the \emph{$\sigma$-field generated by $\rv{y}_1$}, contains all the events that depend only on $\rv{y}_1$. Formally, we have $\Omega = Y^n$ and $\sigma(\rv{y}_1) = \{ \{\rv{y}_1 \in U \} \colon U \subseteq Y \}$, where $\{Y \in U \} = \{ (y_1,\dots,y_n) \colon y_1 \in U \}$ is the event that $\rv{y}_1\in U$. Similarly, we can have sigma fields generated by multiple random variables, for instance, $\sigma(y_1,y_3,y_4)$, that contains all the events that depend only on these three random variables. It is in fact also possible to consider the $\sigma$-algebra generated by zero random variables $\sigma(\{\}) = \{0,\Omega\}$ which is called the \emph{trivial $\sigma$-algebra}.

Conditioning on more random variables results in a larger $\sigma$-algebra, namely, if $i \le j$ then $\sigma(\rv{y}_1,\dots,\rv{y}_i) \subseteq \sigma(\rv{y}_1,\dots,\rv{y}_j)$. Intuitively, larger $\sigma$-algebras contain more information. We say that a $\sigma$-field $F$ is \emph{$\rv{y}$-measurable} if $\sigma(\rv{y}) \subseteq F$, which intuitively holds whenever $F$ it contains all the information on $\rv{y}$. Further, a \emph{filtration} is a collection of nested $\sigma$-algebras $F_0 \subseteq F_1 \cdots \subseteq F_n$.

One can define \emph{conditional expectation} with respect to a $\sigma$-algebra. In our application, each $\sigma$-algebra will be generated by a collection of random variables, and it holds that
\begin{equation} \label{eq:39}
\Ex[\cdot \mid \sigma(\rv{y}_1,\dots,\rv{y}_k)]
= \Ex[\cdot \mid \rv{y}_1,\dots,\rv{y}_k].
\end{equation}
Notice that the quantity in \eqref{eq:39} is a function of $\rv{y}_1,\dots,\rv{y}_k$, hence it is also a random variable.
Additionally, if $F$ is the trivial $\sigma$-algebra then
\[
\Ex[\cdot \mid F]
= \Ex[\cdot \mid \sigma(\{\})]
= \Ex[\cdot].
\]
And if $\rv{y}$ is $F$-measurable, then $\Ex[\rv{y}\mid F] = \rv{y}$.

A collection of random variables $\rv{z}_0,\dots,\rv{z}_n$ defines a \emph{Martingale adapted to the filtration $F_0 \subseteq \cdots \subseteq F_n$} if $\rv{z}_i$ is $F_i$-measurable and if for any $i < j$, $\Ex[\rv{z}_j \mid F_i] = \rv{z}_i$. The simplest case is when $F_i = \sigma(\rv{z}_1,\dots,\rv{z}_i)$, and there, the martingale condition translates to $\Ex[\rv{z}_j \mid \rv{z}_1,\dots,\rv{z}_i] = \rv{z}_i$. However, in the general case $F_i$ can have additional information on other random variables.

Remarkably, Martingales obey high probability bounds. Perhaps the most well known bound is Azuma's inequality, which is an adaptation of Chernoff's bound for Martingales:
\begin{lemma}\label{lem:azuma}
	Let $\rv{y}_0,\dots,\rv{y}_n$ be a Martingale adapted to the filtration $F_0,\dots,F_n$. Let $a_1,\dots,a_n \ge 0$ be numbers such that almost surely, $|\rv{y}_i - \rv{y}_{i-1}| \le a_i$. Then, for any $t\ge 0$,
	\[
	\Pr[\rv{y}_n-\rv{y}_0 > t]
	\le \exp\lp(\frac{-t^2}{2\sum_i a_i^2}\rp).
	\]
\end{lemma}

\subsection{Sampling Without Replacement}\label{sec:app-without-rep}

Further, we have the following version of Chernoff without replacement:
\begin{lemma}[\cite{bardenet2015concentration}]\label{lem:chernoff-wout-replacememnt}
	Let $a_1,\dots,a_N \in \mathbb{R}$ and let $I$ denote a uniformly random subset of $[N]$ of size $n \in \mathbb{N}$. Let $R = \max_i a_i - \min_i a_i$. Then, for any $t > 0$,
	\[
	\Pr\lp[\frac{1}{n}\sum_{i\in I} a_i - \frac{1}{N}\sum_{i=1}^N a_i > t\rp] \le \exp\lp(\frac{-2nt^2}{R^2}\rp).
	\]
\end{lemma}

Another without-replacement lemma:
\begin{lemma}[\cite{chatterjee2005concentration}, Proposition 3.10]\label{lem:permutation}
	Let $\{a_{ij}\}_{i,j\in [n]}$ be a collection of numbers from $[0,1]$. Let $Y = \sum_{i=1}^n a_{i\pi(i)}$ where $\pi$ is drawn from the uniform distribution over the set of permutations of $\{1,\dots,n\}$. Then for any $t\ge 0$,
	\[
	\Pr\lp[ 
	|Y-\Ex Y| \ge t
	\rp]
	\le 2 \exp(-t^2/(4\Ex Y + 2t)).
	\]
\end{lemma}

\begin{proof}[Proof of Lemma~\ref{lem:subset-intersection}]
	First item follows directly from \Cref{lem:chernoff-wout-replacememnt}. The second item follows from \Cref{lem:permutation} as described below.
	Define $m=|U|$. Let $\pi \colon [n] \to [n]$ be a uniformly random permutation and let $I = \{i \colon \pi(i) \le k\}$.
	Define $\{a_{i,j}\}_{i,j\in[n]}$ by $a_{i,j} = 1$ if $i \in U$ and $j\le k$. Notice that for all $i \in U$, $a_{i\pi(i)} = 1$ if $i \in I$ and for all $i \notin U$, $a_{i\pi(i)} = 0$. Hence,
	\[
	Y:= \sum_{i=1}^n a_{i\pi(i)}
	\]
	equals $|I \cap U|$ and $\Ex Y = km/n$. From Lemma~\ref{lem:permutation} we derive that for any $t \ge 0$,
	\[
	\Pr\lp[
	\lp| Y - \Ex Y\rp| \ge t
	\rp]
	\le \exp\lp(- \frac{t^2}{4\Ex Y + 2t}\rp).
	\]
	Substitute $t = \alpha \Ex Y$ and we get that
	\begin{align*}
	\Pr\lp[
	\lp| \frac{Y}{\Ex Y} - 1\rp| \ge \alpha
	\rp]
	= \Pr\lp[
	\lp| Y - \Ex Y\rp| \ge t
	\rp]
	\le \exp\lp(- \frac{t^2}{4\Ex Y + 2t}\rp)
	= \exp\lp(- \frac{\alpha^2 \Ex[Y]^2}{4\Ex Y + 2\alpha \Ex Y}\rp)\\
	\le \exp\lp(- \frac{\alpha^2 \Ex[Y]^2}{6\Ex Y}\rp)
	= \exp\lp(- \frac{\alpha^2 km}{6n}\rp).
	\end{align*}
\end{proof}

\begin{proof}[Proof of Lemma~\ref{lem:var-without-rep}]
	Denote $|U|=m$.
	For any $i \in U$, let $\rv{z}_i$ denote the indicator of whether $i \in \rvI$, and notice that $|U\cap \rvI| = \sum_{i \in U} \rv{z}_i$.
	Therefore, we have 
	\[
	\Ex[|U\cap \rvI|]
	= \Ex\lp[\sum_{i\in U} \rv{z}_i\rp]
	= \sum_{i\in U}\Ex[\rv{z}_i]
	= \sum_{i\in U}k/n
	= mk/n.
	\]
	Next, 
	\[
	\Ex[|U\cap \rvI|^2]
	= \sum_{i,j\in U} \Ex \rv{z}_i \rv{z}_j
	= \sum_{i \in U} \Ex \rv{z}_i^2
	+ 2\sum_{i,j\in U \colon i<j} \Ex \rv{z}_i \rv{z}_j.
	\]
	For the first term, since $\rv{z}_i$ is an indicator, we have
	\[
	\sum_{i \in U} \Ex \rv{z}_i^2
	= \sum_{i \in U} \Ex \rv{z}_i
	= mk/n.
	\]
	For the second term, fix $i < j$, and we have
	\begin{align*}
	\Ex \rv{z}_i \rv{z}_j
	= \Pr[\rv{z}_i=1, \rv{z}_j=1]
	= \Pr[\{i,j\}\subseteq \rvI]
	= \frac{1}{\binom{n}{k}} \lp|\lp\{ I \subseteq [n] \colon |I|=k, i,j\in I \rp\} \rp|\\
	= \frac{\binom{n-2}{k-2}}{\binom{n}{k}}
	= \frac{(n-2)!k!(n-k)!}{n!(k-2)!(n-k)!}
	= \frac{k(k-1)}{n(n-1)}
	\le \frac{k^2}{n^2}.
	\end{align*}
	We derive that
	\[
	\Ex[|U\cap \rvI|^2]
	\le \frac{mk}{n} + \frac{m^2 k^2}{n^2}.
	\]
	Hence,
	\[
	\mathrm{Var}(|U\cap \rvI|)
	= \Ex[|U\cap \rvI|^2]
	- \Ex[|U\cap \rvI|]
	\le \frac{mk}{n} + \frac{m^2 k^2}{n^2} - \frac{m^2 k^2}{n^2}
	= \frac{mk}{n},
	\]
	as required.
\end{proof}